\documentclass{article}

\usepackage[preprint]{neurips_arxiv} 
\usepackage{notations}

\usepackage[utf8]{inputenc} 
\usepackage[T1]{fontenc}    
\usepackage{hyperref}       
\usepackage{url}            
\usepackage{booktabs}       
\usepackage{amsfonts}       
\usepackage{nicefrac}       
\usepackage{microtype}      
\usepackage{xcolor}         
\usepackage{graphicx}       
\usepackage{subfigure}
\usepackage{amsmath,amsfonts,amssymb,amsthm} 
\usepackage{placeins} 
\usepackage{bm}

\usepackage[capitalize,noabbrev]{cleveref} 

\usepackage{microtype}
\usepackage{graphicx}
\usepackage{subfigure}
\usepackage{booktabs} 
\usepackage{notations}
\usepackage{subcaption}

\usepackage{hyperref}

\theoremstyle{plain}
\newtheorem{theorem}{Theorem}
\newtheorem{proposition}{Proposition}

\theoremstyle{definition}

\newtheorem{assumption}[theorem]{Assumption}
\theoremstyle{remark}
\newtheorem{remark}[theorem]{Remark}

\title{When pre-training hurts LoRA fine-tuning:\\a dynamical analysis via single-index models}

\author{%
Gibbs Nwemadji\textsuperscript{1}
\quad
Bruno Loureiro\textsuperscript{2}
\quad
Jean Barbier\textsuperscript{3}
\\[0.8em]
\textsuperscript{1}International School of Advanced Studies, Trieste, Italy
\\
\textsuperscript{2}Département d'Informatique, École Normale Supérieure, PSL \& CNRS
\\
\textsuperscript{3}The Abdus Salam International Centre for Theoretical Physics, Trieste, Italy
}

\begin{document}

\maketitle


\begin{abstract}
   Pre-training on a source task is usually expected to facilitate fine-tuning on similar downstream problems.
In this work, we mathematically show that this naive intuition is not always true: excessive pre-training can computationally \emph{slow down} fine-tuning optimization.
We study this phenomenon for low-rank adaptation (LoRA) fine-tuning on single-index models trained under one-pass SGD. Leveraging a summary statistics description of the fine-tuning dynamics, we precisely characterize how the convergence rate depends on the initial fine-tuning alignment and the degree of non-linearity of the target task. The key take away is that even when the pre-training and downstream tasks are well aligned, strong pre-training can induce a prolonged search phase and hinder convergence.
Our theory thus provides a unified picture of how pre-training strength and task difficulty jointly shape the dynamics and limitations of LoRA fine-tuning in a nontrivial tractable model. On the practical side, we empirically show that our theoretical findings extend beyond our toy model and remain relevant in the context of a vision-transformer model trained on real data.
\end{abstract}

\section{Introduction}
\label{introduction}
Recent advances in machine learning have established pre-trained models as a central paradigm for solving a wide range of tasks \cite{han2021pre}.
Rather than training models from scratch, modern approaches increasingly rely on adapting a fixed, pre-trained representation using limited task-specific data \cite{qiu2020pre, han2021pre}.
This trend is particularly pronounced for large language models (LLMs) but extends more broadly.
Beyond their empirical performance, pre-trained models offer practical advantages, including reduced computational cost and faster convergence.
These benefits have motivated the development of \emph{parameter-efficient fine-tuning} (PEFT) methods~\cite{houlsby2019parameter,xu2023parameter},
which aim to adapt large models while updating only a small subset of parameters.

Among PEFT methods, \emph{Low-Rank Adaptation} (LoRA)~\cite{hu2022lora} has been widely adopted. In this framework, a pre-trained model is adapted to a downstream task by learning a low-rank update of a subset of choice of the weight matrices,
while keeping the remaining parameters fixed~\cite{hu2022lora,dettmers2023qlora}.
Formally, consider a neural network mapping an input $\boldx$ to an output via $f(\boldx;\theta)$,
where $\theta$ denotes the set of model parameters.
Let $\tilde{\omeg}\in\mathbb{R}^{K\times d}$ denote a specific pre-trained weight matrix within $\theta$—for instance,
an attention or projection layer—to which LoRA is applied (in practice, LoRA can be applied simultaneously to multiple matrices). In this setting, $\tilde{\omeg}$ is adapted by learning a low-rank correction of the form $\boldsymbol{U}\omeg$,
where $\boldsymbol{U}\in\mathbb{R}^{K\times R}$ and $\omeg\in\mathbb{R}^{R\times d}$ are trainable matrices with
$R\ll\min\{K,d\}$.
The adapted model can be~written~as
\begin{equation}
    f_{\tilde{\omeg}}(\boldx;\boldsymbol{U},\omeg):=
    f(\boldx;\,\theta\setminus\{\tilde{\omeg}\}\cup\{\tilde{\omeg}+\boldsymbol{U}\omeg\}).
\end{equation}
Compared to full fine-tuning which updates all $k\times d$ entries of $\tilde{\omeg}$, LoRA reduces the number of trainable parameters to $(K+d)R$,
yielding significant computational savings when $R$ is small.
Crucially, the pre-trained matrix $\tilde{\omeg}$ is assumed to have been learned from diverse data, and therefore to encode representations that are informative for the downstream task. Understanding the interplay between the pre-trained weights and the low-rank fine-tuning update is critical to the development of a principled understanding~of~LoRA.

Despite its remarkable success~\cite{yang2024low,mao2025survey}, LoRA is only partially understood.
While recent studies have begun to characterize its expressive power, convergence properties, and optimization behavior~\cite{zeng2023expressive,mu2025convergence,kratsios2025sharp,liang2025does,kim2025lora,xu2025understanding,zhang2025lora},
several fundamental questions remain open. In particular, an important one is \emph{how the pre-trained representations impact the LoRA fine-tuning}. This is relevant as Hu \textit{et al}.~\cite[Sec.~7.3]{hu2022lora} suggest that the LoRA adaptation matrix may amplify task-specific features learned during pre-training but not sufficiently emphasized by the pretrained model.

In this work we address this question in the context of single-index models, focusing on the particular role the pre-trained weights play on the fine-tuning learning dynamics. Single-index models define a class of target functions $f_{\star}(\boldx)=\phi(\omeg_{\star}\cdot \boldx)$ in which the relevant information for prediction effectively lies on a one-dimensional subspace of $\mathbb{R}^{d}$. It has recently gained in popularity as a non-linear but analytically tractable model for studying non-convex optimization and feature learning in high-dimensional learning~\cite{gardner1989three,barbier2019optimal,arous2021online,damian2022neural,ba2022high,cui24asymptotics,dandi2024two}. Ben Arous \textit{et al} \cite{arous2021online} has shown that the convergence rate of one-pass SGD for single-index models is polynomial in the dimension $n = \Theta(d^{k^\star-1})$, with $k_{\star}$ being the \emph{information exponent}, quantifying the degree of non-linearity of activation $\phi$. This sample complexity is far from optimal, with $n=\Theta(d)$ typically sufficing for efficiently retrieving $\boldsymbol{\omega}_{\star}$ for most well-behaved $\phi$ \cite{kalai2009isotron, barbier2019optimal, damian2024computational}.

In this manuscript, we build on this literature to investigate the question of how the initial alignment of the pre-trained weights impact the LoRA fine-tuning under one-pass SGD. We analyze the optimization process, with particular emphasis on the early-stage \emph{search phase}, during which the model must identify the relevant direction in the data. This provides an algorithmic perspective on LoRA that complements existing statistical guarantees. More precisely, our \textbf{main contributions} are three-fold:
\begin{itemize}
    \item First, we quantify how the alignment between the pre-trained representation and the downstream task governs the time needed to escape the search phase. For common activation functions---linear, $\mathrm{erf}$, ReLU, and sigmoid---we uncover a trade-off between improved initial performance and the algorithmic tractability of subsequent LoRA fine-tuning: stronger pre-training can slow down adaptation rather than accelerate it (Figure~\ref{fig:dynamics_linear_activation}). This is consistent with recent empirical observations in LLMs by Springer \textit{et al} \cite{springer2025overtrained} and Isik \textit{et al} \cite{Isik2025Scaling}, often described as \emph{catastrophic overtraining}.
    
    \item Second, we identify regimes, determined jointly by the pre-training alignment and the activation function, in which escape from the search phase is provably impossible under the original labels. We then show that an appropriate label transformation, in particular label squaring, can remove this obstruction and restore escape, clarifying how target pre-processing can change the algorithmic difficulty of LoRA fine-tuning.
    
    \item Third, we complement the theory with controlled vision-transformer LoRA experiments on EMNIST$\to$FashionMNIST and CIFAR-10$\to$SVHN. These experiments confirm the qualitative prediction: later checkpoints can have higher source accuracy while converging to worse downstream performance under LoRA fine-tuning.
\end{itemize}

\begin{figure*}
    \centering
    {\includegraphics[width=0.9\textwidth]{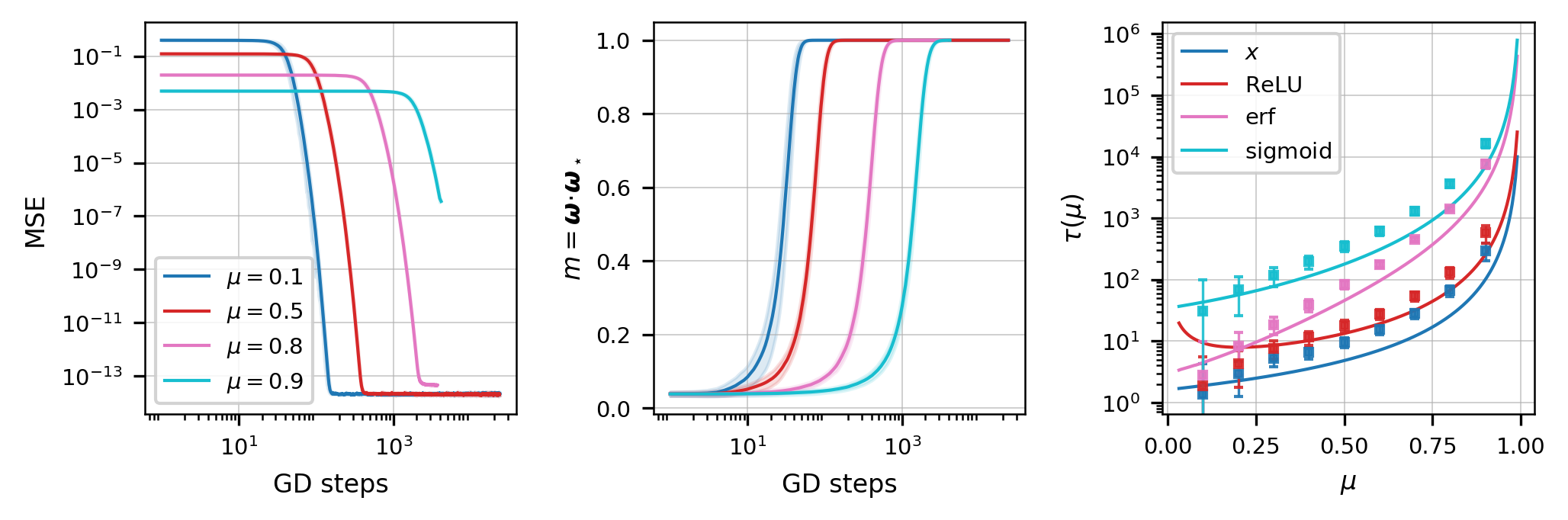}}
\vspace{-10pt}
 \caption{
Learning dynamics and escape times for the student model~\eqref{equ:student_model} trained on data from the teacher model~\eqref{equ:teacher_model}. 
\textbf{Left--middle:} finite online-SGD dynamics for the matched linear setting, showing the test MSE and alignment $m=\boldsymbol{\omega}\!\cdot\!\boldsymbol{\omega}_\star$ for different pretraining alignments $\mu$. Simulations use $d=1000$ and $B=5000$; shaded regions show one standard deviation over three runs.
\textbf{Right:} escape time $\tau(\mu)$ for matched activations with $\mathrm{IE}=1$. Solid lines are theory and markers are online-SGD estimates for different activations. Additional numerical evidence for ReLU and sigmoid activations is provided in App.~\ref{appendix:numerical evidence of slow down}.
}  \label{fig:dynamics_linear_activation}
\end{figure*}

\paragraph{Additional related work.}
A growing literature has recently begun to establish rigorous theoretical foundations for LoRA~\cite{hu2022lora}.
Early work by Zeng \textit{et al} \cite{zeng2023expressive} characterized rank conditions under which LoRA can accurately approximate a target function.
In the neural tangent kernel (NTK) regime, Jang \textit{et al} \cite{jang2024lora} showed that for ranks $R \ge \sqrt{n}$—where $R$ denotes the LoRA rank and $n$ the sample size—the loss landscape contains no spurious local minima.
Beyond asymptotic analyses, Mu \textit{et al} \cite{mu2025convergence} proved that, for a finite number of gradient steps, the convergence rate of LoRA is independent of the rank, although the limit point may differ substantially from that of full-rank gradient descent.
Related studies~\cite{hayou2024impact, kim2025lora,xu2025understanding} highlighted the role of zero initialization in biasing the dynamics toward regions containing global minima, while Zhang \textit{et al} \cite{zhang2025lora} demonstrated that, under carefully chosen initializations, a single gradient step may suffice for recovery.
In an asymmetric setting where only one LoRA factor is trained, Kratsios \textit{et al} \cite{kratsios2025sharp} showed that the generalization error scales as $O(\sqrt{R/n})$ with high probability (up to polylog factors). A common limitation across these works is that pre-trained weights are typically modeled as random or unstructured, and therefore not explicitly aligned with the downstream task, leaving the role of task-relevant pre-training largely unexplored. Our work participates to fill this gap in the context of single-index models.

\section{Setting}
We study a tractable yet non-trivial teacher--student model that captures the essential structure of LoRA-based transfer learning on a non-linear, non-convex task. More precisely, we assume the target function is a single-index model of the form $f_{\star}(\bx)=\phi(\omeg_\star \cdot \boldx)$, where $\omeg_\star\in\mathbb{S}^{d-1}$ is an unknown unit vector and $\boldx\in\mathbb{R}^d$ is the input data. We observe a dataset
$\mathcal{D}=\{(\boldx_i,y_i)\}_{i=1}^n$ consisting of conditionally (on $\omeg_\star$) i.i.d.\ samples, with inputs $\boldx_i\sim\mathcal{N}(0,I_d)$ and labels
\begin{equation}
\label{equ:teacher_model}
    y_i = \phi(\omeg_\star \cdot \boldx_i).
\end{equation}
In analogy with LoRA, we consider a student model obtained by updating a pre-trained direction with limited signal strength, i.e., we study predictors of the form:
\begin{equation}
\label{equ:student_model}
    f_{\tilde{\omeg}}(x_i; u,\omeg)
    = \sigma\bigl((\tilde{\omeg} + u\,\omeg)\cdot \boldx_i\bigr),
\end{equation}
where the learnable parameters consist of a scalar $u\in\mathbb{R}$ and a direction $\omeg\in\mathbb{S}^{d-1}$. The vector $\tilde{\omeg}\in\mathbb{R}^{d}$ represents a
feature direction inherited from pre-training, and is assumed to be aligned with the target direction $\omeg_\star$. A natural choice is therefore $\tilde{\omeg} := \mu\,\omeg_\star$, where the parameter $\mu\in(0,1)$  quantifies the amount of prior information carried by the pre-trained model. In App.~\ref{app_different_pre-trained weight}, we show that this encompasses also misaligned pre-trained weights such $\tilde{\boldsymbol{\omeg}}=\mu\,\boldsymbol{\omeg}_\star+(1-\mu)\,\boldsymbol{\xi}$
without loss of generality. 

This formulation corresponds to a transfer-learning scenario in which pre-training yields an informative but imperfect representation for the downstream task, here given by \eqref{equ:teacher_model}. Fine-tuning on the dataset $\mathcal{D}$ then improves performance beyond what is achievable with the pre-trained weights alone, via the rank-one correction $u\,\boldsymbol{\omega}$. Throughout this work, we assume that both the target $\phi(\cdot)$ and are square-integrable with respect to the standard Gaussian measure. It is worth noting that \cite{springer2025overtrained} provides theoretical grounds to quantify \emph{catastrophic overtraining}, but their analysis is limited to linear models and relies on an explicit alignment penalty to model weight transfer. As a result, it abstracts away from unconstrained fine-tuning as commonly used in practice and does not incorporate PEFT methods. In parallel, \cite{Isik2025Scaling} show that downstream performance may fail to scale monotonically with pre-training when source and target tasks are poorly aligned. In contrast, even when pre-training and downstream tasks are explicitly aligned, our framework exhibits a significant delay in learning induced by stronger pre-training.
 
\subsection{One-pass stochastic gradient descent}
Given a dataset $\mathcal{D}=\{(\boldx_{i},y_{i})\}_{i=1}^{n}$ of $n$ samples drawn from model \cref{equ:teacher_model}, we are interested in the fine-tuning training dynamics of $(u,\boldsymbol{\omega})$ with \emph{spherical one-pass~SGD}:
\begin{align}
\label{equ:gradient_step_omega}
\omeg^{i+1}
&=\frac{\omeg^{i}-\gamma\,\nabla_{\omeg} (y_{i}-f_{\tilde{\omeg}}(\boldx_i; u^{i},\omeg^{i}))^{2}
}{\bigl\|\omeg^{t}-\gamma\,\nabla_{\omeg} (y_{i}-f_{\tilde{\omeg}}(\boldx_i; u^{i},\omeg^{i}))^{2}\bigr\|},\quad
u^{i+1}=u^{i} -\gamma \nabla_{u}(y_{i}-f_{\tilde{\omeg}}(\boldx_i; u^{i},\omeg^{i}))^{2}
\end{align}
where $\gamma>0$ is the learning rate and $i\in[n]$. Note that, in line with other works on one-pass SGD for single-index models, the choice of constraining $\omeg\in\mathbb{S}^{d-1}$ is mostly for analytical simplicity, and it can be shown it does not impact the dynamics considerably \cite{arous2021online}. Note that one-pass SGD (also known as \emph{streaming} or \emph{online} SGD), sees one sample at a time with no repetition, implying that the convergence rate is equal to the sample complexity of the algorithm. 

Since the gradients in \cref{equ:gradient_step_omega} are unbiased estimators of the population gradient, the one-pass dynamics can be seen as a random discretization of gradient flow on the population risk \cite{robbins1951stochastic}:
\begin{equation}
\label{equ:population_loss}
\mathcal{L}(u,\omeg) = \mathbb{E}[\left(y-f_{\tilde{\omeg}}(\boldx; u,\omeg)\right)^{2}]
\end{equation}
where the expectation is over the pair $(\boldx,y)$. The key observation for the analysis that will follow is that the population loss in \cref{equ:population_loss} only sees the data through the pre-activations of the target $\lambda_\star:=\omeg_\star \cdot \boldx$ and model $\lambda:=(\mu\,\omeg_\star + u\,\omeg)\cdot \boldx$. Since $\boldx\sim\mathcal{N}(0,I_d)$, the pair $(\lambda_\star,\lambda)$ is jointly Gaussian with law $$ (\lambda_\star,\lambda) \sim \mathcal{N}\!\left( 0,\, \Sigma \right), \qquad \Sigma = \begin{pmatrix}
1 & \mu + u m \\
\mu + u m & r
\end{pmatrix},
$$
where we have introduced the ``overlap'' $m := \omeg\cdot\omeg_\star$ and the student pre-activation variance $r := \mu^2 + u^2 + 2\mu u m$. An important related quantity is the effective alignment $m_{\rm eff}:=\mu+u m$, which measures the total alignment of the model predictor with the target direction and serves as an indicator of recovery.

\paragraph{Order parameters and their dynamics.}
An important consequence of the discussion above is that the population loss $\mathcal{L}(u,\omeg)$ depends on the parameters $(u,\omeg)$ only through two scalar quantities: the coefficient $u$ and the overlap $m=\omeg\cdot\omeg_\star$. This suggests a reduced description of the high-dimensional learning dynamics in \cref{equ:gradient_step_omega} in terms of the evolution of these two summary statistics.  Indeed, this reduction will be the key technical tool in the analysis of the convergence rate for the LoRA fine-tuning that will follow. 

The evolution of $m$ can be obtained from \cref{equ:gradient_step_omega} by projecting the update $\boldsymbol{\omega}^{i+1}$ onto the target direction $\omeg_\star$. This defines a system of coupled stochastic processes for $(m^{i},u^{i})$, which are not autonomous. However, defining $t=\gamma i$ it can be shown that for small enough learning rate $\gamma$, this process concentrate on a deterministic limit given by gradient descent on the population loss
\begin{align*}
(m^{t+1},u^{t+1})
&=(m^{t},u^{t})-\gamma\,\nabla \tilde{\mathcal{L}}(u^t,m^t)+O(\gamma^2),
\end{align*}
where $\tilde{\mathcal{L}}(u,m)$ denotes the population loss expressed solely in terms of $(u,m)$. 
We refer the reader to App.~\ref{app:evolution_of_order_parameters} for details. For simplicity, we will henceforth use $\mathcal{L}(u,m)$ to denote this reduced loss.

We note that similar ideas, rooted in the statistical physics of learning literature have been employed in wide range of contexts in the study the one-pass SGD dynamics in machine learning \cite{saad1996learning,goldt2019dynamics,goldt2020modeling,veiga2022phase,arnaboldi2023high,arnaboldi2023escaping,saglietti2022analytical,mori2025optimal,soletskyi2025theoretical}. 

We consider initial conditions consistent with standard LoRA practice.
In applications, the scalar parameter $u$ is typically initialized at zero
\cite{hu2022lora, dettmers2023qlora, yang2024low}, while the student direction
$\omeg$ is chosen at random. In the high-dimensional limit $d\to\infty$, such a random initialization yields an initial overlap $m = \omeg \cdot \omeg_\star = O(d^{-1/2})$
with high probability. For analytical convenience, and in order to obtain a nontrivial joint evolution of the order parameters, we therefore assume that both $u$ and $m$ are initialized
at scale $d^{-1/2}$. This scaling remains asymptotically close to the standard initialization used in practice. Consistently, recent works \cite{hayou2024impact, kim2025lora, xu2025understanding} show that small initialization of the LoRA block leads to improved generalization error.

\subsection{Hermite expansion and learning phases}
\label{sec:hermite}
In the standard well-specified $\sigma(\cdot)=\phi(\cdot)$ single-index model, Ben Arous \textit{et al} \cite{arous2021online} has shown that the convergence rate of one-pass SGD in the high-dimensional and small learning rate regime is given by $n=\Theta(d^{k^{\star}-1})$, where $k^{\star}$ is the \emph{information exponent} (IE) associated to the activation function, defined as the smallest non-zero term $k^\star \geq 1$ in the Hermite expansion of the activation function $\phi$. Indeed, this is intuitive, as $k^{\star}$ quantifies the strength of the leading gradient signal available at the early phases~of~SGD. 

In our fine-tuning setting, model and target may have different pre-activation variances, which requires a generalized Hermite expansion. Let $a>0$ denote the variance of a Gaussian pre-activation, and assume that
$\mathbb{E}[\sigma(z)^2] < \infty$ for $z \sim \mathcal{N}(0,a)$.
 Then $\sigma$ admits
the expansion
\begin{equation}
\label{equ:generalize_Hermite}
\sigma(z)
=\sum_{k\ge 0} \frac{\sigma^{[a]}_k}{k!a^{k}}\,\He^{[a]}_k(z),
\end{equation}
with details on $\sigma^{[a]}_k$ and  $\He^{[a]}_k$---that denotes the $k$-th Hermite polynomial---available in App.~\ref{app:Hermite_properties}. The teacher activation operates at unit variance, and thus admits a standard
Hermite expansion $\phi(z)=\sum_{k\ge 0} \frac{\phi_{k}}{k!}\,\He_k(z)$.
The mismatch between the pre-activation variances of the target and model plays a nontrivial role in shaping the IE of the model. To illustrate this effect, we consider the well-specified setting $\sigma(\cdot)=\phi(\cdot)$ and assume, for simplicity, that the activation function is a pure Hermite polynomial of degree $k^\star \ge 2$ corresponding to having a target with IE $=k^\star$. In our student model, the pre-activation variance is $r =r(u,m,\mu):= \mu^2 + u^2 + 2\mu um$, which is generically different from $1$. As a consequence, the Hermite expansion of the student possesses nonzero coefficients only for Hermite polynomials whose degrees match the parity of $k^\star$, thereby inducing an effective reduction of the student’s IE to at most $2$, independently of that of the teacher. In particular, this analysis suggests that a byproduct of learning within the LoRA framework is an implicit regularization of the IE to at most $2$. A detailed derivation of this variance-induced IE reduction is provided in App.~\ref{appen:non-unit variance}.

Using the generalized Hermite expansion \eqref{equ:generalize_Hermite}, together with the detailed computations presented in App.~\ref{app:property_Hermite}, the population loss can be written explicitly as
\begin{equation*}
\mathcal{L}(u,m)
=
\sum_{k=0}^{\infty} \frac{1}{k!}
\Big(
\frac{\phi_k^2}{2}
+
\frac{(\sigma^{[r]}_k)^2}{2r^{k}}
-
\frac{\sigma^{[r]}_k\phi_k}{r^{k}}(\mu+um)^k
\Big).
\end{equation*}
Although the second term depends on $(u,m)$ via $r(u,m,\mu)$, it simply acts as an effective regularization (it corresponds to the term in the loss depending on the student only).
Hence, learning of task-relevant structure is driven by the third term, which governs the sample complexity. When the teacher has ${\rm IE}=k^\star$, the leading nonzero contribution of this term scales as
$m_{\rm eff}^{k^\star}=(\mu+um)^{k^\star}$, while higher-order terms scale as
$m_{\rm eff}^{k^\star+p}$ for $p\in\mathbb{N}_{+}$.
In the setting of Ben Arous \textit{et al} \cite{arous2021online} ($\mu=0,u=1$), one has $m_{\rm eff}=m\ll1$ at early times, so higher-order terms are suppressed and the loss is well approximated by its leading Hermite component.
In contrast, when $\mu\neq0$, although $um$ remains infinitesimal initially, $m_{\rm eff}=O(1)$ due to $\mu$, making higher-order contributions non-negligible.
As a result, the population loss cannot be truncated to its leading Hermite term, showing that \emph{the IE alone is insufficient} to characterize the loss near initialization in the presence of pre-training.

\paragraph{Learning phases.}
Our analysis reveals the existence of two distinct phases in the evolution of the order parameters, which are clearly illustrated by simulations (see Figure~\ref{fig:dynamics_linear_activation}):
\begin{itemize}
    \item \textit{Correlated search phase:} the dynamics starts from a weakly correlated regime in which the order parameters satisfy $\max(|u|,|m|) \ll \mu$, though the effective teacher-student overlap satisfies $m_{\rm eff}\simeq \mu.$
    \item \textit{Descent phase:} once the process enters this regime, at least one of the order parameters becomes comparable to the signal strength, i.e.\ $|m| \ge \mu$ and/or $|u| \geq \mu$.
\end{itemize}
In Ben Arous \textit{et al} \cite{arous2021online}, Thm.~1.5 shows that the majority of samples/updates are used during the search phase, while, once this phase is exited, the model recovers the target direction at an exponential rate in time, making the search phase the most algorithmically relevant regime. This is in sharp contrast to our setting, where the model has finite correlation $\mu>0$ with the target vector $\omeg_\star$ at initialization and hence throughout the full trajectory, so that the search dynamics remain correlated. Our main results (see Section \ref{main_results}) provide a fine-grained analysis of how the signal strength $\mu$ and the choice of the activation function affect the phenomenology of this correlated search phase. Moreover, we establish in App.~\ref{descent_phase} that, upon entering the descent phase, the dynamics still lead to exponential recovery of the teacher direction for matching teacher-student activation function.

\section{Main results}
\label{main_results}
In the correlated search phase, both $u$ and $m$ remain small compared to $\mu$, which implies that the drift of the dynamics is governed by the leading-order dependence of the population-loss gradients on $(u,m)$. We therefore linearize the gradient of its full Hermite expansion, keeping only terms proportional to $u$ and $m$.
In the infinitesimal learning-rate limit, this leads to the following linearized dynamics for the order parameters:
\begin{equation}
\label{equ:dynamic_u_m_correlated_search_phase}
\dot u = u B + m A,
\qquad
\dot m = u A,
\end{equation}
with 
\begin{align}
\label{equ:A_and_B}
A&=-\sum_{k=0}^{\infty}
\frac{\bar{\sigma}^{[r]}_k}{k!\mu^{k+1}}
\Big(-\phi_k + \frac{\sigma^{[r]}_k}{\mu^k}\Big), B=-\Big[
\sum_{k=0}^{\infty}\frac{\bar{\sigma}^{[r]}_k\sigma^{[r]}_k}{k!\,\mu^{2k+2}}
+\sum_{k=1}^{\infty}\frac{\phi_k !\mu^{-k-2}}{(k-1)}
\Big(\sigma^{[r]}_k - \frac{\bar{\sigma}^{[r]}_k}{k}\Big)
\Big]    
\end{align}
with $\bar{\sigma}^{[r]}_k$ defined in App.~\ref{app:gradient_population_risk} and derivations details are provided in the App.~\ref{app:linearize_the_population_loss} and \ref{appendix:escaping_time}. Note that for $A \neq 0$ there is no solution in which $u$ or $m$ evolves independently: any infinitesimal deviation from the origin immediately induces coupled dynamics in both variables.

\subsection{Escaping time}
 Using the linearized gradient dynamics in $(u,m)$, one would expect the exit time $t_{\rm exit}$—defined as the time at which at least one of the order parameters becomes comparable to the signal strength, i.e.\ $|u|=\Theta(\mu)$ and/or $|m|=\Theta(\mu)$—to scale at most as $O(\log d)$, in line with the analysis of Ben Arous \textit{et al} \cite{arous2021online} for IE $=2$. However this intuition is not always correct: depending on the choice of the activation function, the learning dynamics may fail to escape the correlated search phase because of the pre-training. We capture this phenomenon by providing a sharper characterization of $t_{\rm exit}$ that makes its dependence on the drift coefficients $A,B$ explicit. From now on, for sequences $a=a_d$ and $b=b_d$ we write $a\sim_d b$ if $\lim_{d\to\infty}a/b=1$. 

\begin{proposition}[Exit time from the correlated search phase]
\label{prop:escaping_mediocrity}
Let us call $\mu$ the signal strength and assume that $A \neq 0$. Then the time at which the dynamics escapes the correlated
search phase satisfies
\begin{equation}
\label{equ:exit_time}
    t_{\rm exit} \sim_{d} \frac{\tau(\mu)}{2} \log d, \ \ \tau(\mu)
  =
  \frac{-B+\sqrt{B^{2} + 4A^{2}}}{2A^{2}},
\end{equation}
with $A,B$ given in \eqref{equ:A_and_B} depend on the teacher and student activation functions
$\sigma(\cdot), \phi(\cdot)$ and $\mu$.
\end{proposition}
 
\begin{remark}
Prop. \ref{prop:escaping_mediocrity} shows that the exit time scales as $A^{-2}\log d$.
Consequently, small values of $A$ and $B<0$ substantially slow down the transition out of
the correlated search phase. In particular, if for some choice of teacher--student
activation functions and $\mu$ one has $A=0$, then the ratio $t_{\rm exit}/\log d $ diverges. This corresponds to a regime in which the dynamics fails to escape the correlated
search phase despite the presence of non-zero initial correlation. The proof is provided in App.~\ref{appendix:escaping_time}.
\end{remark}

\begin{figure}[t]
    \centering
    \includegraphics[width=\columnwidth]{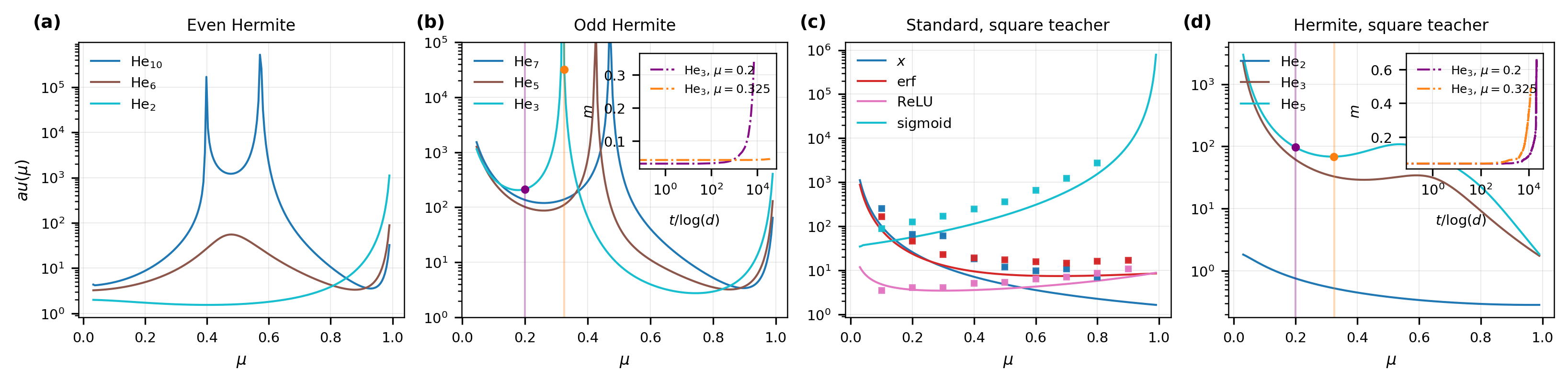}
    \vspace{-21pt}
 \caption{
Characteristic escape time $\tau(\mu)$ for exiting the correlated search phase. Solid lines show theoretical predictions and markers denote numerical estimates.
\textbf{(a)} Even Hermite activations in the matched teacher--student setting.
\textbf{(b)} Odd Hermite activations in the same setting. The vertical lines and colored dots mark $\mu\in\{0.2,0.325\}$, for which the inset shows the overlap dynamics of $\He_3$, illustrating the slowdown near the singular region.
\textbf{(c)} Standard activations when the teacher labels are squared before training the student, thereby modifying the information exponent of the loss.
\textbf{(d)} Hermite activations in the same squared-label setting. Squaring removes the singular behavior seen in panel~\textbf{(b)} and leads to a more regular dependence of $\tau(\mu)$ on $\mu$; the inset shows the corresponding $\He_3$ overlap dynamics at the same two values of $\mu$.
}
\label{fig:escaping_time_IE_1_matching_activation}
\end{figure}

\subsection{Well-specified activation function}

Although we argued that the IE is not enough to characterize the population loss, we nevertheless use it here as a guiding descriptor of the chosen activation function. Our conclusions should therefore be understood as statements about the specific activation considered, rather than about the entire class of functions sharing the same IE.
\paragraph{(a) Linear activation function.}
We first consider the linear case $\sigma(x)=\phi(x)=x$.
As expected, larger values of $\mu$ yield smaller initial test mean squared error,
reflecting stronger initial alignment with the teacher
(left panel of Fig.~\ref{fig:dynamics_linear_activation}).
However, a striking phenomenon emerges during training:
despite their worse initial performance, models initialized with smaller
signal strength $\mu$ converges to the global optimum significantly faster: highly pre-trained models (e.g.\ $\mu=0.9$) exhibit a pronounced
delay before convergence.
This behavior is clearly visible in the evolution of the overlap $m$
(middle panel of Fig.~\ref{fig:dynamics_linear_activation}), where models with larger $\mu$ remain trapped in the correlated
search phase for substantially longer times. Our theory accounts for this trade-off between initial, pre-trained performance and algorithmic tractability of the downstream LoRA fine-tuning. Since the characteristic time $\tau(\mu)$  determines the exit time
$t_{\rm exit}$ (Prop.~\ref{prop:escaping_mediocrity}), 
Fig.~\ref{fig:dynamics_linear_activation} (see right panel)
shows that $\tau(\mu)$ indeed increases monotonically with $\mu$. This slowdown is driven by the drift coefficient $A$ controlling the evolution
of $m$, which for linear activation scales as $A\propto 1-\mu$
(App.~\ref{app:dynamics_of_order_parameter_for_linear_activation}).
Consequently, stronger pre-trained alignment weakens the effective gradient signal along the teacher direction, delaying escape from the correlated search phase. We provide in App.~\ref{app:ODE_of_the_damp_harmonic_oxillator} a
physics-inspired explanation of this phenomenon.
A related phenomenon has been observed in the context of LLMs \cite{springer2025overtrained}, where excessive pre-training, corresponding to increasing the amount of pre-training data, is shown to impair downstream performance.

\paragraph{(b) Activation functions with IE $=1$.}
We next consider commonly used activations with $\mathrm{IE}=1$, including
$\mathrm{erf}$, sigmoid, and ReLU. Overall, the qualitative behavior observed
for the linear activation function largely persists: increasing the signal strength
$\mu$ typically leads to longer residence times in the correlated search phase.
For $\mathrm{erf}$ and sigmoid,
Fig.~\ref{fig:dynamics_linear_activation} (see right panel) shows that $\tau(\mu)$ governing escape grows monotonically with $\mu$. ReLU exhibits a mild deviation from this monotone trend at relatively small signal
strengths: as $\mu$ increases from $0$ to $0.2$, $\tau(\mu)$ decreases
slightly, suggesting that a weakly more informative pre-trained weight can
initially accelerate escape. However, once $\mu$ exceeds $0.2$, the behavior aligns with the other activations. A common feature across all these activation functions is the divergence of the $\tau(\mu)$ as $\mu\to 1^{-}$, implying that highly correlated pre-trained weights require substantially more gradient steps—and, in the online regime, substantially more fresh data—before the dynamics can leave the correlated search phase. Finally, for all IE $=1$ activations considered here, the escape-time curves appear continuous over $\mu\in(0,1)$: escape occurs for any fixed $\mu\in(0,1)$, but the associated time scale depends sensitively on both $\mu$ and the choice of activation function.

\paragraph{(c) Beyond IE $=1$.}
We now consider activation functions with $\mathrm{IE}\ge 2$, which in the standard, non-fine tuning case, are known to be harder to learn. As shown in~\cite{arous2021online, damian2023smoothing}, the duration of the search phase for such activations grows as a power of the IE. Our objective here is to assess how increasing the activation function complexity qualitatively alters the correlated search dynamics relative to the setting studied in~\cite{arous2021online}. To this end, we focus on pure Hermite activation functions—whose Hermite degree coincides with their IE—and distinguish between odd and even degrees, which exhibit qualitatively different behaviors.

\textbf{-- Odd Hermite activations.} For odd-degree Hermite activations, we have:

\begin{proposition}
\label{prop:singularity_pure_odd_Hermite}
Consider the well-specified setting $\sigma=\phi$ with pure
Hermite activation function $\sigma(z)=\He_k(z)$ of odd degree $k\ge 3$. Let $A(\mu)$ be as in \eqref{equ:A_and_B}. There exists at least one $\tilde{\mu}\in(0,1)$ such that $A(\tilde{\mu})=0$. Thus at $\mu=\tilde{\mu}$ the exit time
$t_{\rm exit}\sim_{d} (\tau(\tilde{\mu})\,\log d)/2$ diverges.
\end{proposition}

Prop.~\ref{prop:singularity_pure_odd_Hermite} shows that, in contrast to all activation functions considered so far, for pure odd Hermite activations there exist at least one nonzero value of the pre-training alignment $\tilde{\mu}$ for which the learner cannot escape the correlated search phase despite the initial
correlation. The proof is in App.~\ref{appendix:proof_proposition_Hermite}
Using Prop.~\ref{prop:escaping_mediocrity}, we can further refine the location of the singularity guaranteed by Prop.~\ref{prop:singularity_pure_odd_Hermite}. Indeed, in panel (b) of Figure~\ref{fig:escaping_time_IE_1_matching_activation}, we observe the divergence of $t_{\rm exit}$ for pure odd Hermite activation functions. This critical value shifts to the right as the degree of the Hermite polynomial increases. On both sides of this singularity, $t_{\rm exit}$ exhibits a parabola-like behavior, with again divergences as $\mu\to 0^+$ and $\mu\to 1^-$. Interestingly, as $\mu\to 1^-$, higher Hermite degrees yield a faster escape than lower ones, a counterintuitive behavior given the findings of \cite{arous2021online}. This qualitative behavior is reversed in the regime $\mu\to 0^+$, which agrees with Ben Arous \textit{et al} \cite{arous2021online}. Prop.~\ref{prop:exit_time_mu_to_1_pure_Hermite} provides quantitative results for these extremes. In the inset of Fig.~\ref{fig:escaping_time_IE_1_matching_activation} (see panel (c)), we provide numerical results for Hermite $\He_3$ showing the dynamic of $m$ across epoch for $\mu\in\{0.2, 0.325\}$. This confirms that the dynamics indeed take a significantly longer time for $\mu=0.325$, which is close to the singularity, compared to $\mu=0.2$. Further details on other observables can be found in App.~\ref{app:numerics_evidence_singularity}.

\textbf{-- Even Hermite activation functions.}
Using Prop.~\ref{prop:escaping_mediocrity}, we observe that for even Hermite
degrees that are sufficiently large, two singularities appear.
Indeed, as shown in panel (b) of Fig.~\ref{fig:escaping_time_IE_1_matching_activation},
where $\tau(\mu)$ is reported for pure even Hermite activation functions,
no singularity is observed for $\He_2$ and $\He_6$, whereas clear singularities
emerge for $\He_{10}$. The behavior as $\mu\to 1^{-}$ is qualitatively similar to that observed in the odd Hermite case. A distinct regime emerges, however, as $\mu\to 0^{+}$: $\tau(\mu)$ remains finite across all degrees, a fact for which Prop.~\ref{prop:exit_time_mu_to_1_pure_Hermite} provides a mathematical guarantee. This behavior is the result of the presence of a nonzero mean on the model—suggesting that the mean is easy to learn even for vanishing $\mu$—induced by such activations, whose magnitude increases with the Hermite degree but vanishes in the limit $\mu \to 1$; see App.~\ref{app:mean_even_Hermite} for a formal analysis.  Numerically (cf. App.~\ref{app:numerics_bias_even_Hermite}), we further observe
that for moderate signal strengths $\mu \le 0.2$, the dynamics exit the correlated search phase rapidly, in agreement with theoretical predictions, but subsequently enter an extended plateau in the evolution of $u$, while the overlap $m$ continues to grow.  Once $u$ leaves this plateau, convergence toward the global minimizer becomes rapid. This intermediate stagnation regime is specific to even Hermite activations and suggests that the mean component is learned quickly, whereas higher-order features are learned much later. Since the model has zero mean when using odd Hermite activations, this behavior is not observed in that case; supporting numerical evidence is reported in App~\ref{app:numerics_bias_even_Hermite}.

\textbf{-- Pure Hermite activations in extreme pre-training regimes.}
We now consider the extreme cases in which the pre-trained weights are either highly aligned or extremely weakly aligned with the target direction. For functions $a=a(\mu)$ and $b=b(\mu)$, $a \sim_\mu b$ means $\lim_{\mu\to\mu_0} a/b=1$.

\begin{proposition}
\label{prop:exit_time_mu_to_1_pure_Hermite}
Consider the matching teacher--student setting $\sigma=\phi$ with a pure Hermite activation function of degree $k\ge 1$, $\sigma(z)=\He_k(z)$ with $k\ge 3$. Let $\tau(\mu)$ 
be the characteristic prefactor of $t_{\rm exit}$ appearing in the scaling
\eqref{equ:exit_time} of Prop.~\ref{prop:escaping_mediocrity}.
Then the following asymptotics hold:
\begin{itemize}
    \item With strong alignment $\mu\to 1^{-}$, $$\tau(\mu) \sim_{\mu}  \frac{4}{(2k -1)^2}\frac{1}{(\mu^2-1)^2}.$$
    \item With weak alignment $\mu\to 0^{+}$,
    $$
    \tau(\mu) \sim_{\mu}
    \begin{cases}
         B^{-1}_0, & \text{if } k=2p \text{ with } p\in\mathbb{N},\; p\ge 1,\\[4pt]
        \displaystyle -\frac{B_0}{c_p\mu^2}, & \text{if } k=2p+1 \text{ with } p\in\mathbb{N},\; p\ge 1,
    \end{cases}
    $$
    where $B_{0}=\lim_{\mu\to 0^{+}} B(\mu)$ is finite (recall \eqref{equ:A_and_B}) and $c^{-1}_p=p!(p-1)!2^{2p-1}/(2p-1)!$.
\end{itemize}
\end{proposition}

This result provides fine-grained analysis of observations made in Fig.~\ref{fig:escaping_time_IE_1_matching_activation} (see panels (b) and (c)). For $\mu$ near one, higher-degree Hermite activation functions can escape the correlated search phase faster than lower one despite having a larger IE (and thus being harder to learn in the standard setting). In contrast, near zero, odd and even degrees behave differently, with odd degrees exhibiting a pronounced slowdown (see numerics in Fig.~\ref{fig:dynamics_Hermite_activation_appendix} of App.~\ref{app:numerics_bias_even_Hermite}). Taken together, these scalings clarify how the interplay between pre-training strength and activation function complexity reshapes the search dynamics. The proof of Prop. \ref{prop:exit_time_mu_to_1_pure_Hermite}  is done in App.~\ref{appendix:proof_proposition_Hermite}.

\subsection{Misspecified activation functions}
As discussed in the introduction, when ${\rm IE} > 2$, there exists a fundamental gap between the sample complexity of one-pass SGD and that of optimal algorithms for weakly learning single-index models. This gap stems from the fact that, at initialization, SGD gradients are both independent and weak. In particular, one-pass SGD updates can be viewed as CSQ queries of the form $\mathbb{E}[y,\varphi(x)]$, for which a general lower bound of $n \gtrsim O(d^{-{\rm IE}/2})$ holds \cite{damian2023smoothing}. By contrast, iterative algorithms such as AMP \cite{barbier2019optimal} or full-batch gradient descent \cite{dandi2024benefits}, which typically achieve weak recovery with $n=\tilde{O}(d)$ samples, effectively implement more expressive queries of the form $\mathbb{E}[\psi(y)\varphi(x)]$. These are known as SQ queries and allow for arbitrary transformations of the labels.\footnote{The sample complexity lower bound for SQ algorithms $n \gtrsim \tilde{O}(d^{{\rm GE}/2})$ is instead governed by the so-called \emph{generative exponent}. For most common activation functions $\phi$, one has ${\rm GE}=2$ \cite{damian2024computational}.}

Therefore, it is natural to ask whether label transformations can speed up the convergence of fine-tuning.  Note that transformations of the labels effectively correspond to introducing a mismatch between the target and model activation functions. A simple transformation that achieves this goal is label squaring. As shown in App.~\ref{app:squaring_pure_Hermite}, squaring a pure Hermite polynomial yields a function whose lowest nonzero Hermite component is $\He_2$, thereby reducing the effective IE to be at most $2$. We emphasize that such transformations are not intended to improve the final test mean squared error, but rather to reshape the optimization landscape during the early stage of learning; similar strategies have been employed in related contexts \cite{damian2024computational}. Accordingly, once the search phase is exited, we revert to the original labels in order to fully recover the target direction. This resulting two-stage learning strategy—using transformed, or ``softened'', labels early and switching back to hard labels thereafter—is widely used in practice, notably in semi-supervised and self-training methods such as FixMatch~\cite{sohn2020fixmatch}, as well as in curriculum learning approaches~\cite{bengio2009curriculum}.

Beyond pure Hermite activations, we also investigate how this transformation affects the phenomenology observed so far for activation functions with $\mathrm{IE}=1$. In the panel (c) of Fig.~\ref{fig:escaping_time_IE_1_matching_activation}, we observe that after applying this transformation to the labels generated by activation functions such as the linear $x$, ReLU, and $\textrm{erf}$, the behavior in the correlated search phase is qualitatively \emph{reversed} compared to the matching-activation setting of Fig.~\ref{fig:dynamics_linear_activation} (see right panel). In particular, for $x$ and $\textrm{erf}$, stronger alignment of the pre-trained weight with the target direction (larger $\mu$) now leads to \emph{faster} escape from the correlated search phase, in sharp contrast with the matching-activation case. Moreover, for $\textit{x}$, ReLU, and $\textrm{erf}$, the exit time no longer diverges as $\mu\to 1^{-}$, unlike what was observed previously. For the sigmoid activation, squaring the labels has little effect on $\tau(\mu)$, which is expected given that this transformation induces only a mild change in the effective shape of the activation function. Panel (d) of Fig.~\ref{fig:escaping_time_IE_1_matching_activation} shows that the general trend observed for $\mathrm{IE}\ge 2$ activation functions under this transformation is an \emph{overall decreasing} behavior of $\tau(\mu)$ as $\mu$ increases. Most notably, for odd Hermite activation functions, the singularities previously observed in the matching case disappear after squaring the labels, indicating that this transformation smooths the optimization landscape and restores efficient escape from the correlated search phase. Additional numerical evidence demonstrating how these observations can be leveraged to accelerate learning—and, in particular, to bypass singular values of $\mu$—is provided in the App.~\ref{app:numerics_use_square_label}.

\paragraph{Transformer LoRA experiments.}
To test whether the mechanism isolated by the teacher--student analysis also appears in more standard fine-tuning pipelines, we run controlled vision-transformer LoRA experiments. A small ViT is pretrained on a source classification task using AdamW. Checkpoints are extracted along the pretraining trajectory, after which the backbone is frozen and LoRA modules on the attention projections are trained on the downstream task. When the label space differs between source and downstream tasks, the classification head is reinitialized and trained jointly with the LoRA parameters. In the main text we consider EMNIST Letters$\to$FashionMNIST \cite{cohen2017emnist,xiao2017fashion}, which changes both input distribution and label space, and CIFAR-10$\to$SVHN \cite{krizhevsky2009learning,goodfellow2013multi}, a natural-image transfer task. Table~\ref{tab:transformer-lora-transfer} reports seed-matched results at fixed fine-tuning budgets and shows the same qualitative decoupling predicted by the theory: later checkpoints can have higher source accuracy while adapting worse downstream under the constrained LoRA update. Additional EMNIST$\to$shifted-EMNIST \cite{cohen2017emnist} experiments and full downstream loss/accuracy dynamics are reported in App.~\ref{app:vit-lora-experiments}. Simulations used an NVIDIA A100 for 25 GPU-hours.

\begin{table}[t]
\centering
\caption{
Source accuracy and downstream LoRA adaptation at fixed fine-tuning budgets. Entries are mean $\pm$ std over seed-matched runs. Bold downstream entries indicate the best downstream performance within each task block. Bold source entries highlight the source checkpoint comparison.
}
\begin{tabular}{llccc}
\toprule
Task & Source ckpt. & Source acc. & FT ep. 20 & FT ep. 100 \\
\midrule
EMNIST$\to$FashionMNIST & 002 & 87.52 $\pm$ 0.83 & 77.88 $\pm$ 0.13 & 78.97 $\pm$ 0.19 \\
EMNIST$\to$FashionMNIST & 005 & 91.13 $\pm$ 0.58 & \textbf{79.64 $\pm$ 0.43} & \textbf{80.24 $\pm$ 0.30} \\
EMNIST$\to$FashionMNIST & 100 & \textbf{93.30 $\pm$ 0.24} & 76.91 $\pm$ 0.65 & 77.49 $\pm$ 0.38 \\
\midrule CIFAR-10$\to$SVHN & 050 & 73.10 $\pm$ 0.00 & \textbf{27.38 $\pm$ 1.26} & \textbf{40.33 $\pm$ 1.05} \\
CIFAR-10$\to$SVHN & 100 & 75.26 $\pm$ 0.00 & 26.52 $\pm$ 1.37 & 37.46 $\pm$ 1.64 \\
CIFAR-10$\to$SVHN & 200 & \textbf{76.02 $\pm$ 0.00} & 23.94 $\pm$ 1.46 & 35.63 $\pm$ 1.62 \\
\bottomrule
\end{tabular}
\label{tab:transformer-lora-transfer}
\end{table}

\paragraph{Discussion:}
Overall, our results show that the pretrained signal can hinder, rather than accelerate, LoRA adaptation, with a physics-based explanation provided in App.~\ref{app:ODE_of_the_damp_harmonic_oxillator}. App.~\ref{role_of_the_rank_versus_signal} and Fig.~\ref{fig:dynamics_activation_two_layers_hermite2} extend the picture to two-layer networks, where the pretrained weight controls the search phase while rank mainly affects the final loss. Transformer experiments support the same qualitative mechanism beyond the solvable model, while a theory of attention-based LoRA dynamics is left for future work.


\section*{Acknowledgements}
We thank Kimia Nadjahi for insightful discussions. G.N. thanks Mauro Pastore for extensive discussions and insightful feedback throughout the development of this work. B.L. was supported by the French government, managed by the National Research Agency (ANR), under the France 2030 program with the project references ``ANR-23-IACL-0008'' (PR[AI]RIE-PSAI) and ``ANR-25-CE23-5660'' (MAPLE), as well as the Choose France - CNRS AI Rising Talents program. J.B. was funded by the European Union (ERC, CHORAL, project number 101039794). Views and opinions expressed are however those of the authors only and do not necessarily reflect those of the European Union or the European Research Council. Neither the European Union nor the granting authority can be held responsible for them.

\footnotesize{
\bibliography{references.bib}} 

@article{hu2022lora,
  title={Lora: Low-rank adaptation of large language models.},
  author={Hu, Edward J and Shen, Yelong and Wallis, Phillip and Allen-Zhu, Zeyuan and Li, Yuanzhi and Wang, Shean and Wang, Lu and Chen, Weizhu and others},
  journal={ICLR},
  volume={1},
  number={2},
  pages={3},
  year={2022}
}

@article{dettmers2023qlora,
  title={Qlora: Efficient finetuning of quantized llms},
  author={Dettmers, Tim and Pagnoni, Artidoro and Holtzman, Ari and Zettlemoyer, Luke},
  journal={Advances in neural information processing systems},
  volume={36},
  pages={10088--10115},
  year={2023}
}

@inproceedings{
isik2025scaling,
title={Scaling Laws for Downstream Task Performance in Machine Translation},
author={Berivan Isik and Natalia Ponomareva and Hussein Hazimeh and Dimitris Paparas and Sergei Vassilvitskii and Sanmi Koyejo},
booktitle={The Thirteenth International Conference on Learning Representations},
year={2025},
url={https://openreview.net/forum?id=vPOMTkmSiu}
}

@article{barbier2025generalization,
  title={Generalization performance of narrow one-hidden layer networks in the teacher-student setting},
  author={Barbier, Jean and Gerace, Federica and Ingrosso, Alessandro and Lauditi, Clarissa and Malatesta, Enrico M and Nwemadji, Gibbs and Ortiz, Rodrigo P{\'e}rez},
  journal={arXiv preprint arXiv:2507.00629},
  year={2025}
}

@article{gerace2024gaussian,
  title={Gaussian universality of perceptrons with random labels},
  author={Gerace, Federica and Krzakala, Florent and Loureiro, Bruno and Stephan, Ludovic and Zdeborov{\'a}, Lenka},
  journal={Physical Review E},
  volume={109},
  number={3},
  pages={034305},
  year={2024},
  publisher={APS}
}

@article{arnaboldi2024repetita,
  title={Repetita iuvant: Data repetition allows sgd to learn high-dimensional multi-index functions},
  author={Arnaboldi, Luca and Dandi, Yatin and Krzakala, Florent and Pesce, Luca and Stephan, Ludovic},
  journal={arXiv preprint arXiv:2405.15459},
  year={2024}
}

@article{hayou2024impact,
  title={The impact of initialization on lora finetuning dynamics},
  author={Hayou, Soufiane and Ghosh, Nikhil and Yu, Bin},
  journal={Advances in Neural Information Processing Systems},
  volume={37},
  pages={117015--117040},
  year={2024}
}

@article{yang2024low,
  title={Low-rank adaptation for foundation models: A comprehensive review},
  author={Yang, Menglin and Chen, Jialin and Tao, Jinkai and Zhang, Yifei and Liu, Jiahong and Zhang, Jiasheng and Ma, Qiyao and Verma, Harshit and Zhang, Regina and Zhou, Min and others},
  journal={arXiv preprint arXiv:2501.00365},
  year={2024}
}

@article{zeng2023expressive,
  title={The expressive power of low-rank adaptation},
  author={Zeng, Yuchen and Lee, Kangwook},
  journal={arXiv preprint arXiv:2310.17513},
  year={2023}
}

@article{mu2025convergence,
  title={On the Convergence Rate of LoRA Gradient Descent},
  author={Mu, Siqiao and Klabjan, Diego},
  journal={arXiv preprint arXiv:2512.18248},
  year={2025}
}

@article{kratsios2025sharp,
  title={Sharp Generalization Bounds for Foundation Models with Asymmetric Randomized Low-Rank Adapters},
  author={Kratsios, Anastasis and Cheng, Tin Sum and Lucchi, Aurelien and Borde, Haitz S{\'a}ez de Oc{\'a}riz},
  journal={arXiv preprint arXiv:2506.14530},
  year={2025}
}

@article{liang2025does,
  title={Does Low Rank Adaptation Lead to Lower Robustness against Training-Time Attacks?},
  author={Liang, Zi and Hu, Haibo and Ye, Qingqing and Xiao, Yaxin and Li, Ronghua},
  journal={arXiv preprint arXiv:2505.12871},
  year={2025}
}

@article{kim2025lora,
  title={LoRA Training Provably Converges to a Low-Rank Global Minimum or It Fails Loudly (But it Probably Won't Fail)},
  author={Kim, Junsu and Kim, Jaeyeon and Ryu, Ernest K},
  journal={arXiv preprint arXiv:2502.09376},
  year={2025}
}

@article{xu2025understanding,
  title={Understanding the Learning Dynamics of LoRA: A Gradient Flow Perspective on Low-Rank Adaptation in Matrix Factorization},
  author={Xu, Ziqing and Min, Hancheng and MacDonald, Lachlan Ewen and Luo, Jinqi and Tarmoun, Salma and Mallada, Enrique and Vidal, Ren{\'e}},
  journal={arXiv preprint arXiv:2503.06982},
  year={2025}
}

@article{zhang2025lora,
  title={LoRA-One: One-Step Full Gradient Could Suffice for Fine-Tuning Large Language Models, Provably and Efficiently},
  author={Zhang, Yuanhe and Liu, Fanghui and Chen, Yudong},
  journal={arXiv preprint arXiv:2502.01235},
  year={2025}
}

@article{mao2025survey,
  title={A survey on lora of large language models},
  author={Mao, Yuren and Ge, Yuhang and Fan, Yijiang and Xu, Wenyi and Mi, Yu and Hu, Zhonghao and Gao, Yunjun},
  journal={Frontiers of Computer Science},
  volume={19},
  number={7},
  pages={197605},
  year={2025},
  publisher={Springer}
}

@inproceedings{houlsby2019parameter,
  title={Parameter-efficient transfer learning for NLP},
  author={Houlsby, Neil and Giurgiu, Andrei and Jastrzebski, Stanislaw and Morrone, Bruna and De Laroussilhe, Quentin and Gesmundo, Andrea and Attariyan, Mona and Gelly, Sylvain},
  booktitle={International conference on machine learning},
  pages={2790--2799},
  year={2019},
  organization={PMLR}
}

@article{barbier2019optimal,
  title={Optimal errors and phase transitions in high-dimensional generalized linear models},
  author={Barbier, Jean and Krzakala, Florent and Macris, Nicolas and Miolane, L{\'e}o and Zdeborov{\'a}, Lenka},
  journal={Proceedings of the National Academy of Sciences},
  volume={116},
  number={12},
  pages={5451--5460},
  year={2019},
  publisher={National Academy of Sciences}
}

@article{gardner1989three,
  title={Three unfinished works on the optimal storage capacity of networks},
  author={Gardner, Elizabeth and Derrida, Bernard},
  journal={Journal of Physics A: Mathematical and General},
  volume={22},
  number={12},
  pages={1983},
  year={1989},
  publisher={IOP Publishing}
}

@inproceedings{kalai2009isotron,
  title={The Isotron Algorithm: High-Dimensional Isotonic Regression.},
  author={Kalai, Adam Tauman and Sastry, Ravi},
  booktitle={COLT},
  volume={1},
  pages={9},
  year={2009}
}

@article{han2021pre,
  title={Pre-trained models: Past, present and future},
  author={Han, Xu and Zhang, Zhengyan and Ding, Ning and Gu, Yuxian and Liu, Xiao and Huo, Yuqi and Qiu, Jiezhong and Yao, Yuan and Zhang, Ao and Zhang, Liang and others},
  journal={Ai Open},
  volume={2},
  pages={225--250},
  year={2021},
  publisher={Elsevier}
}

@article{qiu2020pre,
  title={Pre-trained models for natural language processing: A survey},
  author={Qiu, Xipeng and Sun, Tianxiang and Xu, Yige and Shao, Yunfan and Dai, Ning and Huang, Xuanjing},
  journal={Science China technological sciences},
  volume={63},
  number={10},
  pages={1872--1897},
  year={2020},
  publisher={Springer}
}

@article{jang2024lora,
  title={LoRA training in the NTK regime has no spurious local minima},
  author={Jang, Uijeong and Lee, Jason D and Ryu, Ernest K},
  journal={arXiv preprint arXiv:2402.11867},
  year={2024}
}

@article{xu2023parameter,
  title={Parameter-efficient fine-tuning methods for pretrained language models: A critical review and assessment},
  author={Xu, Lingling and Xie, Haoran and Qin, Si-Zhao Joe and Tao, Xiaohui and Wang, Fu Lee},
  journal={arXiv preprint arXiv:2312.12148},
  year={2023}
}

@article{arous2021online,
  title={Online stochastic gradient descent on non-convex losses from high-dimensional inference},
  author={Ben Arous, Gerard and Gheissari, Reza and Jagannath, Aukosh},
  journal={Journal of Machine Learning Research},
  volume={22},
  number={106},
  pages={1--51},
  year={2021}
}

@article{damian2023smoothing,
  title={Smoothing the landscape boosts the signal for sgd: Optimal sample complexity for learning single index models},
  author={Damian, Alex and Nichani, Eshaan and Ge, Rong and Lee, Jason D},
  journal={Advances in Neural Information Processing Systems},
  volume={36},
  pages={752--784},
  year={2023}
}

@article{damian2024computational,
  title={Computational-statistical gaps in gaussian single-index models},
  author={Damian, Alex and Pillaud-Vivien, Loucas and Lee, Jason D and Bruna, Joan},
  journal={arXiv preprint arXiv:2403.05529},
  year={2024}
}

@article{sohn2020fixmatch,
  title={Fixmatch: Simplifying semi-supervised learning with consistency and confidence},
  author={Sohn, Kihyuk and Berthelot, David and Carlini, Nicholas and Zhang, Zizhao and Zhang, Han and Raffel, Colin A and Cubuk, Ekin Dogus and Kurakin, Alexey and Li, Chun-Liang},
  journal={Advances in neural information processing systems},
  volume={33},
  pages={596--608},
  year={2020}
}

@book{mezard1987spin,
  title={Spin glass theory and beyond: An Introduction to the Replica Method and Its Applications},
  author={M{\'e}zard, Marc and Parisi, Giorgio and Virasoro, Miguel Angel},
  volume={9},
  year={1987},
  publisher={World Scientific Publishing Company}
}

@inproceedings{choromanska2015loss,
  title={The loss surfaces of multilayer networks},
  author={Choromanska, Anna and Henaff, Mikael and Mathieu, Michael and Arous, G{\'e}rard Ben and LeCun, Yann},
  booktitle={Artificial intelligence and statistics},
  pages={192--204},
  year={2015},
  organization={PMLR}
}

@article{mehta2019high,
  title={A high-bias, low-variance introduction to machine learning for physicists},
  author={Mehta, Pankaj and Bukov, Marin and Wang, Ching-Hao and Day, Alexandre GR and Richardson, Clint and Fisher, Charles K and Schwab, David J},
  journal={Physics reports},
  volume={810},
  pages={1--124},
  year={2019},
  publisher={Elsevier}
}

@inproceedings{
springer2025overtrained,
title={Overtrained Language Models Are Harder to Fine-Tune},
author={Jacob Mitchell Springer and Sachin Goyal and Kaiyue Wen and Tanishq Kumar and Xiang Yue and Sadhika Malladi and Graham Neubig and Aditi Raghunathan},
booktitle={Forty-second International Conference on Machine Learning},
year={2025},
url={https://openreview.net/forum?id=YW6edSufht}
}

@inproceedings{bengio2009curriculum,
  title={Curriculum learning},
  author={Bengio, Yoshua and Louradour, J{\'e}r{\^o}me and Collobert, Ronan and Weston, Jason},
  booktitle={Proceedings of the 26th annual international conference on machine learning},
  pages={41--48},
  year={2009}
}

@inproceedings{damian2022neural,
  title={Neural networks can learn representations with gradient descent},
  author={Damian, Alexandru and Lee, Jason and Soltanolkotabi, Mahdi},
  booktitle={Conference on Learning Theory},
  pages={5413--5452},
  year={2022},
  organization={PMLR}
}

@article{ba2022high,
  title={High-dimensional asymptotics of feature learning: How one gradient step improves the representation},
  author={Ba, Jimmy and Erdogdu, Murat A and Suzuki, Taiji and Wang, Zhichao and Wu, Denny and Yang, Greg},
  journal={Advances in Neural Information Processing Systems},
  volume={35},
  pages={37932--37946},
  year={2022}
}

@InProceedings{cui24asymptotics,
  title = 	 {Asymptotics of feature learning in two-layer networks after one gradient-step},
  author =       {Cui, Hugo and Pesce, Luca and Dandi, Yatin and Krzakala, Florent and Lu, Yue and Zdeborova, Lenka and Loureiro, Bruno},
  booktitle = 	 {Proceedings of the 41st International Conference on Machine Learning},
  pages = 	 {9662--9695},
  year = 	 {2024},
  editor = 	 {Salakhutdinov, Ruslan and Kolter, Zico and Heller, Katherine and Weller, Adrian and Oliver, Nuria and Scarlett, Jonathan and Berkenkamp, Felix},
  volume = 	 {235},
  series = 	 {Proceedings of Machine Learning Research},
  month = 	 {21--27 Jul},
  publisher =    {PMLR},
}

@article{dandi2024two,
  title={How two-layer neural networks learn, one (giant) step at a time},
  author={Dandi, Yatin and Krzakala, Florent and Loureiro, Bruno and Pesce, Luca and Stephan, Ludovic},
  journal={Journal of Machine Learning Research},
  volume={25},
  number={349},
  pages={1--65},
  year={2024}
}

@article{robbins1951stochastic,
  title={A stochastic approximation method},
  author={Robbins, Herbert and Monro, Sutton},
  journal={The annals of mathematical statistics},
  pages={400--407},
  year={1951},
  publisher={JSTOR}
}

@inproceedings{arnaboldi2023high,
  title={From high-dimensional \& mean-field dynamics to dimensionless odes: A unifying approach to sgd in two-layers networks},
  author={Arnaboldi, Luca and Stephan, Ludovic and Krzakala, Florent and Loureiro, Bruno},
  booktitle={The Thirty Sixth Annual Conference on Learning Theory},
  pages={1199--1227},
  year={2023},
  organization={PMLR}
}

@article{veiga2022phase,
  title={Phase diagram of stochastic gradient descent in high-dimensional two-layer neural networks},
  author={Veiga, Rodrigo and Stephan, Ludovic and Loureiro, Bruno and Krzakala, Florent and Zdeborov{\'a}, Lenka},
  journal={Advances in Neural Information Processing Systems},
  volume={35},
  pages={23244--23255},
  year={2022}
}

@article{mori2025optimal,
  title={Optimal protocols for continual learning via statistical physics and control theory},
  author={Mori, Francesco and Mannelli, Stefano Sarao and Mignacco, Francesca},
  journal={Journal of Statistical Mechanics: Theory and Experiment},
  volume={2025},
  number={8},
  pages={084004},
  year={2025},
  publisher={IOP Publishing}
}

@article{dandi2024benefits,
  title={The benefits of reusing batches for gradient descent in two-layer networks: Breaking the curse of information and leap exponents},
  author={Dandi, Yatin and Troiani, Emanuele and Arnaboldi, Luca and Pesce, Luca and Zdeborov{\'a}, Lenka and Krzakala, Florent},
  journal={arXiv preprint arXiv:2402.03220},
  year={2024}
}

@article{arnaboldi2023escaping,
  title={Escaping mediocrity: how two-layer networks learn hard single-index models with SGD},
  author={Arnaboldi, Luca and Krzakala, Florent and Loureiro, Bruno and Stephan, Ludovic},
  journal={CoRR},
  year={2023}
}

@article{soletskyi2025theoretical,
  title={A theoretical perspective on mode collapse in variational inference},
  author={Soletskyi, Roman and Gabri{\'e}, Marylou and Loureiro, Bruno},
  journal={Machine Learning: Science and Technology},
  volume={6},
  number={2},
  pages={025056},
  year={2025},
  publisher={IOP Publishing}
}

@article{saglietti2022analytical,
  title={An analytical theory of curriculum learning in teacher-student networks},
  author={Saglietti, Luca and Mannelli, Stefano and Saxe, Andrew},
  journal={Advances in Neural Information Processing Systems},
  volume={35},
  pages={21113--21127},
  year={2022}
}

@article{goldt2020modeling,
  title={Modeling the influence of data structure on learning in neural networks: The hidden manifold model},
  author={Goldt, Sebastian and M{\'e}zard, Marc and Krzakala, Florent and Zdeborov{\'a}, Lenka},
  journal={Physical Review X},
  volume={10},
  number={4},
  pages={041044},
  year={2020},
  publisher={APS}
}

@article{saad1996learning,
  title={Learning with noise and regularizers in multilayer neural networks},
  author={Saad, David and Solla, Sara},
  journal={Advances in Neural Information Processing Systems},
  volume={9},
  year={1996}
}

@article{goldt2019dynamics,
  title={Dynamics of stochastic gradient descent for two-layer neural networks in the teacher-student setup},
  author={Goldt, Sebastian and Advani, Madhu and Saxe, Andrew M and Krzakala, Florent and Zdeborov{\'a}, Lenka},
  journal={Advances in neural information processing systems},
  volume={32},
  year={2019}
}

@article{xiao2017fashion,
  title={Fashion-mnist: a novel image dataset for benchmarking machine learning algorithms},
  author={Xiao, Han and Rasul, Kashif and Vollgraf, Roland},
  journal={arXiv preprint arXiv:1708.07747},
  year={2017}
}

@inproceedings{cohen2017emnist,
  title={EMNIST: Extending MNIST to handwritten letters},
  author={Cohen, Gregory and Afshar, Saeed and Tapson, Jonathan and Van Schaik, Andre},
  booktitle={2017 international joint conference on neural networks (IJCNN)},
  pages={2921--2926},
  year={2017},
  organization={IEEE}
}

@article{krizhevsky2009learning,
  title={Learning multiple layers of features from tiny images},
  author={Krizhevsky, Alex and Hinton, Geoffrey and others},
  year={2009},
  publisher={Toronto, ON, Canada}
}

@article{goodfellow2013multi,
  title={Multi-digit number recognition from street view imagery using deep convolutional neural networks},
  author={Goodfellow, Ian J and Bulatov, Yaroslav and Ibarz, Julian and Arnoud, Sacha and Shet, Vinay},
  journal={arXiv preprint arXiv:1312.6082},
  year={2013}
}

\newpage

\appendix

\section{Recap setting}

We are in the following setup:
\begin{itemize}
    \item \textbf{Data:} Data are generated by a noiseless model:
    \begin{equation}
    \label{equ:teacher}
        y_i = \phi(\omeg_\star \cdot \boldx_i),
    \end{equation}
    with $\boldx_i \sim \mathcal{N}(0, I_d)$, $\omeg_\star \in \mathbb{S}^{d-1}$, and $\phi(x)$ a generic activation function.

    \item \textbf{Model:} We consider the following simplified model for LoRA:
    \begin{equation}
    \label{equ:student}
        \hat{y}_i = f(\boldx,u,\omeg) = \sigma\big( (\tilde{\omeg} + u\omeg)\cdot \boldx_i \big)
    \end{equation}
    with
    \begin{equation}
    \label{equ:pre_trained_weight}
        \tilde{\omeg} = \mu \omeg_\star,
    \end{equation}
    where $\mu \in (0,1)$; the trainable parameters are $u \in \mathbb{R}$ and $\omeg \in \mathbb{S}^{d-1}$.

    \item \textbf{Overlap:} We define the overlaps as follows:
    \begin{equation}
    \label{equ:overlap_r_m_eff}
        m = \omeg_\star \cdot \omeg, 
        \quad m_{\rm eff}=\mu+um, \quad r=\norm{\tilde{\omeg}+u\omeg}^2=\mu^2+u^2+2\mu um.
    \end{equation}
    An interesting quantity that we can also track is
    \begin{equation}
        m_{\rm eff} = \omeg_\star \cdot (\mu \omeg_\star + u \omeg) = \mu + u m,
    \end{equation}
    which, in the 1D tensor case, is sufficient to quantify how close the combination $\mu \omeg_\star + u \omeg$ is to $\omeg_\star$. Notably, at the global minimum—regardless of the activation function or the value of $\mu$—this quantity always equals $1$ (or $\pm 1$ for even activations) in the scenario $\phi(\cdot)=\sigma(\cdot)$.

    \item \textbf{Loss:} The population loss under consideration is the mean square loss given by
    \begin{equation}
        \mathcal{L}(u,\omeg) 
        = \frac{1}{2} \mathbb{E}_{x}\Big[\big( y - \hat{y} \big)^2\Big]
        = \frac{1}{2} \mathbb{E}_{x}\Big[\Big( \phi( \omeg_\star \cdot \boldx_i ) 
        - \sigma\big( (\mu \omeg_\star + u \omeg)\cdot \boldx_i \big) \Big)^2\Big],
    \end{equation}
    with
    \[
    L(u,\omeg)
    = \frac{1}{2n}\sum_{i=1}^{n}
    \Big( \phi( \omeg_\star \cdot \boldx_i )
    - \sigma\big( (\mu \omeg_\star + u \omeg)\cdot \boldx_i \big) \Big)^2
    \]
    being the empirical loss.
\end{itemize}

We are interested in understanding how the activation functions $\phi(x), \sigma(x)$ and the signal strength $\mu$ affect the exit time, namely the time required to exit the search phase. For our analysis, we will consider generic activation functions, for which their Hermite expansions will be exploited subsequently. We investigate correlation-based losses in App.~\ref{app:different_loss} and show that, for linear activations, pre-trained weights do not affect the dynamics, highlighting that the impact of pre-training depends critically on the choice of loss.

\subsection{Another model of pre-training}
\label{app_different_pre-trained weight}

Another natural candidate for a pre-trained weight is
\begin{equation}
\label{equ_different_weight}
\tilde{\omeg}=\mu\,\omeg_\star+(1-\mu)\,\boldsymbol{\xi},
\end{equation}
where $\boldsymbol{\xi}$ is a random unit vector independent of $\boldsymbol{\omega}_\star$.
In this setting, the effective overlap takes the form
\begin{align}
\label{equ:m_eff_diff}
m_{\mathrm{eff}}&=\mu+(1-\mu)\,\omeg_\star \cdot \boldsymbol{\xi}+u m
=\mu+u m+O(d^{-1/2}),\\
\nonumber
r&:=\norm{\mu\omeg_\star+(1-\mu)\,\boldsymbol{\xi}+u\omeg}^2
=\mu^2+(1-\mu)^2+u^2+2\mu(1-\mu)\omeg_\star \cdot \boldsymbol{\xi}
+2\mu u m
+2(1-\mu)u \omeg \cdot \boldsymbol{\xi},\\
\label{equ:r_eff_diff}
r&=\mu^2+(1-\mu)^2+u^2+2\mu u m+O(d^{-1/2}),
\end{align}
where the $O(d^{-1/2})$ term arises from the high-dimensional limit: with high probability, the dot product of two random vectors is of order $O(d^{-1/2})$, and these vectors therefore remain orthogonal throughout the dynamics.

From \eqref{equ:m_eff_diff}, we see that, at the level of the overlap $m_{\rm eff}$,  \eqref{equ:pre_trained_weight} is equivalent to \eqref{equ_different_weight}. However, the norm of the pre-activation is different from that of \eqref{equ:pre_trained_weight}, as there is an additional $(1-\mu)^2$ term in this second model (see \eqref{equ:overlap_r_m_eff}). Since $r^2$ also drives the dynamics, computing the gradient with respect to $u$ or $m$ shows that this additional term has no effect on the dynamics of $u$ and $m$. Therefore, beyond a trivial shift of the pre-activation
norm, this model provides no additional insight into the phenomenology observed with \eqref{equ:pre_trained_weight}. We illustrate this by providing in Figure~\ref{fig:dynamics_activation_different_pre-trained}, numerical evidence for a student model with pre-trained weights designed according to~\eqref{equ_different_weight}; details of the experimental hyperparameters are provided in the caption. The key takeaway is that the same phenomenology is observed as for the student model~\eqref{equ:student}.
\begin{figure}
    \centering
    \vspace{0.0em}
    \begin{subfigure}
        \centering
        \includegraphics[width=0.9\textwidth]{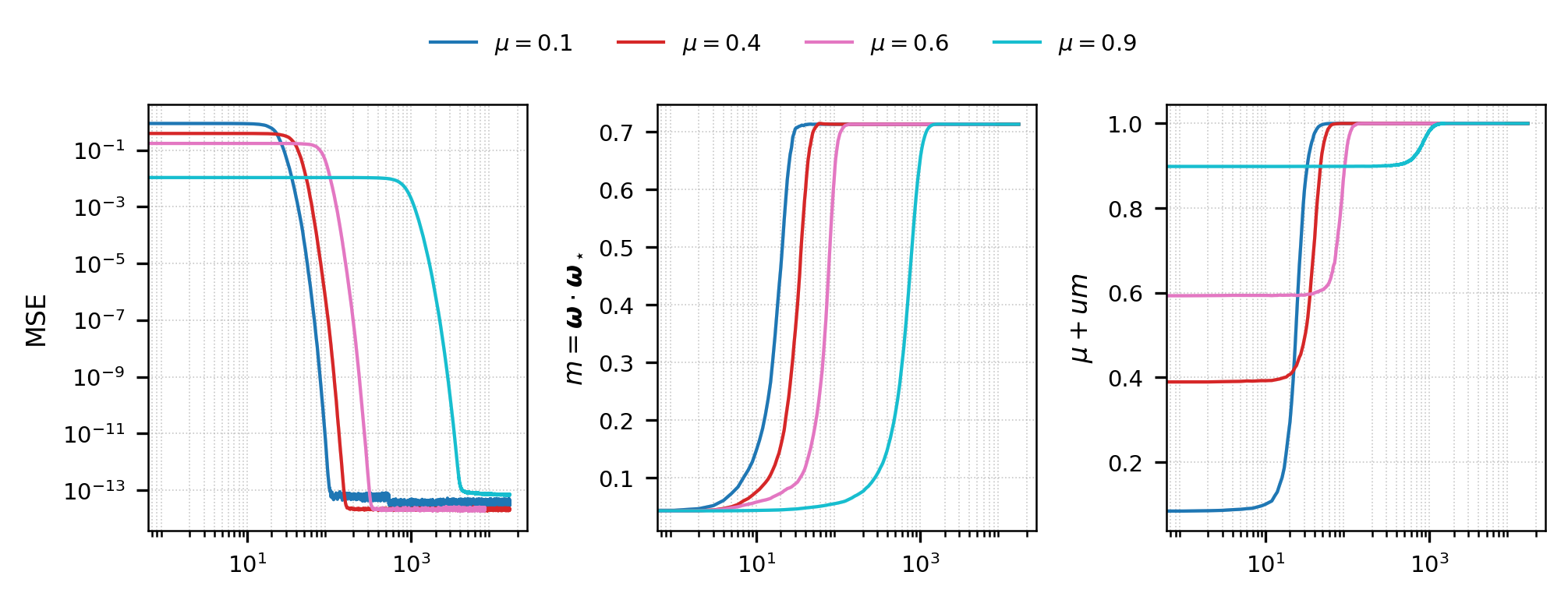}
        \label{fig:diff_teacher_dynamics_linear_activation_appendix}
    \end{subfigure}
    \vspace{-1.0em}
    \begin{subfigure}
        \centering
        \includegraphics[width=0.9\textwidth]{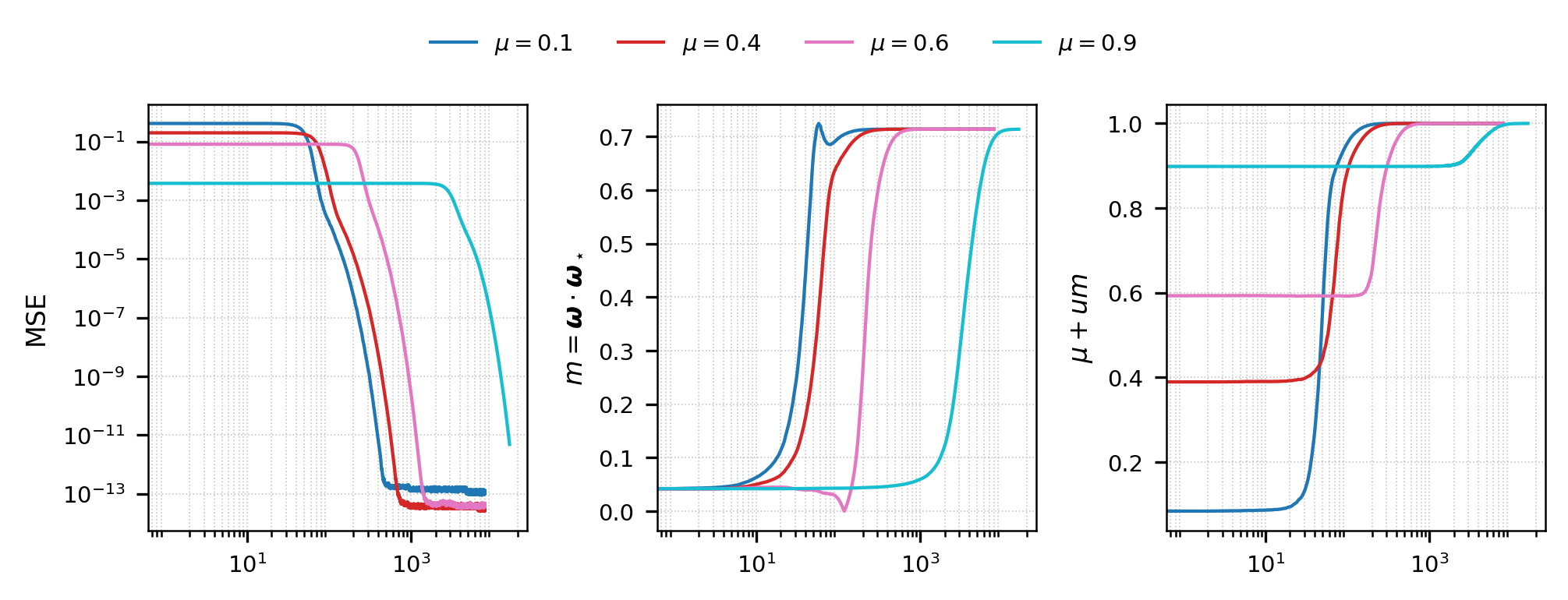}
        \label{fig:diff_teacher_dynamics_erf_activation_appendix}
    \end{subfigure}
    \vspace{-1.0em}
    \begin{subfigure}
        \centering
        \includegraphics[width=0.9\textwidth]{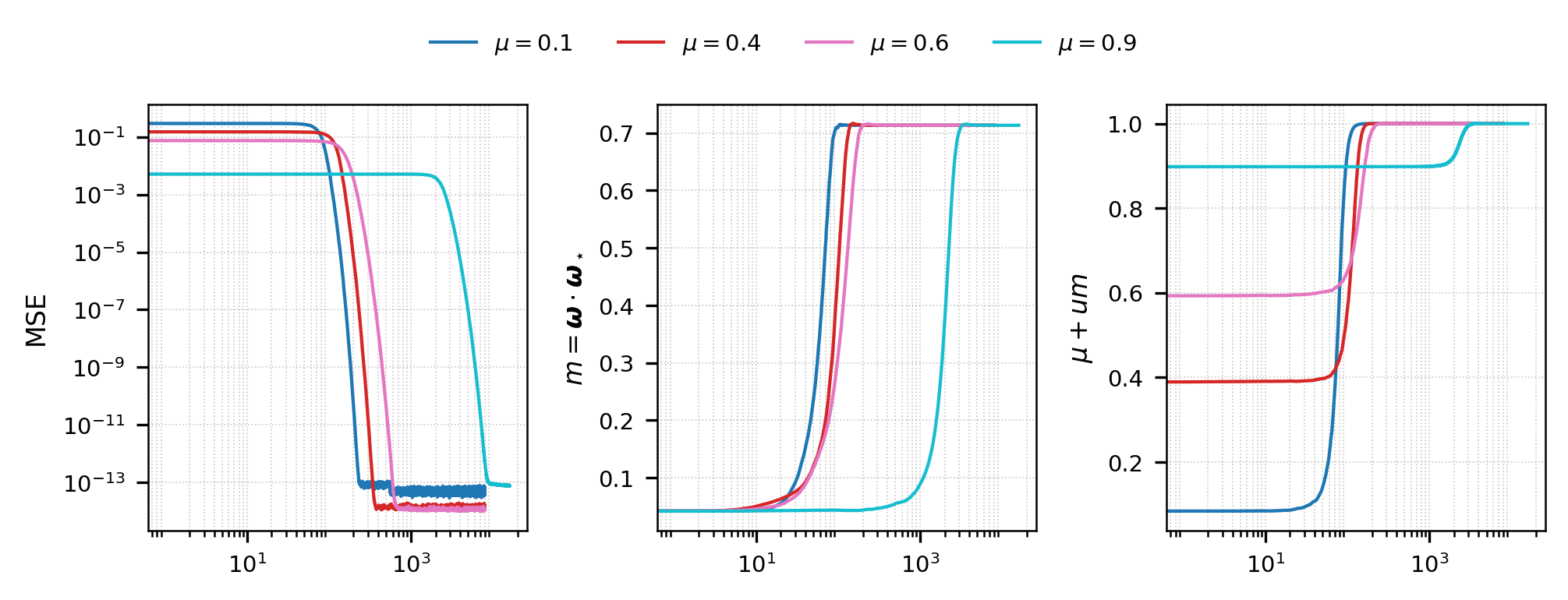}
        \label{fig:diff_dynamics_relu_activation_appendix}
    \end{subfigure}
    \caption{Learning dynamics of the student model~\eqref{equ_different_weight} for different levels of pre-training alignment $\mu \in \{0.1, 0.4, 0.6, 0.9\}$, trained on data generated by the teacher model~\eqref{equ:teacher}. We consider the matched teacher--student Linear, Erf, and ReLU (from top to bottom) activation setting and train the student using one-pass SGD with batch size $B=5000$ and input dimension $d=1000$. The panels report the test mean squared error (MSE) (left), the alignment between the student and teacher directions $m=\boldsymbol{\omega}_\star\!\cdot\!\boldsymbol{\omega}$ (middle), and the effective teacher--student overlap $m_{\mathrm{eff}}=\mu+u m$ (right).
Shaded regions indicate one standard deviation over three independent runs.
The spherical constraint is enforced by normalizing $\boldsymbol{\omega}$ after each gradient step and we used  a learning rate $lr=0.2 .$}
    \label{fig:dynamics_activation_different_pre-trained}
\end{figure}

\section{Hermite expansion}
\label{app:generalize_Hermite}

\subsection{Properties of scaled Hermite polynomials}
\label{app:Hermite_properties}

Let $z\sim\mathcal N(0,r)$ with $r>0$. We denote by $\{\mathrm{He}^{[r]}_k\}_{k\ge0}$ the probabilists' Hermite polynomials orthogonal with respect to the Gaussian measure $\mathcal N(0,r)$.

\paragraph{Definition.}
The scaled Hermite polynomials are defined through the generating function
\begin{equation}
\label{app:property_Hermite definition}
\exp\!\left(zt-\frac{r t^2}{2}\right)
=\sum_{k=0}^{\infty}\frac{t^k}{k!}\,\mathrm{He}^{[r]}_\ell(z).
\end{equation}

\paragraph{Orthogonality.}
For all $k,m\ge0$,
\begin{equation}
\label{app:property_Hermite orthogonality}
\mathbb E_{z\sim\mathcal N(0,r)}
\!\left[\mathrm{He}^{[r]}_k(z)\,\mathrm{He}^{[r]}_m(z)\right]
=\delta_{k m}\,k!\,r^k .
\end{equation}

\paragraph{Zero mean.}
All non-constant Hermite polynomials have zero expectation:
\begin{equation}
\label{app:property_Hermite zero_mean}
\mathbb E_{z\sim\mathcal N(0,r)}\!\left[\mathrm{He}^{[r]}_k(z)\right]=0
\quad\text{for all }k\ge1,
\qquad
\mathrm{He}^{[r]}_0(z)=1.
\end{equation}

\paragraph{Derivative identity.}
Hermite polynomials satisfy the lowering relation
\begin{equation}
\label{app:property_Hermite derivative}
\frac{d}{dz}\,\mathrm{He}^{[r]}_k(z)
= k\,\mathrm{He}^{[r]}_{k-1}(z).
\end{equation}

\paragraph{Multiplication by the coordinate.}
Multiplication by $z$ raises and lowers the Hermite degree:
\begin{equation}
\label{app:property_Hermite multiplication_constant}
z\,\mathrm{He}^{[r]}_k(z)
= \mathrm{He}^{[r]}_{k+1}(z)
+ k\,r\,\mathrm{He}^{[r]}_{k-1}(z).
\end{equation}

\paragraph{Scaling relation.}
If $z=\sqrt r\,x$ with $x\sim\mathcal N(0,1)$, then
\begin{equation}
\label{app:property_Hermite scaling relation}
\mathrm{He}^{[r]}_k(z)
= r^{k/2}\,\mathrm{He}_k(x),
\end{equation}
where $\mathrm{He}_k$ denotes the standard probabilists' Hermite polynomial.

\paragraph{Ornstein--Uhlenbeck eigenfunctions.}
Define the Ornstein--Uhlenbeck operator
\begin{equation}
\label{app:property_Hermite eigenfunction}
\mathcal L_r f(z) := r f''(z) - z f'(z).
\end{equation}
Then Hermite polynomials are eigenfunctions of $\mathcal L_r$:
\begin{equation}
\mathcal L_r\,\mathrm{He}^{[r]}_\ell(z)
= -\ell\,\mathrm{He}^{[r]}_\ell(z).
\end{equation}

\paragraph{Parity.}
Hermite polynomials have definite parity:
\begin{equation}
\label{app:property_Hermite parity}
\mathrm{He}^{[r]}_\ell(-z)=(-1)^\ell\,\mathrm{He}^{[r]}_\ell(z).
\end{equation}
As a consequence, even (resp.\ odd) functions have vanishing odd (resp.\ even) Hermite coefficients.

\paragraph{Hermite expansion of square-integrable functions.}
Any function $\sigma\in L^2(\mathcal N(0,r))$ admits the expansion
\begin{equation}
\label{app:property_Hermite expansion}
\sigma(z)=\sum_{k=0}^\infty
\frac{\sigma^{[r]}_k}{k!\,r^k}\,
\mathrm{He}^{[r]}_k(z),
\qquad
\sigma^{[r]}_k
:=\mathbb E\!\left[\sigma(z)\,\mathrm{He}^{[r]}_k(z)\right].
\end{equation}

\paragraph{Gaussian expectation.}
Only the zeroth Hermite coefficient contributes to the mean:
\begin{equation}
\label{app:property_Hermite expectation}
\mathbb E_{z\sim\mathcal N(0,r)}[\sigma(z)] = \sigma^{[r]}_0.
\end{equation}

\paragraph{Interpretation.}
The family $\{\mathrm{He}^{[r]}_k\}_{k\ge0}$ forms an orthogonal basis of $L^2(\mathcal N(0,r))$, often referred to as the Wiener chaos decomposition. Each Hermite degree corresponds to an independent mode under Gaussian expectations and dynamics.

\subsection{Hermite expansion of a generic activation function}
\label{app:property_Hermite}

The Hermite expansion of $\phi(\cdot)$ with preactivation $z\sim\mathcal{N}(0,1)$ is
\begin{equation}
\label{app:property_Hermite unit_variance_preactivation}
\phi(z)=\sum_{k\ge 0}\frac{\phi_k}{k!}\,\He_k(z)
\qquad \Rightarrow \qquad
\E[\phi(z)^2]=\int Dz\,\phi(z)^2=\sum_{k=0}^{\infty}\frac{\phi_k^2}{k!},
\end{equation}
with
\begin{equation*}
Dz:=\diff z\,\exp(-z^2/2)/\sqrt{2\pi}.
\end{equation*}

Now consider $z_1\sim\mathcal{N}(0,r_1)$ and $z_2\sim\mathcal{N}(0,r_2)$.
From the Mehler expansion, one gets
\begin{equation}
\EE[\sigma(z_1)\sigma(z_2)]
= \sum_{k=0}^{\infty}
\frac{\sigma_k^{[r_1]}\sigma_k^{[r_2]}}{k!\,r_{1}^{k}r_{2}^{k}}\,
h^k,
\qquad \text{with}\qquad
h=\EE[z_1 z_2].
\end{equation}

As a warm-up, let us compute how the population loss depends on the summary statistic $m$:
\begin{equation}
\label{equ:population_loss_app}
\mathcal{L}(\omeg)
= \frac{1}{2}\,\mathbb{E}\Big[\big(\phi(\omeg_{\star}\cdot \boldx) - \sigma\big((\mu \omeg_{\star} + u\omeg)\cdot \boldx\big)\big)^2\Big]
= \frac{1}{2}\left( \mathbb{E}[\phi(\lambda_{\star})^2] + \mathbb{E}[\sigma(\mu\lambda_{\star} + u\lambda)^2] \right)
- \mathbb{E}[\phi(\lambda_{\star}) \sigma(\mu\lambda_{\star} + u\lambda)],
\end{equation}
where we recall $\lambda_{\star}=\omeg_{\star}\cdot \boldx$ and $\lambda=\omeg\cdot \boldx$.

Now let us recall the distribution of the different preactivations:
\begin{equation*}
\lambda_{\star}\sim\mathcal{N}(0,1),
\qquad
\mu \lambda_{\star} + u\lambda \sim \mathcal{N}(0,r),
\qquad
\EE[\lambda_{\star}(\mu\lambda_{\star} + u\lambda)] = \mu + u m,
\end{equation*}
with $m=\omeg_{\star}\cdot \omeg$ and $r=\mu^2 + u^2 + 2\mu u m$.

One then exploits the previous property to obtain
\begin{align*}
\EE[\phi(\lambda_{\star})^2]
&= \sum_{k=0}^{+\infty}\frac{\phi_k^2}{k!},\quad
\EE[\sigma(\mu\lambda_{\star} + u\lambda)^2]
= \sum_{k=0}^{+\infty}\frac{\big(\sigma^{[r]}_k\big)^2}{k!\,r^k},\quad
\EE[\phi(\lambda_{\star})\sigma(\mu\lambda_{\star} + u\lambda)]
= \sum_{k=0}^{+\infty}\frac{\phi_{k}\sigma^{[r]}_{k}}{k!\,r^k}\,(\mu + u m)^k.
\end{align*}
The population loss becomes
\begin{equation}
\mathcal{L}(u,m)= \sum_{k=0}^{+\infty}\frac{\phi_k^2}{2k!}+ \sum_{k=0}^{+\infty}\frac{\big(\sigma^{[r]}_k\big)^2}{2k!\,r^{k}}-\sum_{k=0}^{+\infty}\frac{\sigma^{[r]}_k \phi_k}{k!\,r^{k}}\,(\mu+u m)^k= \sum_{k=0}^{+\infty}\frac{\phi_k^2}{2k!} + \mathcal{L}_{1}(u,m)+\mathcal{L}_{2}(u,m).
\end{equation}

\subsection{Expansion for non-unit variance pre-activation in the standard Hermite basis}
\label{appen:non-unit variance}

In Appendix~\ref{app:property_Hermite}, we derived in
\eqref{app:property_Hermite expansion} the Hermite expansion of an
activation function under a Gaussian pre-activation with zero mean and
non-unit variance, expressed in the Hermite basis associated with the
corresponding Gaussian measure. Since the information exponent (IE) is defined
with respect to the standard Hermite basis associated with a unit-variance
Gaussian~\cite{arous2021online}, our goal in this appendix is to map the
resulting expansion onto the standard Hermite basis.

Using the definition of Hermite polynomials associated with a Gaussian
measure of variance $r$ (see
\eqref{app:property_Hermite definition}), we can rewrite
\eqref{app:property_Hermite scaling relation} as
\begin{equation*}
\He^{[r]}_k(z)=r^{k/2}\,\He_k\!\left(\frac{z}{\sqrt r}\right),
\end{equation*}
which follows from the change of variables $t \mapsto t/\sqrt r$ in the
generating function. To express this polynomial in the standard Hermite basis,
we use the classical Hermite scaling identity
\begin{equation*}
\He_k(a z)=\sum_{j=0}^{\lfloor k/2 \rfloor}a^{k-2j}(a^2-1)^j\frac{k!}{(k-2j)!\, j!\, 2^j}\,\He_{k-2j}(z).
\end{equation*}

Applying this identity with $a = 1/\sqrt r$,
we obtain
$a^{k-2j} = r^{-(k-2j)/2}$ and $a^2 - 1 = (1-r)/r$.
Multiplying by the prefactor $r^{k/2}$ yields
$$r^{k/2} a^{k-2j} (a^2 - 1)^j = (-1)^j (r-1)^j,$$
which leads to the expansion
\begin{equation}
\label{eq:scaled_Hermite_standard_basis}
\He^{[r]}_k(z)
=\sum_{j=0}^{\lfloor k/2 \rfloor}(-1)^j (r-1)^j\,\frac{k!}{(k-2j)!\, j!\, 2^j}\;\He_{k-2j}(z).
\end{equation}

This expression shows that if the coefficients of the non-standard expansion
satisfy $\sigma^{[r]}_k = 0$ for all $k \le k^\star - 1$, with
$\sigma^{[r]}_{k^\star} \neq 0$, then the lowest-order Hermite polynomial
appearing in the standard basis expansion is either $\He_0$ or $\He_1$,
depending on the parity of $k^\star$. As a consequence, the information
exponent of the resulting activation function is at most $2$, independently
of the choice of activation.

Finally, expressing the non-unit-variance expansion
\eqref{app:property_Hermite expansion} in the standard Hermite
basis yields
\begin{equation}
\label{app:property_Hermite_non_unit_exp_function_standard_basis}
\sigma(z)=\sum_{k \ge 0}\frac{\sigma^{[r]}_k}{k! r^k}\sum_{j=0}^{\lfloor k/2 \rfloor}(-1)^j(r-1)^j\,\frac{k!}{(k-2j)!\, j!\, 2^j}\;\He_{k-2j}(z).
\end{equation}

\subsection{Squaring a Hermite activation}
\label{app:squaring_pure_Hermite}

We argued in the main text that squaring the labels constitutes a natural
transformation that can reduce the information exponent (IE) of an activation
function. In this appendix, we provide the mathematical justification for this
claim. Since the student activation function has IE at most $2$, it suffices
to apply such a transformation to the teacher activation only.

Starting from the Hermite expansion for unit-variance pre-activations
\eqref{app:property_Hermite unit_variance_preactivation}, we write
\begin{equation}
\label{eq:sigma_square_Hermite}
\big( \sigma(z) \big)^2=\sum_{k,k' \ge 0}\frac{\sigma_k \sigma_{k'}}{k!\,k'!}\,\He_k(z)\,\He_{k'}(z)=\sum_{k,k' \ge 0}\frac{\sigma_k \sigma_{k'}}{k!\,k'!}\sum_{j=0}^{\min(k,k')}j!\binom{k}{j}\binom{k'}{j}\,\He_{k+k'-2j}(z).
\end{equation}

We now specialize to the case where the activation is a pure Hermite
polynomial, $\sigma(z) = \He_k(z)$. In this case, the above expression reduces
to
\begin{equation}
\label{eq:pure_Hermite_square}
\big( \sigma(z) \big)^2
=\sum_{j=0}^{k}
j!\binom{k}{j}^2\,\He_{2k-2j}(z)=k!\,\He_0(z)+k^2 (k-1)!\,\He_2(z)+\cdots.
\end{equation}
The expansion \eqref{eq:pure_Hermite_square} shows that, besides the constant
term, the lowest-order Hermite polynomial appearing in $\big(\sigma(z)\big)^2$
is $\He_2(z)$. Since the constant component does not affect correlations with
the input direction, the effective lowest non-trivial Hermite level governing
the learning dynamics is therefore $\He_2$. As a consequence, squaring a pure
Hermite activation reduces the information exponent of the teacher to $\mathrm{IE}=2$, independently of the original degree $k$.

\paragraph{Eliminating the delay near the singularity with two stage learning based on label transformation}
\label{app:numerics_use_square_label}

In  the main, we showed that label transformations can enable the student to learn faster by effectively reducing the escape time from the correlated search phase. Focusing on label squaring, we demonstrated that this transformation can eliminate the singularity that arises for certain activation functions. In Figure~\ref{fig:curiculum_learning_He3_activation_appendix}, we provide numerical evidence for a two-stage learning procedure at a pre-training alignment $\mu=0.325$, which was shown to correspond to a near-singular regime (i.e., failure to escape the correlated search phase cf. Figure \ref{fig:dynamics_Hermite_activation_appendix}) for the Hermite-$3$ activation $\He_3$. Across both stages, the student is trained using one-pass SGD (details of the experimental setup are provided in the caption). The two stages differ as follows: in the first stage, the labels are squared and used to train the student; once training has proceeded for a sufficiently long time, we revert to the original labels in the second stage. As observed in Figure~\ref{fig:curiculum_learning_He3_activation_appendix}, this learning paradigm enables a much faster escape from the correlated search phase—compared to the standard training procedure shown in Figure~\ref{fig:dynamics_Hermite_activation_appendix}—and subsequently allows the student to recover the teacher direction during the descent phase.

\begin{figure}[!htbp]
    \centering
    \includegraphics[width=1.0\textwidth]{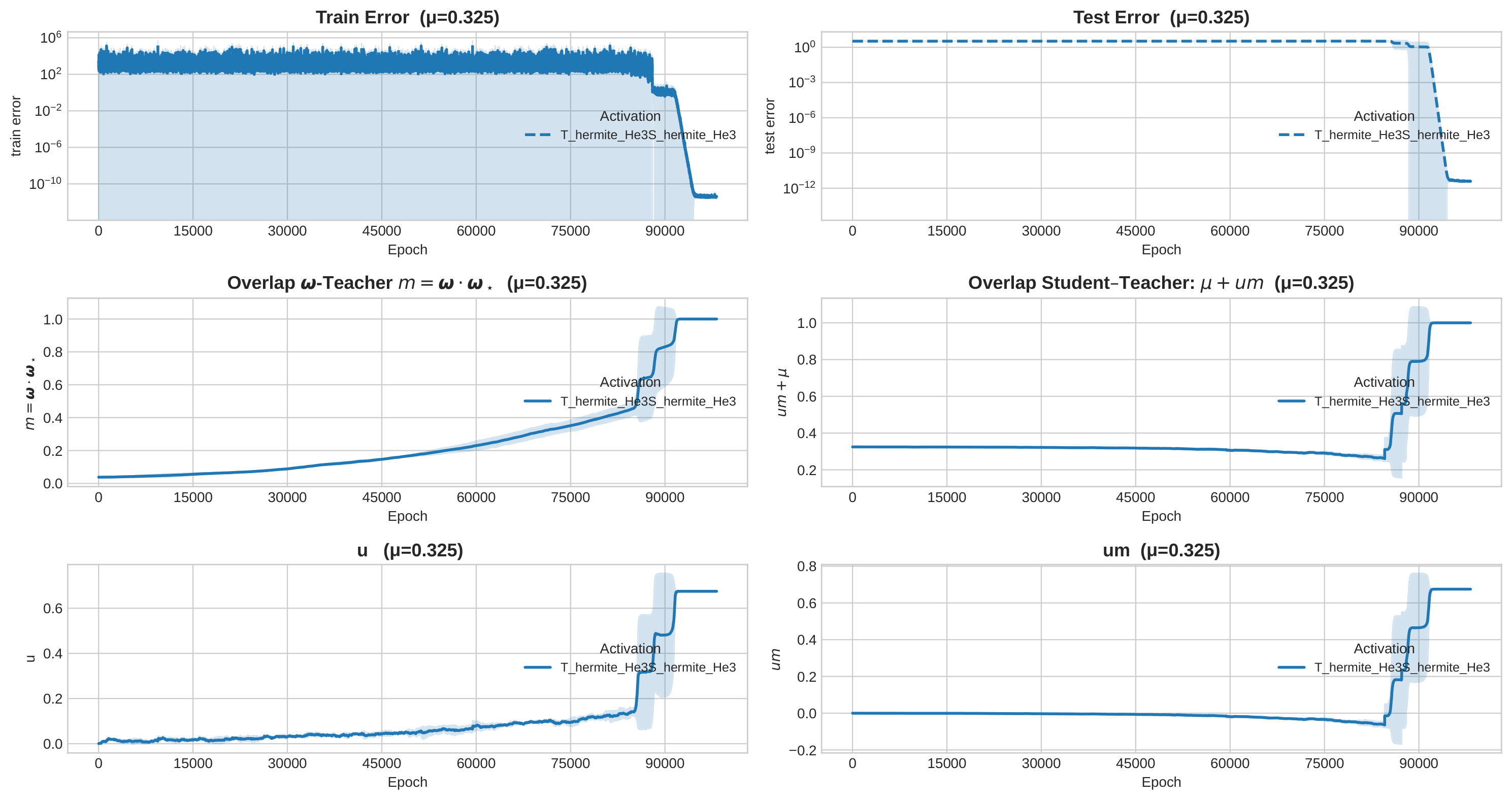}
   \caption{
Learning dynamics of the student model~\eqref{equ:student} for a pre-training alignment level $\mu=0.325$, trained on data generated by the teacher model~\eqref{equ:teacher}.
Both teacher and student activations are Hermite polynomials of degree $k=3$, as indicated in the legend.
The student is trained using a two-stage one-pass SGD procedure.
In the first stage, the ground-truth labels are squared and used until the overlap $m$ converges to approximately $0.5$; in the second stage, training proceeds using the original labels.
The hyperparameters are a batch size $B=5000$ and input dimension $d=1000$.
A spherical constraint is enforced by normalizing $\boldsymbol{\omega}$ after each gradient step, and the learning rate is set to $lr=10^{-2}\,\delta_k$, with $\delta_k=1/(k!\,k)$.
We note that when standard one-pass SGD is used on square label, the overlap remains close to zero ($m\simeq 0.04$; see the third row of Fig.~\ref{fig:dynamics_Hermite_activation_appendix}) even after $2.5\times 10^{5}$ epochs.
In contrast, with the two-stage procedure, the overlap has already escaped the correlated search phase by $2\times 10^{4}$ epochs.
The first stage ends at approximately $8.5\times 10^{4}$ epochs; during the second stage, a long plateau is observed before the dynamics eventually converge toward the global minimum of the loss.
Results are averaged over three independent runs. The observed staircase-like behavior arises from averaging across different random seeds: when the second learning phase begins, each seed reaches global convergence at a slightly different epoch. Because the transition is sharp, averaging these misaligned jumps produces the apparent stepwise structure.
}
    \label{fig:curiculum_learning_He3_activation_appendix}
\end{figure}

\section{Online SGD on the sphere}
\label{app:evolution_of_order_parameters}

We consider the spherical variant of stochastic gradient descent (SGD), where the weight vector $\boldsymbol{\omega}\in\mathbb{S}^{d-1}$ is constrained to lie on the unit sphere.
The update rule reads
\begin{equation}
\label{eq:gradient_step}
\boldsymbol{\omega}^{t+1}
=
\frac{
\boldsymbol{\omega}^{t}
-
\gamma\,\nabla^{\mathbb{S}^{d-1}}_{\boldsymbol{\omega}}\mathcal{L}(u,\boldsymbol{\omega},\boldx_i)
}{
\big\|
\boldsymbol{\omega}^{t}
-
\gamma\,\nabla^{\mathbb{S}^{d-1}}_{\boldsymbol{\omega}}\mathcal{L}(u,\boldsymbol{\omega},\boldx_i)
\big\|
},
\end{equation}
where $\nabla^{\mathbb{S}^{d-1}}_{\boldsymbol{\omega}}$ denotes the spherical gradient.
Explicitly, it is given by
\begin{equation}
\label{eq:spherical_gradient}
\boldsymbol{h}_{i}:=
\nabla^{\mathbb{S}^{d-1}}_{\boldsymbol{\omega}}\mathcal{L}(u,\boldsymbol{\omega},x_i)
=
\big(I_d-\boldsymbol{\omega}\boldsymbol{\omega}^\intercal\big)\nabla_{\boldsymbol{\omega}}\mathcal{L}(u,\boldsymbol{\omega},\boldx_i)
=
(\boldsymbol{\omega}_\star - m\boldsymbol{\omega})\,
\frac{\partial \mathcal{L}_i}{\partial m},
\end{equation}
where $m=\boldsymbol{\omega}\cdot\boldsymbol{\omega}_\star$.
In the last equality, we used the decomposition
\begin{equation}
\nabla_{\boldsymbol{\omega}}\mathcal{L}_i (u,\boldsymbol{\omega})
=
2\boldsymbol{\omega}\,\frac{\partial\mathcal{L}_i}{\partial q}
+
\boldsymbol{\omega}_\star\,\frac{\partial\mathcal{L}_i}{\partial m}.
\end{equation}

\paragraph{Small learning-rate expansion.}
Assuming $\gamma$ to be small, we expand the normalization factor in \eqref{eq:gradient_step}:
\begin{align}
\Big\|\boldsymbol{\omega}^t-\gamma\boldsymbol{h}_i
\Big\|^{-1}&=\Big(1+\gamma^2\|\boldsymbol{h}_i\|^2
\Big)^{-1/2}=1-\frac{\gamma^2}{2}\|\boldsymbol{h}_i\|^2+O(\gamma^3),
\end{align}
where
\begin{equation}
\label{eq:h_norm}
\|\boldsymbol{h}_i\|^2=(1-m^2)\Big(\frac{\partial\mathcal{L}_i}{\partial m}\Big)^2.
\end{equation}

\paragraph{Evolution of the overlap.}
Multiplying \eqref{eq:gradient_step} by $\boldsymbol{\omega}_\star$ yields the update equation for the overlap:
\begin{align}
\label{eq:m_step}
m^t
&=\Big(m^{t-1}-\gamma\,\boldsymbol{h}_i\cdot\boldsymbol{\omega}_\star\Big)
\Big(1-\frac{\gamma^2}{2}\|\boldsymbol{h}_i\|^2+O(\gamma^3)\Big) \nonumber\\
&=m^{t-1}-\gamma\Big(\boldsymbol{h}_i\cdot\boldsymbol{\omega}_\star-\frac{\gamma}{2}\|\boldsymbol{h}_i\|^2 m^{t-1}\Big)
+O(\gamma^3),
\end{align}
with
\begin{equation}
\boldsymbol{h}_i\cdot\boldsymbol{\omega}_\star
=
(1-m^2)\frac{\partial\mathcal{L}_i}{\partial m}.
\end{equation}

Now after averaging over the data we obtain \begin{equation}
    m^{t}=m^{t-1}-\gamma\Big(\boldsymbol{h}\cdot\boldsymbol{\omega}_\star-\frac{\gamma}{2}\|\boldsymbol{h}\|^2 m^{t-1}\Big)
+O(\gamma^3)
\end{equation}
with $\boldsymbol{h} :=\mathbb{E}_{\boldx_i}[\boldsymbol{h}_i]$ and $\|\boldsymbol{h}\|^2 :=\mathbb{E}_{\boldx_i}[\|\boldsymbol{h}_i\|^2]$.
The update is contractive provided the learning rate satisfies
\begin{equation}
\label{eq:learning_rate}
\gamma \le \frac{2\,\boldsymbol{h}\cdot\boldsymbol{\omega}_\star}{\|\boldsymbol{h}\|^2 m} \quad\Longleftrightarrow\quad \gamma \le \frac{2}{m} \Big(\frac{\partial\mathcal{L}}{\partial m}\Big)^{-1}.
\end{equation}

\paragraph{Continuous-time limit.}
We now set $\gamma=\delta\,dt$ with $\delta\in\mathbb{R}$ chosen such that the stability condition \eqref{eq:learning_rate} holds.
Equation~\eqref{eq:m_step} then becomes
\begin{equation}
\frac{m^t-m^{t-1}}{dt}=-\delta\,\boldsymbol{h}\cdot\boldsymbol{\omega}_\star+O(\delta^2 dt^2).
\end{equation}
Taking the limit $dt\to0$ yields the gradient-flow equation
\begin{equation}
\label{eq:dynamic_m}
\frac{dm}{dt}
=
-\delta(1-m^2)\frac{\partial\mathcal{L}}{\partial m}.
\end{equation}
Similarly, differentiation with respect to $u$ gives
\begin{equation}
\label{eq:dynamic_u}
\frac{du}{dt}
=
-\delta\,\frac{\partial\mathcal{L}}{\partial u}.
\end{equation}

\paragraph{Choice of time rescaling.}
The parameter $\delta$ is introduced to ensure that the dynamical equations \eqref{eq:dynamic_m}–\eqref{eq:dynamic_u} remain $\mathcal{O}(1)$ for different choices of activation functions, which affect the scaling of $\partial\mathcal{L}/\partial m$ and $\partial\mathcal{L}/\partial u$.
In particular, for pure Hermite activations of degree $k^\star$, an appropriate choice is
\begin{equation}
\label{equ:delta_hermite}
\delta=(k^\star! \, k^\star)^{-1}.
\end{equation}

\section{Population loss analysis}
\label{appendix:population_risk}

\subsection{Gradient of the population loss}
\label{app:gradient_population_risk}

By leveraging \eqref{app:property_Hermite expansion} and \eqref{app:property_Hermite unit_variance_preactivation} , the population loss defined in \eqref{equ:population_loss_app} becomes:
\begin{align}
\nonumber
&\mathcal{L}(u,m)  = \sum_{k=0}^{+\infty} \frac{\phi^2_k}{2 k!} + \sum_{k=0}^{+\infty} \frac{(\sigma^{[r]}_k)^2}{2k!r^{k}} - \sum_{k=0}^{+\infty} \frac{\sigma^{[r]}_k \phi_k}{k!r^{k}} (\mu+um)^k 
\end{align}
with $m= \omeg_{\star}\cdot \omeg, ~r=\mu^2 + u^2 + 2\mu um.$ 

The partial derivative of $\mathcal{L}$ respect to $u$ and $m$ are \begin{align}
\label{eq:grad_u_m}
\nonumber   & \frac{\partial \mathcal{L}}{ \partial u} =(u+\mu m) \sum_{k=0}^{+\infty} \frac{\sigma^{[r]}_k\bar{\sigma}^{[r]}_k}{k! r^{k+1}}-\sum_{k=0}^{+\infty} \phi_{k} (\mu + um)^{k-1} \Bigg[ -\frac{(u+m\mu)(\mu+mu)}{(k-1)!r^{k+1}}\sigma^{[r]}_{k} + \frac{\sigma^{[r]}_k m}{(k-1)!r^k} + \frac{(\mu+mu)(u+m\mu)\bar{\sigma}^{[r]}_k}{k!r^{k+1}} \Bigg] \\
   & \frac{\partial \mathcal{L}}{ \partial m} =u\mu \sum_{k=0}^{+\infty} \frac{\sigma^{[r]}_k\bar{\sigma}^{[r]}_k}{k! r^{k+1}} -\sum_{k=0}^{+\infty} \phi_{k} (\mu + um)^{k-1}\Bigg[ -\frac{u\mu(\mu+mu)}{(k-1)!r^{k+1}}\sigma^{[r]}_{k} + \frac{u\sigma^{[r]}_k }{(k-1)!r^k}+ \frac{(\mu+mu)u\mu\bar{\sigma}^{[r]}_k}{k!r^{k+1}} \Bigg]
\end{align}
with $\bar{\sigma}^{[r]}_k =  r^{\frac{k+1}{2}} \int \DD z \He_{k}(z)\sigma^{'}(\sqrt{r}z)z.$
The summation $\sum_{k=0}^{\infty}$ is used in a loose manner to lighten the notation. It should be understood that whenever a factorial prefactor of the form $1/(k-1)!$ appears, the summation effectively starts at $k=1$ rather than $k=0$.

\subsection{Gradient of the population loss in the correlated search phase}
\label{app:linearize_the_population_loss}
In this phase we are using $m\ll \mu$ and $u\ll\mu$, we can then do the following simplification $\mu+um\simeq \mu,~r\simeq \mu^2,$ \eqref{eq:grad_u_m} becomes \begin{align}
\label{eq:grad_u_m_2}
  \nonumber  \frac{\partial \mathcal{L}}{d u} &=  m \sum_{k=0}^{+\infty}\frac{\bar{\sigma}^{[r]}_k}{k!\mu^{k+1}}(-\phi_k +\frac{\sigma^{[r]}_k}{\mu^k}) +u\sum_{k=0}^{+\infty} \Bigg[ \frac{\bar{\sigma}^{[r]}_k\sigma^{[r]}_k}{k!r^{k+1}}+\frac{\phi_k}{(k-1)!\mu^{k+2}}\big(\sigma^{[r]}_k - \frac{\bar{\sigma}^{[r]}_k}{k}\big) \Bigg] \\
    \frac{\partial \mathcal{L}}{d m} &=u\sum_{k=0}^{+\infty} \frac{\bar{\sigma}^{[r]}_k}{k!\mu^{k+1}}(-\phi_k +\frac{\sigma^{[r]}_k}{\mu^k}),
\end{align}
where $\sigma^{[r]}_k $ and $\bar{\sigma}^{[r]}_k$ can be rewritten to become \begin{equation}
    \sigma^{[r]}_k = r^{\frac{k}{2}}\int D_z \He_k (z)\sigma(z\sqrt{r}) =r^{\frac{k}{2}} \mathbb{E}[\He_k (z)\sigma(z\sqrt{r})] ,\quad \bar{\sigma}^{[r]}_k =  r^{\frac{k+1}{2}}\int D_z \He_k(z) z \sigma^{'}(z\sqrt{r}) =r^{\frac{k+1}{2}}\mathbb{E}[z\He_k (z)\sigma^{'}(z\sqrt{r})].
\end{equation}

The equations \eqref{eq:grad_u_m_2}  can be cast as
\begin{align}
\label{equ:gradient_u_general}
    \frac{\partial \mathcal{L}}{\partial u} &= -uB -m  A   \\
      \label{equ:gradient_m_general}
    \frac{\partial\mathcal{L}}{\partial m} &=-u  A
\end{align}
with \begin{equation}
\label{app:linearize_the_population_loss A_and_B}
    A =-\sum_{k=0}^{+\infty}\frac{\bar{\sigma}^{[r]}_k}{k!\mu^{k+1}}(-\phi_k +\frac{\sigma^{[r]}_k}{\mu^k}), \quad B=-\sum_{k=0}^{+\infty} \Bigg[ \frac{\bar{\sigma}^{[r]}_k\sigma^{[r]}_k}{k!\mu^{2k+2}}+\frac{\phi_k}{(k-1)!\mu^{k+2}}\big(\sigma^{[r]}_k - \frac{\bar{\sigma}^{[r]}_k}{k}\big) \Bigg].
\end{equation}

\section{Escaping the correlated search phase}
\subsection{Generic activation function: proof of Prop. \ref{prop:escaping_mediocrity}}
\label{appendix:escaping_time}
\begin{proof}
Using \eqref{eq:dynamic_m} and \eqref{eq:dynamic_u}, together with the
linearized gradients \eqref{equ:gradient_u_general} and
\eqref{equ:gradient_m_general}, one obtains
\begin{equation}
\label{equ:dynamic_u_general}
    \begin{cases}
        \dfrac{d u}{d t} = u B + m A,  \\
        \dfrac{d m}{d t} = u A,
    \end{cases}
\end{equation}
where the factor $\delta$ has been absorbed into the definitions of $A$ and
$B$.

We rewrite \eqref{equ:dynamic_u_general} in matrix form as
\begin{equation}
\begin{bmatrix}
  \dot{u} \\
  \dot{m}
\end{bmatrix}
=
\begin{bmatrix}
  B & A \\
  A & 0
\end{bmatrix}
\begin{bmatrix}
  u \\
  m
\end{bmatrix}.
\end{equation}

This matrix has eigenvalues
\begin{equation}
  \lambda_{\pm}(\mu)
  = \frac{B \pm \sqrt{B^2 + 4A^2}}{2}.
\end{equation}
In the correlated search regime of interest, we assume $A \neq 0$, so that
$\lambda_+(\mu) > 0$ and the fixed point is linearly unstable.

Let $\epsilon := u_0^2 + m_0^2$ denote the squared norm of the initial
condition. Solving the linear system yields exponential growth along the
unstable eigendirection at rate $\lambda_+(\mu)$. As a result, the exit time
from a ball of radius $R$ satisfies, up to additive constants,
\begin{equation}
\label{equ:t_exit}
t_{\rm exit}(\mu)
=
\frac{1}{\lambda_+(\mu)}
\log\!\left(\frac{R}{\sqrt{\epsilon}}\right)
+ \mathcal{O}(1).
\end{equation}

We now specify the initialization. For random isotropic initialization on the
sphere, we have $m_0 = \mathcal{O}(d^{-1/2})$ with high probability. We further
assume $u_0^2 \sim m_0^2$, so that $\epsilon \sim d^{-1}$. For any fixed escape
radius $R>0$, this yields
\begin{equation}
t_{\rm exit}(\mu)
=
\frac{1}{2\lambda_+(\mu)}\log d
+\mathcal{O}(1).
\end{equation}

Recalling that $\tau(\mu)=\lambda_+(\mu)^{-1}$, we conclude that
\begin{equation}
t_{\rm exit}
\;\sim\;
\frac{\tau(\mu)}{2}\,\log d,
\end{equation}
which proves Prop.~\ref{prop:escaping_mediocrity}.
\end{proof}

\subsection{Case of linear activation function}
\label{app:dynamics_of_order_parameter_for_linear_activation}

The dynamics of the order parameters are given by
\begin{equation*}
\begin{cases}
  \dot{u} = (1 - \mu)m - u, \\
  \dot{m} = (1 - \mu)u(1 - m^2).
\end{cases}
\end{equation*}

Linearizing the dynamics around the fixed point $(u,m)=(0,0)$ yields a linear
system whose Jacobian matrix has eigenvalues
\begin{equation*}
  \lambda_{\pm}(\mu)
  = \frac{-1 \pm \sqrt{1 + 4(1 - \mu)^2}}{2}.
\end{equation*}

The associated characteristic time scale $\tau(\mu)$ is given by
\begin{equation*}
  \tau(\mu)
  = \frac{1}{\lambda_+(\mu)}
  = \frac{1 + \sqrt{1 + 4(1 - \mu)^2}}{2(1 - \mu)^2}.
\end{equation*}

Note that this time scale diverges as $\mu \to 1^{-}$:
\begin{equation*}
  \tau(\mu) \sim \frac{1}{(1 - \mu)^2},
  \quad \text{as } \mu \to 1^-.
\end{equation*}

\subsection{Numerical evidence for function with IE$=1$}
\label{appendix:numerical evidence of slow down}

Figure~\ref{fig:dynamics_activation_comparison_appendix} provides additional evidence of catastrophic overtraining across linear, erf, ReLU, and sigmoid activations over a wide range of pre-training alignments, with experimental details given in the caption.
\begin{figure}[!htpp]
    \centering
    \begin{subfigure}
        \centering
        \includegraphics[width=0.75\textwidth]{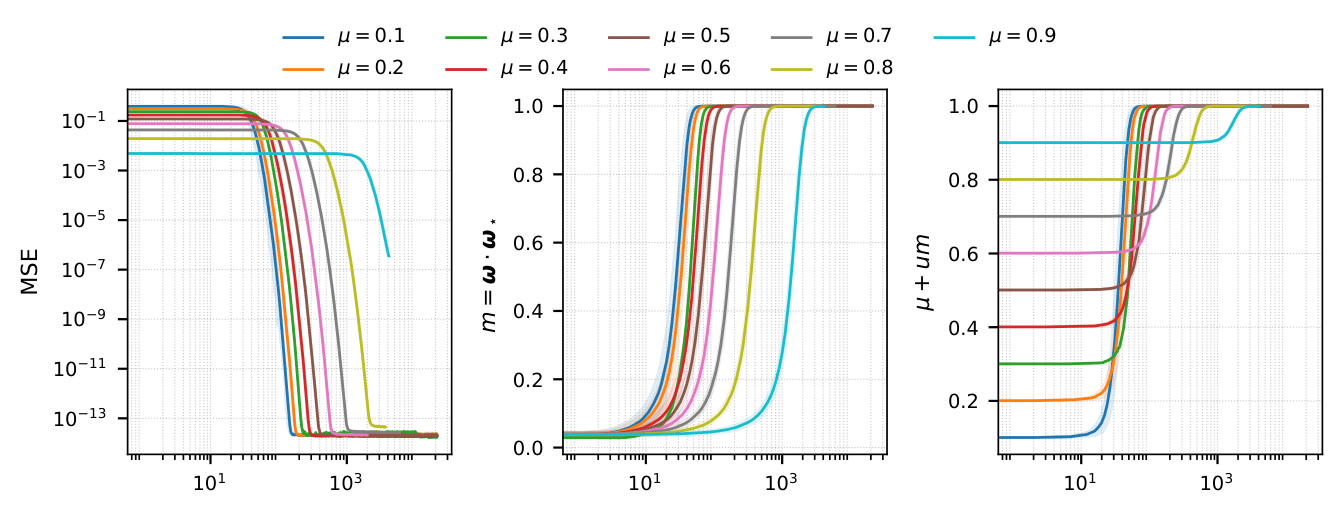}
    \end{subfigure}
    \vspace{-1.9em}
    \begin{subfigure}
        \centering
        \includegraphics[width=0.75\textwidth]{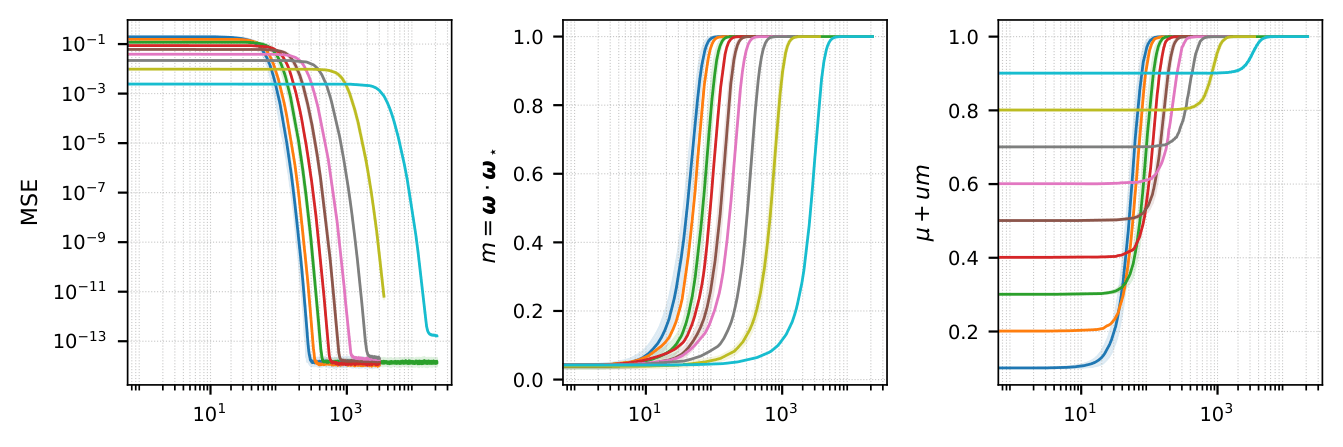}
    \end{subfigure}
    \vspace{-1.0em}
    \begin{subfigure}
        \centering
        \includegraphics[width=0.75\textwidth]{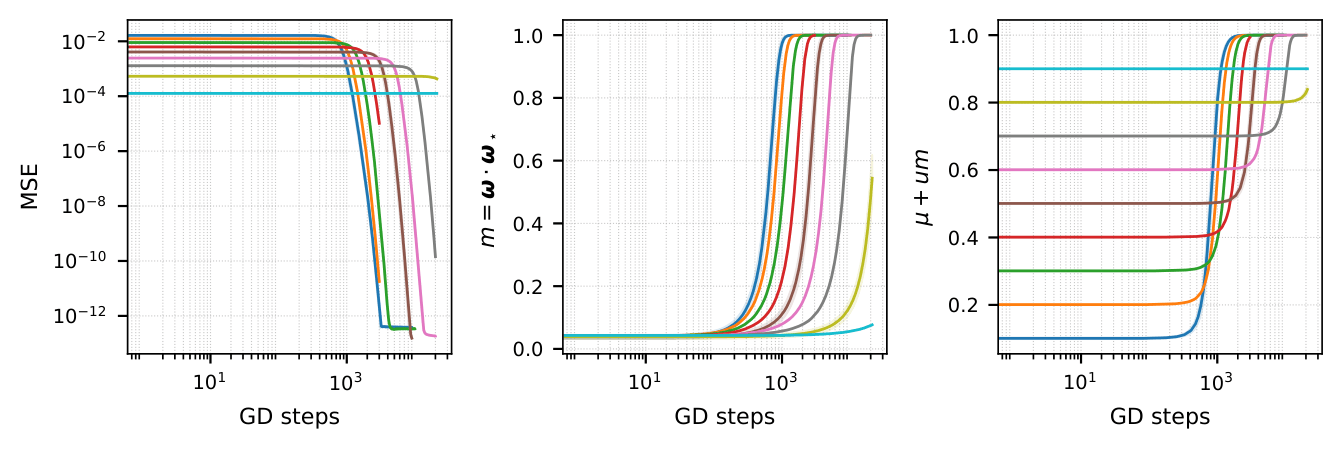}
    \end{subfigure}
    \caption{Learning dynamics of the student model~\eqref{equ:student} for different levels of pre-training alignment $\mu \in \{0.1, 0.2, 0.3, 0.4, 0.5, 0.8, 0.9\}$, trained on data generated by the teacher model~\eqref{equ:teacher}. We consider the matched teacher--student Linear, ReLU, and Sigmoid (from top to bottom) activation setting and train the student using one-pass SGD with batch size $B=5000$ and input dimension $d=1000$. The panels report the test mean squared error (MSE) (left), the alignment between the student and teacher directions $m=\boldsymbol{\omega}_\star\!\cdot\!\boldsymbol{\omega}$ (middle), and the effective teacher--student overlap $m_{\mathrm{eff}}=\mu+u m$ (right).
Shaded regions indicate one standard deviation over three independent runs.
The spherical constraint is enforced by normalizing $\boldsymbol{\omega}$ after each gradient step and we used  a learning rate $lr=0.2 .$}
    \label{fig:dynamics_activation_comparison_appendix}
\end{figure}

\subsection{Case of pure Hermite polynomial: proof of
Propositions~\ref{prop:singularity_pure_odd_Hermite} and
\ref{prop:exit_time_mu_to_1_pure_Hermite}}
\label{appendix:proof_proposition_Hermite}

A common ingredient in these propositions is that they all involve the
quantities $A$ and $B$. We therefore first derive explicit analytical
expressions for these quantities in the case of Hermite polynomial
activations, and then use these expressions to complete the proofs of
each result. For this class of functions, the Hermite coefficients
$\sigma^{[r]}_k$ can be computed explicitly.

Assume that $\sigma(\cdot)$ is the $k^{\star}$-th Hermite polynomial.
This implies that $\sigma_k = \sigma_{k^\star}\,\delta_{k^\star,k}= k^{\star}!\,\delta_{k^\star,k}.$ The associated Hermite expansion for the student in the standard Hermite basis then reads:
\begin{equation}
\sigma_{k}=\sigma_{k^\star}\delta_{k^\star,k}=k^{\star}!\delta_{k^\star,k},\quad   \sigma^{[r]}_k = r^{\frac{k}{2}} \mathbb{E}[\He_k(z) \He_{k^\star}(z\sqrt{r})] ,\quad \bar{\sigma}^{[r]}_k =r^{\frac{k}{2}}\mathbb{E}[\He_k(z) z\sqrt{r} \He^{'}_{k^\star}(z\sqrt{r})].
\end{equation}
By leveragint the following property \begin{align*}
    \mathbb{E}[\He_i(z) \He_{j}(a z)] &=\begin{cases}
        j!a^{i}\frac{\big(\frac{a^2-1}{2} \big)^{\frac{j-i}{2}}}{\big( \frac{j-i}{2} \big)!},\quad \text{if} \quad j \geq i \quad \text{and} \quad j-i=2h,\\
        0\quad \quad \quad \quad \text{otherwise},
    \end{cases}\\
    \nonumber    \He_{k}^{'}(x) &= k\He_{k-1}(x)\\
    x \He_{k}^{'}(x) &= k x \He_{k-1}(x) = k \Big( \He_{k}(x) + (k-1 )\He_{k-2}(x) \Big)
\end{align*}
we obtain the following coefficient
\begin{equation}
\label{term_Hermite_student_coefficient}
\sigma_{k}=\sigma_{k^\star}\delta_{k^\star,k}=k^{\star}!\delta_{k^\star,k},\quad \sigma^{[r]}_k = k^\star! r^{k}\frac{\Big(\frac{r-1}{2}\Big)^{\frac{k^\star-k}{2}}}{\Big(\frac{k^\star-k}{2}\Big)!}~ \text{for which }~ 0\le k \le k^\star~\text{and}~ k \equiv k^{\star}\pmod{2}. 
\end{equation}

\begin{align*}
    \bar{\sigma}^{[r]}_k 
&=\begin{cases} k^\star! r^{k} \frac{\Big(\frac{r-1}{2} \Big)^{\frac{k^\star-2-k}{2}}}{\Big(\frac{k^\star-2-k}{2} \Big)!}\Big(\frac{k^\star r-k}{k^\star-k} \Big),\quad \text{if}~ 0 \le k < k^\star~\text{and}~ k \equiv k^{\star} \pmod{2} \\
k^\star \sigma^{[r]}_{k} \quad \text{if}\quad k= k^{\star}.
\end{cases}
\end{align*}

For $k=k^\star$ the coefficient blend to \begin{equation*}
    \sigma^{[r]}_{k^\star} = k^\star!r^{k^\star},\quad \bar{\sigma}^{[r]}_{k^\star} = k^\star!k^\star r^{k^\star}.
\end{equation*}

In the next we define the admissible indices $k$ as integers satisfying
$0 \le k < k^{\star} ~\text{and}~ k \equiv k^{\star} \pmod{2}$
, we will plug-in these expressions into $A$ and $B$ (see \eqref{app:linearize_the_population_loss A_and_B}) that leads to
\begin{align}
\label{app:linearize_the_population_loss intermediate A_and_B}
 \nonumber   A &=-\Bigg[\frac{k^\star!k^\star r^{k^\star} }{k^\star!\mu^{k^\star+1}}(-k^\star! +\frac{k^\star! r^{k^\star}}{\mu^{k^\star}})+\sum_{k< k^{\star}}\frac{\bar{\sigma}^{[r]}_k\sigma^{[r]}_k}{k!\mu^{2k+1}} \Bigg]=-\Bigg[k^\star!k^\star \mu^{k^\star-1}(-1 +\mu^{k^\star})+\sum_{k< k^{\star}}\frac{\bar{\sigma}^{[r]}_k\sigma^{[r]}_k}{k!\mu^{2k+1}} \Bigg]\\
   B &=-\Bigg[\frac{(k^\star!)^2 k^\star r^{2k^\star}}{k^{\star}!\mu^{2k^\star+2}}+ \sum_{k<k^{\star}} \frac{\bar{\sigma}^{[r]}_k\sigma^{[r]}_k}{k!\mu^{2k+2}} \Bigg]=-\Bigg[k^\star! k^\star \mu^{2k^{\star}-2}+ \sum_{k<k^{\star}} \frac{\bar{\sigma}^{[r]}_k\sigma^{[r]}_k}{k!\mu^{2k+2}} \Bigg]
\end{align}
with \begin{align*}
    \bar{\sigma}^{[r]}_k\sigma^{[r]}_k &= (k^\star!)^2 r^{2k} \frac{\Big(\frac{r-1}{2}\Big)^{k^\star-k-1}}{\Big(\frac{k^\star-k}{2}\Big)!\Big(\frac{k^\star-k-2}{2}\Big)!}\Big(\frac{k^\star r-k}{k^\star-k} \Big) \Rightarrow \frac{\bar{\sigma}^{[r]}_k\sigma^{[r]}_k}{k!\mu^{2k}}= \frac{(k^\star!)^2}{k!} \mu^{2k} \frac{\Big(\frac{r-1}{2}\Big)^{k^\star-k-1}}{\Big(\frac{k^\star-k}{2}\Big)!\Big(\frac{k^\star-k-2}{2}\Big)!}\Big(\frac{k^\star r-k}{k^\star-k} \Big),
\end{align*}
that becomes \begin{equation}
\label{app:linearize_the_population_loss  term for A and B}
    \frac{\bar{\sigma}^{[r]}_k\sigma^{[r]}_k}{k!\mu^{2k+2}} = \frac{(k^\star!)^2}{k!} \mu^{2k-2} \frac{\Big(\frac{r-1}{2}\Big)^{k^\star-k-1}}{\Big(\frac{k^\star-k}{2}\Big)!\Big(\frac{k^\star-k-2}{2}\Big)!}\Big(\frac{k^\star r-k}{k^\star-k} \Big).
\end{equation}

Using \eqref{app:linearize_the_population_loss term for A and B}, \eqref{app:linearize_the_population_loss intermediate A_and_B}  becomes
\begin{equation}
    A =-\Bigg[k^\star!k^\star \mu^{k^\star-1}(-1 +\mu^{k^\star})+(k^\star!)^2 \mu^{-1}f(k^\star,\mu) \Bigg],~
   B =-\Bigg[k^\star! k^\star \mu^{2k^{\star}-2}+  (k^\star!)^2 \mu^{-2}f(k^\star,\mu) \Bigg],
\end{equation}
which after rescaling $A\rightarrow \frac{A}{k^\star!k^\star} $ and $B\rightarrow \frac{B}{k^\star!k^\star} $ (we set $\delta=(k^\star!k^\star)^{-1}$ as it was mentioned in \eqref{equ:delta_hermite}) becomes
\begin{align}
\label{equ: A bar}
  A =-\Bigg[\mu^{k^\star-1}(-1 +\mu^{k^\star})+\frac{(k^\star!)}{k^\star} \mu^{-1}f(k^\star,\mu)  \Bigg], ~~\text{and}~~ 
  B=-\Bigg[ \mu^{2k^{\star}-2}+  \frac{k^\star!}{k^\star} \mu^{-2} f(k^\star,\mu) \Bigg],
\end{align}
with $f$ given by \begin{align}
\label{equ: f}
   f(k^\star,\mu) &= \sum_{k<k^{\star}} \frac{\mu^{2k}}{k!}  \frac{\Big(\frac{r-1}{2}\Big)^{k^\star-k-1}}{\Big(\frac{k^\star-k}{2}\Big)!\Big(\frac{k^\star-k-2}{2}\Big)!}\Big(\frac{k^\star r-k}{k^\star-k} \Big)=\sum_{m= 1}^{p} \frac{\Big(1+\frac{k^\star(r-1)}{2m} \Big)\mu^{2(k^\star-2m)}\big(r-1\big)^{2m-1}}{(k^\star-2m)!m!(m-1)!2^{2m-1}},
\end{align}
where $m=\frac{k^\star-k}{2}$ and we assume $k^\star=2p+1$ or $k^\star=2p.$ 

Now we are concerned with values of $\mu$ in the neighborhood of zero such that the lowest power of $\mu$ dominates $f(k^\star,\mu)$. Since $r=\mu^2$ then 
\begin{align*}
    \Big(1+\frac{k^\star(r-1)}{2m} \Big)&=1-\frac{k^\star}{2m}+\frac{k^\star}{2m}\mu^2\\
    (r-1)^{2m-1}&=(\mu^2-1)^{2m-1}=\sum_{t=0}^{2m-1} \binom{2m-1}{t} (-1)^{2m-1-t}(\mu^2)^{t}=-1+\mu^2+\sum_{t=2}^{2m-1} \binom{2m-1}{t} (-1)^{2m-1-t}(\mu^2)^{t}.
\end{align*}
then $f(k^\star,\mu)$ can be rewritten as \begin{align}
\label{equ:f_term}
\nonumber    f(k^\star,\mu) &= \frac{\Big(1-\frac{k^\star}{2p}+\frac{k^\star}{2p}\mu^2 \Big)\mu^{2(k^\star-2p)}}{(k^\star-2p)!p!(p-1)!2^{2p-1}}\Big(-1+\mu^2\\
\nonumber &+\sum_{t=2}^{2p-1} \binom{2p-1}{t} (-1)^{2p-1-t}(\mu^2)^{t} \Big)+\sum_{m= 1}^{p-1} \frac{\Big(1+\frac{k^\star(r-1)}{2m} \Big)\mu^{2(k^\star-2m)}\big(r-1\big)^{2m-1}}{(k^\star-2m)!m!(m-1)!2^{2m-1}}
\end{align}

\begin{equation*}
f(k^\star,\mu)=
\begin{cases}
\dfrac{\mu^2\Big(-1+\mu^2+\sum_{t=2}^{2p-1} \binom{2p-1}{t} (-1)^{2p-1-t}(\mu^2)^{t} \Big)}{p!(p-1)!2^{2p-1}}
+\tilde{f}(k^\star,\mu)
& \text{if}\quad k^\star=2p,\\[0.5em]
\dfrac{\Big(-\frac{1}{2p}+(\frac{1}{2p}+1)\mu^2 \Big)\Big(-1+\mu^2+\sum_{t=2}^{2p-1} \binom{2p-1}{t} (-1)^{2p-1-t}(\mu^2)^{t} \Big)\mu^{2}}{p!(p-1)!2^{2p-1}}
+\tilde{f}(k^\star,\mu)
& \text{if}\quad k^\star=2p+1,
\end{cases}
\end{equation*}

with \begin{equation*}
  \tilde{f}(k^\star,\mu)  =\sum_{m= 1}^{p-1} \frac{\Big(1+\frac{k^\star(r-1)}{2m} \Big)\mu^{2(k^\star-2m)}\big(r-1\big)^{2m-1}}{(k^\star-2m)!m!(m-1)!2^{2m-1}}.
\end{equation*}

Let's call $f_{0}(k^\star,\mu)$ the part of $f(k^\star,\mu)$ that contains the lowest power of $\mu,$ $f_{0}(k^\star,\mu)$ is given by \begin{equation}
\label{equ: f_0}
f_{0}(k^\star,\mu)=
\begin{cases}
-\dfrac{\mu^2}{p!(p-1)!2^{2p-1}}, 
& \text{if } k^\star = 2p, \\[1ex]
\dfrac{\mu^{2}}{(2p)p!(p-1)!2^{2p-1}}, 
& \text{if } k^\star = 2p+1.
\end{cases}
\end{equation}

Having assembled all the necessary ingredients, we are now ready to proceed with the proof.

\begin{proof}
\underline{Proof of Proposition \ref{prop:singularity_pure_odd_Hermite}}. 

Let's call $k^\star =2p+1 \ge 3$, the Hermite degree. Using \eqref{equ: f}, let's consider $\mu$ such that \begin{equation}
    1+\frac{k^{\star}(\mu^2-1)}{2}>0 \Leftrightarrow \mu^2 > 1-\frac{2}{k^\star} \Leftrightarrow \mu >\mu_0,
\end{equation}
with $\mu_{0}=\sqrt{1-\frac{2}{k^\star}}\in(0,1).$ Now consider $\mu>\mu_0$ therefore $\forall m\in \{1,\cdots,p\}$ the following hold \begin{equation}
\begin{cases}
    1+\frac{k^\star(\mu^2-1)}{2m}>0\\
    (r-1)^{2m-1} <0
\end{cases} \implies f(k^\star,\mu) <0.
\end{equation}
Using \eqref{equ: A bar} one can therefore concludes that $\forall \mu>\mu_0, A(\mu)>0.$

Using \eqref{equ: f_0} and \eqref{equ: A bar}, for $\mu=\varepsilon>0$ with $\varepsilon$ in the neighbourhood of zero, we have
 \begin{align}
\nonumber \frac{k^\star!}{k^\star}\varepsilon^{-1}f(k^\star,\varepsilon)&=\frac{(2p+1)!}{(2p+1)} \frac{\varepsilon}{(2p)p!(p-1)!2^{2p-1}}+O(\varepsilon^2)\\
&\implies A(\varepsilon) = -\frac{(2p+1)!}{(2p+1)} \frac{\varepsilon}{(2p)p!(p-1)!2^{2p-1}}+O(\varepsilon^2) <0.
\end{align} 

Since $A$ is continuous on $(0,1)$ and changes sign on this interval, there exists $\tilde{\mu}\in(0,1)$ such that $A(\tilde{\mu})=0$.

\end{proof}

\begin{proof}
\underline{Proof of Proposition \ref{prop:exit_time_mu_to_1_pure_Hermite}}. 

Using \eqref{equ: f} and \eqref{equ: f_0}, one can show that when $\mu$ is in the neighborhood of zero, $f(k^\star,\mu)$ becomes
\begin{equation}
\frac{k^{\star}!}{k^{\star}}\mu^{-1}f(k^\star,\mu)=
\begin{cases}
-\frac{(2p)!}{2p}\dfrac{\mu}{p!(p-1)!2^{2p-1}}+O(\mu^2), 
& \text{if } k^\star = 2p, \\[1ex]
\frac{(2p+1)!}{2p+1}\dfrac{\mu}{(2p)p!(p-1)!2^{2p-1}}+O(\mu^2), 
& \text{if } k^\star = 2p+1.
\end{cases}
\end{equation}
Now given that $k^\star >2$ one obtains the following \begin{align}
    A &= -\Big(\mu^{k^\star-1}(-1+\mu^{k^\star}) +\frac{k^{\star}!}{k^{\star}}\mu^{-1}f(k^\star,\mu)\Big)=
\begin{cases}
\frac{(2p)!}{2p}\dfrac{\mu}{p!(p-1)!2^{2p-1}}+O(\mu^2), 
& \text{if } k^\star = 2p, \\[1ex]
-\frac{(2p+1)!}{2p+1}\dfrac{\mu}{(2p)p!(p-1)!2^{2p-1}}+O(\mu^2), 
& \text{if } k^\star = 2p+1,
\end{cases}\\
B&=-\Bigg[ \mu^{2k^{\star}-2}+  \frac{k^\star!}{k^\star} \mu^{-2} f(k^\star,\mu) \Bigg]=\begin{cases}
\displaystyle
\frac{(2p)!}{2p}\frac{1}{p!(p-1)! \, 2^{2p-1}}+O(\mu^2),
& \text{if } k^\star = 2p, \\[1ex]
\displaystyle
-\frac{(2p+1)!}{2p+1}\frac{1}{p!(p-1)! \, (2p)\, 2^{2p-1}} +O(\mu^2),
& \text{if } k^\star = 2p+1.
\end{cases}
\end{align}

In the regime $\mu \to 0^{+}$, let's call $A_0$ and $B_0$ respectively the leading order of $A$ and $B$. Using \eqref{equ:t_exit} we get
\begin{equation*}
    \sqrt{B_0^2 + 4A_0^2}
    = |B_0|\left(1 + \frac{4A_0^2}{B_0^2}\right)^{1/2}
    = |B_0|\left(1 + \frac{2A^2}{B_0^2}\right)= |B_0|(1 + \frac{2A_0^2}{B^2_0}),
\end{equation*}
which leads to
\begin{equation*}
  \frac{-B_0 + |B_0|\left(1 + \frac{2A_0^2}{B_0^2}\right)}{2A_0^2}
  =\frac{-B_0 + |B_0|}{2A_0^2} + \frac{|B_0|}{B_0^2} =
  \begin{cases}
     \displaystyle \frac{1}{B_0} & \text{if } B_0 > 0, \\[6pt]
     \displaystyle -\frac{B_0}{A_0^2}+\frac{1}{|B_0|} & \text{if } B_0 < 0.
  \end{cases}
\end{equation*}

One finally obtains
\begin{equation}
  \tau(\mu)
  \sim
  \begin{cases}
     \displaystyle \frac{1}{B_0},
     & \text{if } k^{\star} = 2p \text{ with } p>1, \\[8pt]
     \displaystyle
     -\frac{B_0p!(p-1)!(2p)2^{2p-1}}{(2p+1)!}\,
     \frac{1}{\mu^2},
     & \text{if } k^{\star} = 2p+1 \text{ with } p>1.
  \end{cases}
\end{equation}

From \eqref{equ: f} we have that $f(k^\star,\mu)$ is expressed as a positive powers of $\mu^2-1$ (we recall that $r= \mu^2$). 
In the regime $\mu^2\to 1,$ we have 
\begin{align}
   f(k^\star,\mu) &= \sum_{m= 1}^{p} \frac{\Big(1+\frac{k^\star(r-1)}{2m} \Big)\mu^{2(k^\star-2m)}\big(r-1\big)^{2m-1}}{(k^\star-2m)!m!(m-1)!2^{2m-1}} = \frac{\mu^{2(k^\star-2)}(\mu^2-1)}{(k^\star-2)!2}+O((\mu^2-1)^2),\\
   -1+\mu^{k^\star} &= \frac{k^\star}{2}(\mu^2-1) +O((\mu^2-1)^2),
\end{align}
which leads to \begin{align*}
    A&=-\Big(\frac{k^\star}{2}(\mu^2-1)\mu^{k^\star-1} +\frac{k^\star-1}{2\mu} \mu^{2(k^\star-2)}(\mu^2-1)+O((\mu^2-1)^2)\Big)=a_{1}(\mu)\varepsilon + O(\varepsilon^2)\\
    B&=-\Big(\mu^{2k^\star-2} +\frac{k^\star-1}{2\mu^2} \mu^{2(k^\star-2)}(\mu^2-1)+O((\mu^2-1)^2)\Big)=b_{0}(\mu)+b_{1}(\mu)\varepsilon+O(\varepsilon^2),
\end{align*}
with $\varepsilon=\mu^2-1,a_{1}(\mu)=-\frac{k^\star}{2}\mu^{k^\star-1} -\frac{k^\star-1}{2\mu}\mu^{2(k^\star-2)}, b_{0}(\mu)=-\mu^{2k^\star-2}$ and $b_{1}(\mu)=-\frac{k^\star-1}{2\mu^2} \mu^{2(k^\star-2)}.$

One then has \begin{align*}
    \tau(\mu)&=\frac{-(b_{0}(\mu)+b_{1}(\mu)\varepsilon)+\sqrt{(b_{0}(\mu)+b_{1}(\mu)\varepsilon)^2+4a^{2}_{1}(\mu)\varepsilon^2}}{2a^{2}_{1}(\mu)\varepsilon^2}=-\frac{b_0(\mu)}{a_1^2(\mu)}\frac{1}{\varepsilon^2}+O(\frac{1}{\varepsilon})\\
    &=\frac{\mu^{2k^\star-2}}{(\frac{k^\star}{2}\mu^{k^\star-1} +\frac{k^\star-1}{2\mu}\mu^{2(k^\star-2)})^2}\frac{1}{(\mu^2-1)^2}+O(\frac{1}{\mu^2-1})\sim \frac{4}{(2k^\star -1)^2}\frac{1}{(\mu^2-1)^2}.
\end{align*}

\end{proof}

\begin{figure}[!htbp]
    \centering
    \includegraphics[width=0.5\linewidth]{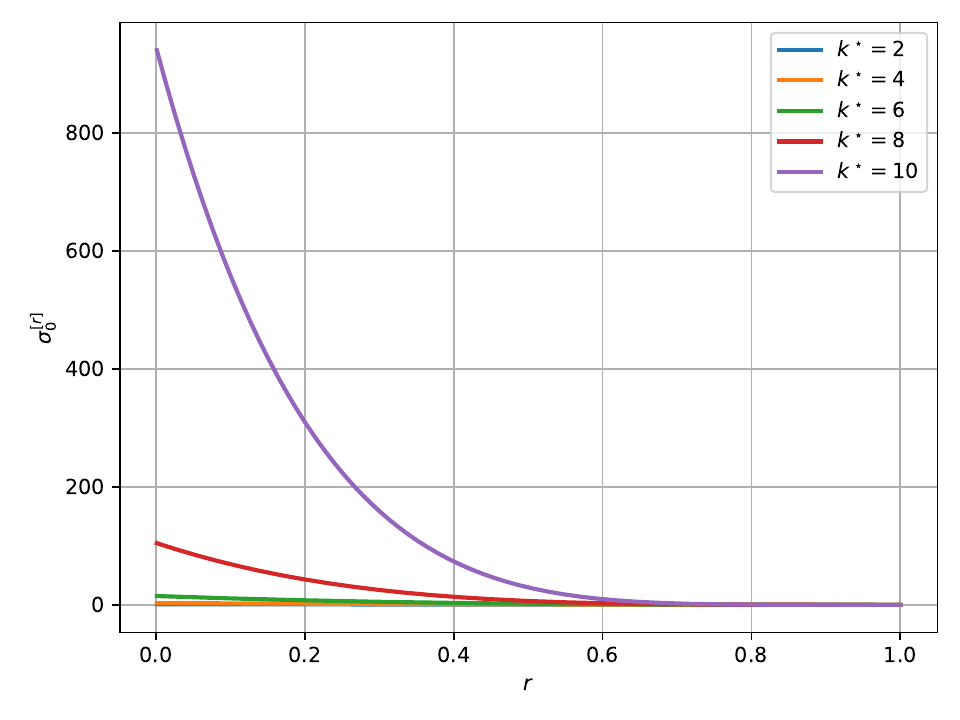}
    \caption{Mean student output (y-axis) for a pure Hermite activation of degree $k^{\star}$ as a function of the pre-activation variance (x-axis).}
    \label{fig:mean_even_degree_Hermite}
\end{figure}

\begin{figure}
    \centering
    \begin{subfigure}
        \centering
        \includegraphics[width=0.9\textwidth]{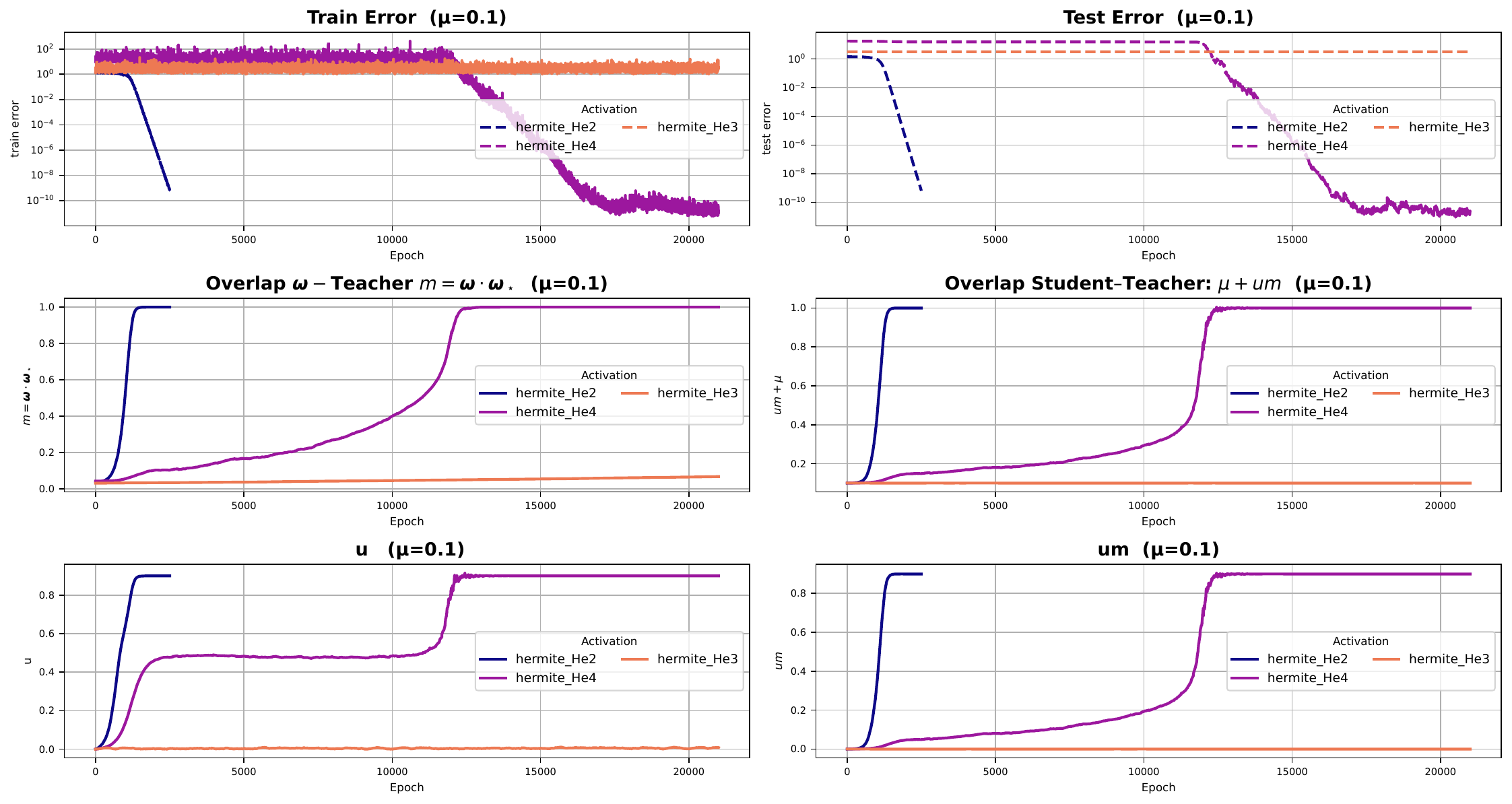}
    \end{subfigure}
    \vspace{-0.5em}
    \begin{subfigure}
        \centering
        \includegraphics[width=0.9\textwidth]{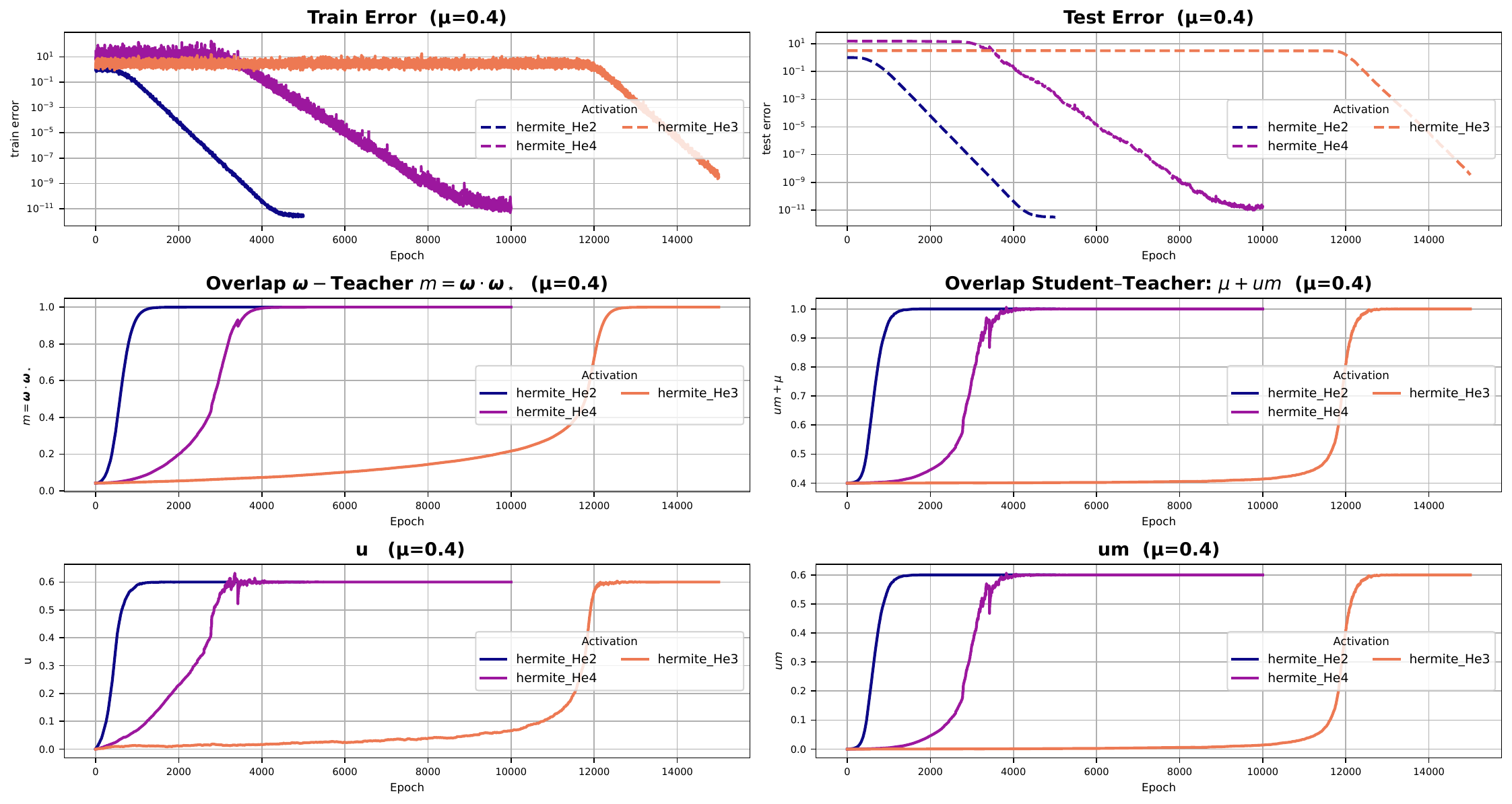}
    \end{subfigure}
    \caption{Learning dynamics of the student model~\eqref{equ:student} for different levels of pre-training alignment $\mu \in \{0.1, 0.4\}$, trained on data generated by the teacher model~\eqref{equ:teacher}. We consider the matched teacher--student activation function to be Hermite polynomial with degree $k\in\{2,3,4\}$ as indicated in the legend. The student was trained using one-pass SGD with batch size $B=5000$ and input dimension $d=1000$. The panels report the dynamic of the test/train mean squared error (MSE), the alignment between the student and teacher directions $m=\boldsymbol{\omega}_\star\!\cdot\!\boldsymbol{\omega}$, the effective teacher--student overlap $m_{\mathrm{eff}}=\mu+u m$, $u$ and the product between overlap $um$.
The results were average over three independent runs.
The spherical constraint is enforced by normalizing $\boldsymbol{\omega}$ after each gradient step and we used  a learning rate $lr=10^{-2}\times \delta_{k} $ with $\delta_k=1/(k!\times k)$.}
    \label{fig:numrics_dynamics_mean_activation_appendix}
\end{figure}

\subsection{Effect of the mean for pure even Hermite activations}
\label{app:mean_even_Hermite}

We consider the case in which the teacher activation function is a Hermite
polynomial of degree $k^{\star}$, and the student employs the same activation.
As shown in \eqref{term_Hermite_student_coefficient}, the Hermite coefficients
of the student model depend on the parity of $k^{\star}$.
In particular, for odd degrees $k^{\star}=2p+1$, the first nonzero coefficient
in the Hermite expansion is $\sigma^{[r]}_{1}$, whereas for even degrees
$k^{\star}=2p$, it is $\sigma^{[r]}_{0}$.
As a direct consequence, the mean of the activation satisfies $\mathbb{E}_{z\sim\mathcal{N}(0,r)}[\sigma(z)]=
\begin{cases}
0,& \text{if } k^{\star}=2p+1,\; p\ge 1, \\[0.3em]
\sigma^{[r]}_0,& \text{if } k^{\star}=2p,\; p\ge 1,
\end{cases}$
where we used the identity \eqref{app:property_Hermite expectation}. For even-degree Hermite activations, the nonzero mean is $\sigma^{[r]}_0=k^{\star}!\frac{(r-1)^{k^{\star}/2}}{2^{k^{\star}/2}\,(\tfrac{k^{\star}}{2})!}.$
This nonzero mean induces an additional contribution to the early-time dynamics
of the order parameters $u$ and $m$.
Importantly, this effect vanishes continuously as $r \to 1$, as illustrated in
Fig.~\ref{fig:mean_even_degree_Hermite}.

In the regime considered in Appendix~\ref{app:linearize_the_population_loss},
where $r \simeq \mu^2$, increasing the initial signal strength $\mu$ reduces the
magnitude of the zeroth Hermite coefficient $\sigma^{[r]}_0$, and thus weakens
the corresponding mean effect.
Numerically, this manifests as the student dynamics after escaping the correlated search phase exhibiting an initial
plateau for small $\mu$, before escaping toward alignment, whereas this plateau
progressively disappears as $\mu$ increases (see
Fig.~\ref{fig:numrics_dynamics_mean_activation_appendix}).

\paragraph{Numerical evidence of the effect of the presence of the mean on the dynamics of Hermite polynomials}
\label{app:numerics_bias_even_Hermite}

When considering Hermite activations, as predicted by our theory (cf.\ Figure~\ref{fig:escaping_time_IE_1_matching_activation}), we observe that for odd degrees the escape time diverges as $\mu \to 0^{+}$, i.e., $\tau(\mu)\to +\infty$, whereas for even degrees the escape time $\tau$ remains finite. In the main text, we argued that this qualitative difference arises from the fact that students with even-degree Hermite activations exhibit a non-zero mean, while those with odd-degree activations do not. In Figure~\ref{fig:numrics_dynamics_mean_activation_appendix}, we provide numerical evidence supporting this claim. Furthermore, Figure~\ref{fig:mean_even_degree_Hermite} illustrates how this mean depends on both~$\mu$ and the Hermite degree. The key takeaway is that whenever the mean is non-vanishing, learning initially proceeds much faster; however, after this component is learned, the dynamics enter an extended plateau before higher-order features are eventually acquired.

\section{A physical interpretation of LoRA dynamics}
\label{app:ODE_of_the_damp_harmonic_oxillator}

Using the linearized dynamics in the correlated search phase
\eqref{equ:dynamic_u_m_correlated_search_phase}, which are fully characterized
by the coefficients $A$ and $B$, the evolution of the order parameters admits a
useful interpretation in terms of an equivalent physical system near an
unstable fixed point. In particular, we show that the effective dynamics in
this regime are formally equivalent to those of a damped nonlinear oscillator.
Such physics-inspired viewpoints have long been used in machine learning, both
to motivate optimization algorithms and to provide theoretical insight into
learning dynamics~\cite{mezard1987spin,choromanska2015loss,mehta2019high}.

Concretely, the coupled first-order system in
\eqref{equ:dynamic_u_m_correlated_search_phase} can be rewritten as the
second-order ordinary differential equation
\begin{equation}
\label{equ:ODE_of_the_damp_harmonic_oxillator}
     \ddot{g}(t) - B\,\dot{g}(t) + \frac{\partial V}{\partial g}
     = 0,
\end{equation}
which describes the motion of a particle in a one-dimensional nonlinear
potential $V(g) = -A^{2}\log\!\big(\cosh g\big)$.
The relation $g(t)=\tanh^{-1} m(t)$ links the auxiliary variable $g$ to the order
parameter $m$, so that $g\to\pm\infty$ corresponds to $m\to\pm 1$. Within this mechanical analogy, the term proportional to $\dot g$ acts as a
linear velocity-dependent force: it is dissipative when $B<0$, corresponding to
effective damping, and amplifying when $B>0$, corresponding to negative damping.
Accordingly, $B$ controls whether the dynamics locally dissipate or inject
energy, while $A$ sets the curvature of the potential and thus the strength of
the signal-induced instability. Altogether, this suggests that the right part of the LoRA block primarily determines the effective directions and curvature that drive learning, while the left part mainly controls the time scale at which the dynamics progress along these directions. Detailed
derivations are provided in
Appendix~\ref{app:ODE_of_the_damp_harmonic_oxillator_proof}.

\emph{For linear activations, we have $A = (1-\mu)$ and $B=-1$. In the physics interpretation, the dynamics are equivalent to a particle rolling down a cliff, initially placed at the top (cf. Fig.\ref{fig:linear_potential}). Converging to the global minimizer corresponds to the particle moving away from the top of the cliff. The parameter $A$ controls the slope of the cliff, implying that smaller values of~$\mu$ correspond to steeper slopes and therefore faster escape from the top. The parameter $B$, on the other hand, acts as an additional force applied to the particle during its descent: a negative value of~$B$ corresponds to a dissipative (damping) force,
whereas a positive value would correspond to an amplifying (impulsive) force.
} See Appendix~\ref{app:ODE_of_the_damp_harmonic_oxillator_proof_linear} for detailed derivations.

\subsection{Proof of the physics interpretation of the linearized dynamics}
\label{app:ODE_of_the_damp_harmonic_oxillator_proof}

Our analysis holds in the regime described in
Appendix~\ref{app:linearize_the_population_loss}, where higher-order powers of
the order parameters can be neglected. In this regime, the learning dynamics
are governed by the closed system of ordinary differential equations given in
Eq.~\eqref{equ:dynamic_u_general}. Since these equations are linear in the order
parameters, they define a system of first-order linear differential equations.
In this appendix, we show that this system can be recast as a second-order
differential equation, which admits a natural interpretation as a
one-dimensional dissipative mechanical system.

\paragraph{Change of variables.}
We introduce the following change of variables:
\begin{equation}
  m(t) = \tanh g(t),
\end{equation}
where $g(t) \in \mathbb{R}$. Differentiating with respect to time yields
\begin{align}
  \dot{m}(t)
  &= \dot{g}(t)\cosh^{-2} g(t)
   = \dot{g}(t)\bigl(1 - m(t)^2\bigr).
\end{align}
In the linearized regime considered in
Appendix~\ref{app:linearize_the_population_loss}, the overlap $m(t)$ remains
small throughout the search phase, so that $m(t)^2 \ll 1$. As a consequence, we
may approximate
\begin{equation}
  \dot{m}(t) \simeq \dot{g}(t).
\end{equation}

\paragraph{Derivation of the second-order equation.}
Using the second dynamical equation in
\eqref{equ:dynamic_u_general}, differentiating once more with respect to time
gives
\begin{equation*}
  \ddot{g}(t) = A \dot{u}(t).
\end{equation*}
From Eq.~\eqref{equ:dynamic_u_general}, the evolution of $u(t)$ is
\begin{equation*}
  \dot{u}(t) = B u(t) + A m(t).
\end{equation*}
Substituting this expression into the equation for $\ddot{g}(t)$ and using
$u(t) \simeq \dot{g}(t)$, we obtain
\begin{align*}
  \ddot{g}(t)
  &= A \bigl(B u(t) + A m(t)\bigr)
   = B \dot{g}(t) + A^2 \tanh g(t).
\end{align*}
Rearranging terms, we arrive at the second-order differential equation
\begin{equation}
  \ddot{g}(t) - B \dot{g}(t) - A^2 \tanh g(t) = 0.
  \label{eq:newtonian_r}
\end{equation}

\paragraph{Mechanical interpretation.}
Equation~\eqref{eq:newtonian_r} can be written in Newtonian form with a
velocity-dependent force:
\begin{equation}
  \ddot{g}(t) - B \dot{g}(t) + \frac{\partial V(g)}{\partial g} = 0,
\end{equation}
where the effective potential $V(g)$ is
\begin{equation}
  V(g) = - A^2 \log\,\bigl(\cosh g\bigr).
\end{equation}
This representation interprets the learning dynamics as the motion of a
particle in a one-dimensional potential $V(g)$, subject to a linear force
proportional to the velocity. Importantly, this term is \emph{dissipative} only
when $B<0$ (effective friction), while it becomes \emph{amplifying} when $B>0$
(negative friction).

This distinction is made explicit by introducing the associated mechanical
energy
\begin{equation}
  E(t) := \frac{1}{2}\dot{g}(t)^2 + V\bigl(g(t)\bigr).
\end{equation}
Differentiating and using the equation of motion yields
\begin{align}
  \frac{\mathrm{d}}{\mathrm{d}t}E(t)
  &= \dot{g}(t)\ddot{g}(t) + V'\bigl(g(t)\bigr)\dot{g}(t) = \dot{g}(t)\Bigl(\ddot{g}(t) + V'\bigl(g(t)\bigr)\Bigr)
   = B\,\dot{g}(t)^2.
  \label{eq:energy_balance_B}
\end{align}
Therefore, if $B<0$ then $\dot{E}(t)\le 0$ and the dynamics monotonically
dissipate energy, consistent with a frictional interpretation. Conversely, if
$B>0$ then $\dot{E}(t)\ge 0$ and the velocity-dependent term injects energy into
the system, leading to effective self-amplification. Hence, the qualitative
behavior (plateaus, instabilities, and escape times) is governed by the
competition between the driving set by the potential shape (controlled by $A$)
and the sign and magnitude of the velocity-dependent term controlled by $B$.

\begin{figure}
    \centering
    \includegraphics[width=0.5\linewidth]{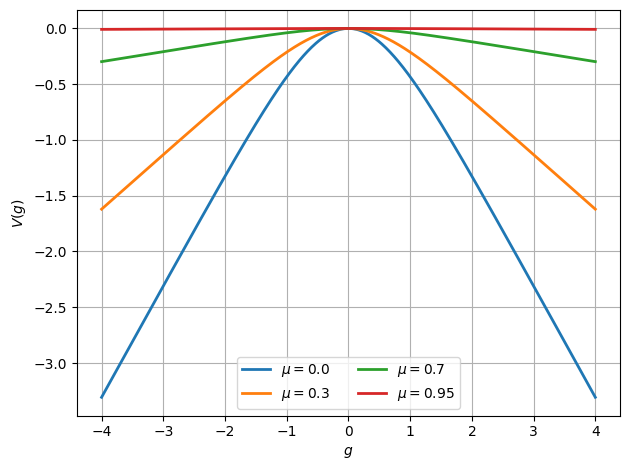}
    \caption{Effective potential for linear activation}
    \label{fig:linear_potential}
\end{figure}

\paragraph{Case of linear activation}
\label{app:ODE_of_the_damp_harmonic_oxillator_proof_linear}

When both the teacher and the student use a linear activation function, the
coefficients reduce to $B=-1$ and $A=1-\mu$, and
Eq.~\eqref{eq:newtonian_r} becomes
\begin{equation}
\label{diff_oscillation}
  \ddot{g}(t) + \dot{g}(t) - (1-\mu)^2 \tanh g(t) = 0.
\end{equation}
This is a nonlinear second-order ordinary differential equation.

\begin{itemize}
\item \textbf{Case $\mu=1$.}
When $\mu=1$, the nonlinear term vanishes and
Eq.~\eqref{diff_oscillation} reduces to a linear equation, which admits the
explicit solution
\begin{equation*}
  g(t) = g_0 + \dot{g}_0\bigl(1 - e^{-t}\bigr).
\end{equation*}
In terms of the original order parameters, this corresponds to
\begin{equation*}
  u(t) = u_0 e^{-t}, \qquad m(t) = m_0.
\end{equation*}
Hence, $u(t)\to 0$ as $t\to\infty$, while the overlap $m(t)$ remains frozen at its
initial value. For a random initialization
$\boldsymbol{\omega}\sim \mathrm{Unif}(\mathbb{S}^{d-1})$, we have
$m_0 = O(d^{-1/2})$, so that the initial alignment vanishes in the large-$d$
limit with high probability. In this regime, the associated population loss
satisfies $\mathcal{L}(u_0,m_0)=o(1)\; \text{as}\; d\to +\infty$, indicating that the system is initialized in the vicinity of the teacher solution.

\item \textbf{Case $0 \le \mu < 1$.}
For $\mu<1$, Eq.~\eqref{diff_oscillation} does not admit a closed-form analytical
solution. Nevertheless, it admits a clear mechanical interpretation as the
damped motion of a particle in the effective potential
\begin{equation}
  V(g) = - (1-\mu)^2 \log\!\bigl(\cosh g\bigr).
\end{equation}
For any $\mu<1$, this potential is concave and symmetric around $g=0$, with a
curvature that vanishes as $\mu \to 1$. A random initialization corresponds to
$m\simeq 0$, and hence $g\simeq 0$, placing the particle near the unstable
maximum of the potential. Perfect alignment corresponds to $m(t)=\pm 1$, or
equivalently $g(t)\to \pm\infty$, which can be interpreted as the particle
escaping from the top of the potential. The convergence rate toward alignment is
controlled by the competition between the potential-induced instability (cf. Fig. \ref{fig:linear_potential}),
parametrized by $(1-\mu)^2$, and the effective damping coefficient, which in
this case is equal to $1$.
\end{itemize}

\section{Descent phase}
\label{descent_phase}
In this section, we aim to provide mathematical results showing that, once the
dynamics enter the \textit{descent phase}, they necessarily converge toward a global minimizer of the population loss. Our results are built on the following
assumption on the population loss.

\begin{assumption}[Population loss]
\label{ass:grad_loss_function_m}
There exist $\eta\in(0,1)$ and $c>0$ such that, for all $u,m\in\mathbb{R}$ with
$|m|\in[\eta,1),\, \operatorname{sign}(m)\,\partial_m \mathcal{L}(u,m) \;\le\; -c.$
\end{assumption}

In this phase, at least one of the order parameters scales with $\mu$, so that
the linear approximation used in the correlated search phase is no longer
valid and the full structure of the population loss must be taken into
account. The dynamics of the order parameters in the descent phase are governed
by
\begin{equation}
\label{equ:dynamic_descent_phase}
    \dot m = -\,\phi_m(u,m)\,(1-m^2),
    \qquad 
    \dot u = -\,\phi_u(u,m),
\end{equation}
where $\phi_m := \partial_m \mathcal{L}(u,m)$ and 
$\phi_u := \partial_u \mathcal{L}(u,m)$.

The population loss $\mathcal{L}(u,m)$ acts as a Lyapunov function for the flow
\eqref{equ:dynamic_descent_phase}. Taking its time derivative along the
trajectories yields
\begin{equation}
\label{equ:dynamic_lyapunov}
\frac{d\mathcal{L}(u,m)}{dt}
= -\,\phi_u(u,m)^2 - \phi_m(u,m)^2\,(1-m^2)
\;\le\; 0,
\end{equation}
with equality if and only if $\phi_u(u,m)=0$ or $|m|=1$.
Equation~\eqref{equ:dynamic_lyapunov} shows that the dynamics monotonically
decrease the population loss. Since both $u$ and $m$ are nontrivial along the
trajectory, the dynamics converge to a minimizer of the population loss. The following result guarantees
that, in the descent phase, the overlap satisfies $|m|\to 1$.

\begin{proposition}
\label{prop:exp_convergence_m}
Assume that Assumption~\ref{ass:grad_loss_function_m} holds, and let $(u(t),m(t))$
be a solution of \eqref{equ:dynamic_descent_phase} with
$|m(t_{0})|\in[\eta,1)$ for some $t_0\ge t_{\rm exit}$. Then there exists a
constant $c>0$ (depending only on $\mathcal{L}$) such that, for all $t\ge t_0$,
\begin{equation}
    1 - |m(t)|
    \;\le\;
    \bigl(1 - |m(t_0)|\bigr)\,e^{-2c\,(t-t_0)}.
\end{equation}
In particular, $|m(t)|\to 1$ exponentially fast as $t\to\infty$, and the sign of
$m(t)$ remains constant for all $t\ge t_0$.
\end{proposition}

We showed above that $|m(t)| \to 1$. It remains to characterize the
asymptotic behavior of the scalar parameter $u$. It is important to keep in
mind that, although $|m|=1$, in the presence of teacher--student activation
mismatch the global minimizer of the population loss does not necessarily
correspond to recovery of the ground-truth direction, meaning that
$m_{\rm eff}=\mu+um\neq \pm 1$. Therefore, full recovery of the teacher
direction is guaranteed only in the matching case $\phi(\cdot)=\sigma(\cdot)$.
To prove this, we require the following assumption.

\begin{assumption}[Slice well-posedness and sign-identifiability]
\label{ass:slice_wellposed_identifiable_relu}
Assume the matching setting $\sigma=\phi$. The activation
$\sigma:\mathbb R\to\mathbb R$ is locally Lipschitz (hence differentiable
a.e.) and of at most polynomial growth, so that the slice losses defined below
are finite and $C^1$ in $u$.

Fix $s\in\{-1,+1\}$ and define the slice loss
$$\mathcal{L}(u,s)=\mathbb E\Big[\big(\sigma(\lambda_\star)- \sigma((\mu+us)\lambda_\star)\big)^2\Big],\qquad
\lambda_\star\sim\mathcal N(0,1).$$

We assume:
\begin{enumerate}
\item[i)] \textbf{Convex slice with bounded sublevel sets (no escape to infinity).}
For each $s\in\{-1,+1\}$, the map $u\mapsto \mathcal L(u,s)$ is convex.
Moreover, for the relevant loss level $\ell_0:=\mathcal L(u(t_0),s)$ (with $t_0$
as in Proposition~\ref{prop:conv_u_given_m_parity}), the sublevel set
$\{u\in\mathbb R:\ \mathcal L(u,s)\le \ell_0\}$ is bounded.

\item[ii)] \textbf{No nontrivial scaling symmetries.}
If there exists $a\neq 0$ such that $\sigma(a z)=\sigma(z)$ for all
$z\in\mathbb R$, then $a\in\{-1,+1\}$.

\item[iii)] \textbf{Realizability on the slice.}
For each $s\in\{-1,+1\}$, the minimum is attained at zero loss:
$\min_u \mathcal L(u,s)=0$.
\end{enumerate}
\end{assumption}

\begin{remark}
Convexity in Assumption~\ref{ass:slice_wellposed_identifiable_relu} is not
assumed to be strict; in particular, when $\sigma$ is even, the slice loss may
admit two global minimizers.
\end{remark}

Under this assumption, the following proposition holds.

\begin{proposition}
\label{prop:conv_u_given_m_parity}
Assume the matching teacher--student setting $\sigma=\phi$ and that
Assumption~\ref{ass:slice_wellposed_identifiable_relu} holds.
Let $(u(t),m(t))$ be any solution of \eqref{equ:dynamic_descent_phase} such that
$m(t)\longrightarrow s\in\{-1,+1\}\qquad\text{as } t\to\infty.$

Then $u(t)$ converges to a global minimizer of the slice loss
$\mathcal L(\,\cdot\,,s)$. Moreover,
$$
u(t)\longrightarrow
\begin{cases}
s(1-\mu), & \text{if $\sigma$ is not even},\\[4pt]
s(1-\mu)\ \text{or}\ s(-1-\mu), & \text{if $\sigma$ is even}.
\end{cases}
$$
\end{proposition}

\begin{remark}
In the matching activation setting, the above propositions guarantee that once
the dynamics escape the correlated search phase, they converge so as to recover
the teacher direction. No such guarantee can be made in the presence of
activation mismatch between teacher and student. Proofs of
Propositions~\ref{prop:exp_convergence_m} and
\ref{prop:conv_u_given_m_parity} are provided in
Appendix~\ref{app_proof_m_to_1_u_to_1}.
\end{remark}

\subsection{Descent phase: proof of Proposition 
\ref{prop:exp_convergence_m} and Proposition \ref{prop:conv_u_given_m_parity}}
\label{app_proof_m_to_1_u_to_1}

\begin{proof}
\underline{Proof of Proposition~\ref{prop:exp_convergence_m}.}
Write $\phi_m(u,m) = \partial_m \mathcal{L}(u,m)$ and define
$$
h(u,m) := -\,\phi_m(u,m) = -\,\partial_m \mathcal{L}(u,m).
$$
By Assumption~\ref{ass:grad_loss_function_m}, for all $u$ and all $m$ with
$|m|\in[\eta,1)$,
$$
\operatorname{sign}(m)\,\partial_m \mathcal{L}(u,m) \;\le\; -c,
\quad\text{so that}\quad
-\,\operatorname{sign}(m)\,\partial_m \mathcal{L}(u,m) \;\ge\; c.
$$

The $m$-dynamics in \eqref{equ:dynamic_descent_phase} is
$$
\dot m = -\,\phi_m(u,m)\,(1-m^2)
       = -\,\partial_m \mathcal{L}(u,m)\,(1-m^2).
$$
For $t\ge t_0$, since $|m(t_0)|\in(\eta,1)$ and the drift always pushes $|m|$
away from zero in the region $|m|\in[\eta,1)$, we have $|m(t)|\in[\eta,1)$
and $\operatorname{sign}(m(t)) = \operatorname{sign}(m(t_0))$ for all $t\ge t_0$.

Set $s(t):=|m(t)|$. Then
$$
\dot s(t)
= \frac{d}{dt}|m(t)|
= \operatorname{sign}(m(t))\,\dot m(t)
= -\,\operatorname{sign}(m(t))\,\partial_m \mathcal{L}(u(t),m(t))\,(1-m(t)^2).
$$
Using the assumption
$-\,\operatorname{sign}(m)\,\partial_m \mathcal{L}(u,m) \ge c$
for $|m|\in[\eta,1)$ and the identity $1-m(t)^2 = 1-s(t)^2$, we obtain
$$
\dot s(t) \;\ge\; c\,(1-s(t)^2).
$$

Define $g(t) := 1 - s(t)$. Then
\begin{align*}
    \dot g(t)
    &= -\dot s(t)
     \;\le\; -c\,(1-s(t)^2) \\
    &= -c\,(1-s(t))\,(1+s(t))
     \;\le\; -c\,(1+\eta)\,g(t),
\end{align*}
since $s(t)\ge\eta$ for all $t\ge t_0$.
Renaming $2c := c(1+\eta)$ (i.e., absorbing constants into $c$), we obtain
$$
\dot g(t) \;\le\; -2c\,g(t).
$$
By Grönwall's inequality, for all $t\ge t_0$,
$$g(t) \;\le\; g(t_0)\,e^{-2c\,(t-t_0)},$$
that is,
$$ 1 - |m(t)| \;\le\;
    \bigl(1 - |m(t_0)|\bigr)\,e^{-2c\,(t-t_0)}.
$$
Hence $|m(t)|\to 1$ exponentially fast as $t\to\infty$, and the sign of $m(t)$
cannot change after $t_0$.
\end{proof}

By Proposition~\ref{prop:exp_convergence_m}, in the descent phase we have 
$|m(t)|\to 1$ and the sign of $m(t)$ is constant. We now show that, 
conditional on $m(t)\to s\in\{-1,1\}$, the scalar parameter $u(t)$ converges 
to the corresponding optimal value $s(1-\mu)$.

\begin{proof}
\underline{Proof of Proposition~\ref{prop:conv_u_given_m_parity}.}
Fix $s\in\{-1,1\}$. In the matching-activation case $\sigma=\phi$, and for $m=s$,
the population loss along the teacher direction can be written as
$$\mathcal L(u,s)
=
\mathbb E_{\lambda_\star}\Big[
\big(\sigma(\lambda_\star)-\sigma((\mu+us)\lambda_\star)\big)^2
\Big],
\qquad \lambda_\star\sim\mathcal N(0,1),$$
which  is nonnegative. 

Moreover, by definition,
$\mathcal L\bigl(s(1-\mu),s\bigr)=0,$
and, if $\sigma$ is even, also $\mathcal L\bigl(s(-1-\mu),s\bigr)=0.$ Hence $\min_u \mathcal L(u,s)=0$.

We first show convergence to the set of global minimizers.
Along the dynamics~\eqref{equ:dynamic_descent_phase}, we have
$$
\frac{d}{dt}\,\mathcal L(u(t),m(t))
= \phi_u(u,m)\,\dot u + \phi_m(u,m)\,\dot m
= -\,\phi_u(u,m)^2 - (1-m^2)\phi_m(u,m)^2
\;\le\; 0.
$$

Therefore, $\mathcal L(u(t),m(t))$ is nonincreasing and bounded from below, and
thus converges as $t\to\infty$. By Assumption~\ref{ass:slice_wellposed_identifiable_relu}\,(i), the slice loss $u\mapsto\mathcal L(u,s)$ is convex and has bounded sublevel sets. Since
$\mathcal L(u(t),s)$ is nonincreasing along the trajectory, it follows that
$u(t)$ remains in a bounded sublevel set and therefore remains bounded for all
$t\ge t_0$. Consequently, the trajectory admits limit points.

Let $(u_\infty,m_\infty)$ be any limit point.
Since $m(t)\to s$ by hypothesis, we have $m_\infty=s$.
Moreover, since $\mathcal L(u(t),m(t))$ converges and its derivative is a sum of
negative squares, it follows that
$$
\phi_u(u(t),m(t))\longrightarrow 0
\qquad\text{as }t\to\infty.
$$

By continuity of $\phi_u$ and the fact that $m(t)\to s$, any limit point
$u_\infty$ of $u(t)$ satisfies $\phi_u(u_\infty,s)=0$.
Since $u\mapsto \mathcal L(u,s)$ is convex, every critical point is a global
minimizer, hence $u_\infty\in\arg\min_u \mathcal L(u,s)$.
Because the trajectory $u(t)$ is bounded and all its limit points belong to the
set of minimizers, we conclude that $u(t)$ converges to a global minimizer of
$\mathcal L(\,\cdot\,,s)$.

It remains to identify the minimizers.
Let $u$ be such that $\mathcal L(u,s)=0$.
Then
$$
\sigma((\mu+us)z)=\sigma(z)
\qquad\text{for a.e.\ } z\in\mathbb R.$$

By Assumption~\ref{ass:slice_wellposed_identifiable_relu}\,(ii),
this implies $\mu+us\in\{-1,+1\}$.
If $\sigma$ is not even, the identity $\sigma(-z)=\sigma(z)$ fails, so the case
$\mu+us=-1$ is impossible, and the only minimizer is $u=s(1-\mu)$.
If $\sigma$ is even, both $\mu+us=\pm1$ are admissible, yielding the two
minimizers $u=s(1-\mu)$ and $u=s(-1-\mu)$.

Therefore,
$$
u(t)\longrightarrow
\begin{cases}
s(1-\mu), & \text{if $\sigma$ is not even},\\[4pt]
s(1-\mu)\ \text{or}\ s(-1-\mu), & \text{if $\sigma$ is even}.
\end{cases}
$$
\end{proof}

\section{The role of the loss function}
\label{app:different_loss}

So far, our analysis has focused on the mean squared loss. In this appendix, we
investigate how the choice of a different loss function affects the role played
by the signal strength $\mu$. More precisely, we consider the correlation loss
\begin{equation}
    \ell(\hat{y},y) = 1 - \hat{y}y,
\end{equation}
which belongs to the correlated statistical query (CSQ) family and was also
employed in recent work such as~\cite{damian2023smoothing,gerace2024gaussian,arnaboldi2024repetita}.

We consider the teacher--student setting with linear activation functions for
both the teacher and the student. In this case, the population loss takes the
form
\begin{equation}
    \mathcal{L}(u,m)
    := \mathbb{E}_{\boldsymbol{x}}
    \Big[ 1 - (\omeg_\star \cdot \boldsymbol{x})
    \big((\mu \omeg_\star + u \omeg)\cdot \boldsymbol{x}\big) \Big]
    = 1 - \mu - u m,
\end{equation}
where $m := \omeg \cdot \omeg_\star$ denotes the overlap between the student
direction $\omeg$ and the teacher direction $\omeg_\star$.

The gradients of the population loss with respect to the order parameters $u$
and $m$ are given by
\begin{equation}
    \frac{\partial \mathcal{L}}{\partial u} = -m,
    \qquad
    \frac{\partial \mathcal{L}}{\partial m} = -u.
\end{equation}
Notably, the gradient of the population loss is completely independent of the
signal strength~$\mu$.

\paragraph{Interpretation.}
In the setting considered here, the fact that the population gradient does not
depend on $\mu$ has a direct and striking consequence: an arbitrarily strong
planted component of the teacher direction $\omeg_\star$ in the student model
does not influence the learning dynamics. In other words, the optimization
problem behaves as if it were effectively blind to the initial signal encoded
by~$\mu$. This behavior is in sharp contrast with what is observed when using
the quadratic loss, where the signal strength $\mu$ plays a crucial role in
shaping the learning dynamics and the escape from the search phase. The key
observation here concerns the structure of the population gradient field, which
remains entirely independent of the signal strength~$\mu$. The present observation suggests that, in practical scenarios such as LoRA
fine-tuning, the choice of the loss function can play a non-trivial and
sometimes dominant role in determining whether and how the pre-trained weight
is exploited during learning.

\section{Numerical evidence of delayed learning near the singularity}
\label{app:numerics_evidence_singularity}

Focusing on the Hermite-$3$ activation $\He_3$, we provide in Figure~\ref{fig:dynamics_Hermite_activation_appendix} numerical evidence of the effect of the pre-training alignment~$\mu$ for a student model trained using one-pass
SGD (further details of the numerical setup are given in the caption). As predicted by the theory and illustrated in Figure~\ref{fig:escaping_time_IE_1_matching_activation}(see panel (b)), the parameter~$\mu$ has a non-trivial impact on the dynamics governing the escape from the correlated search phase. The observations reported here are fully consistent with the theoretical predictions. 
Crucially, near the singular regime at $\mu=0.325$, a significant delay in escaping the search phase is observed numerically, in agreement with the theory. Across the different observables, we find that the quantity that exhibits a mild increase is the overlap~$m$, while the LoRA coefficient~$u$ remains noisy during this stage.

\begin{figure}[!htbp]
    \centering
    \begin{subfigure}
        \centering
        \includegraphics[width=1.0\textwidth]{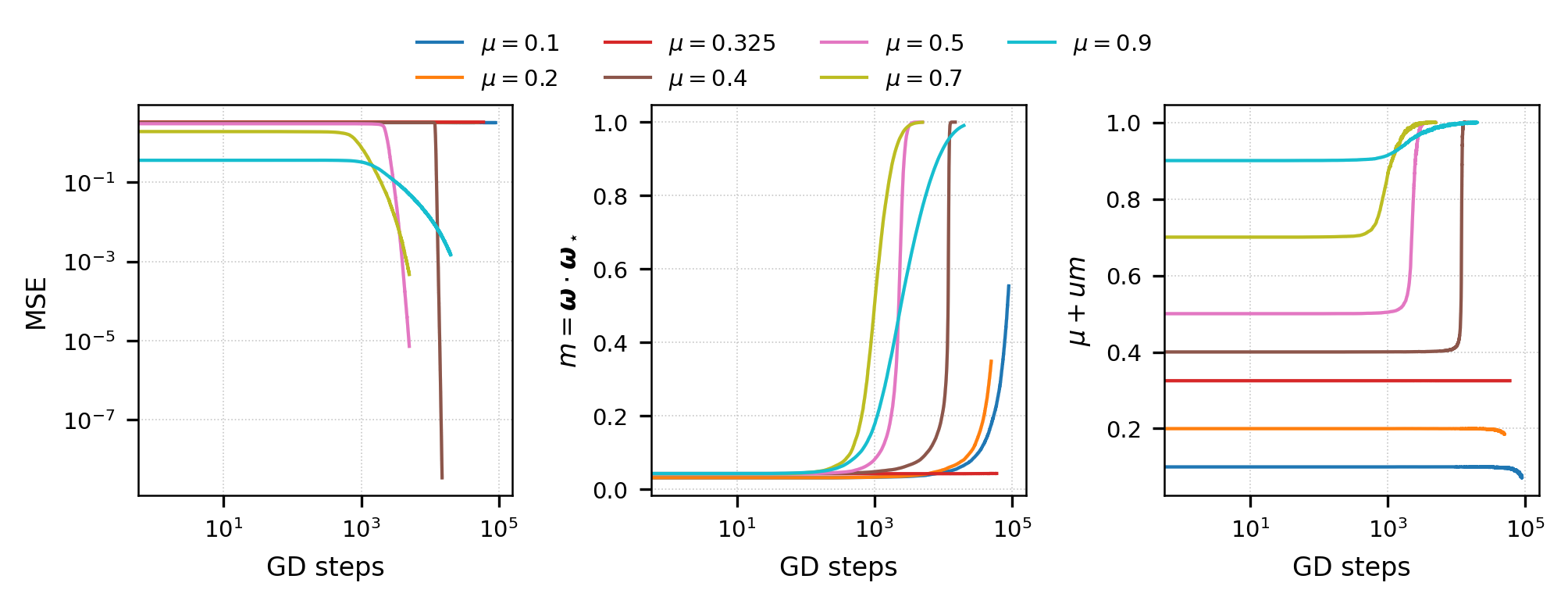}
    \end{subfigure}
    \vspace{-1.0em}
    \begin{subfigure}
        \centering
        \includegraphics[width=1.0\textwidth]{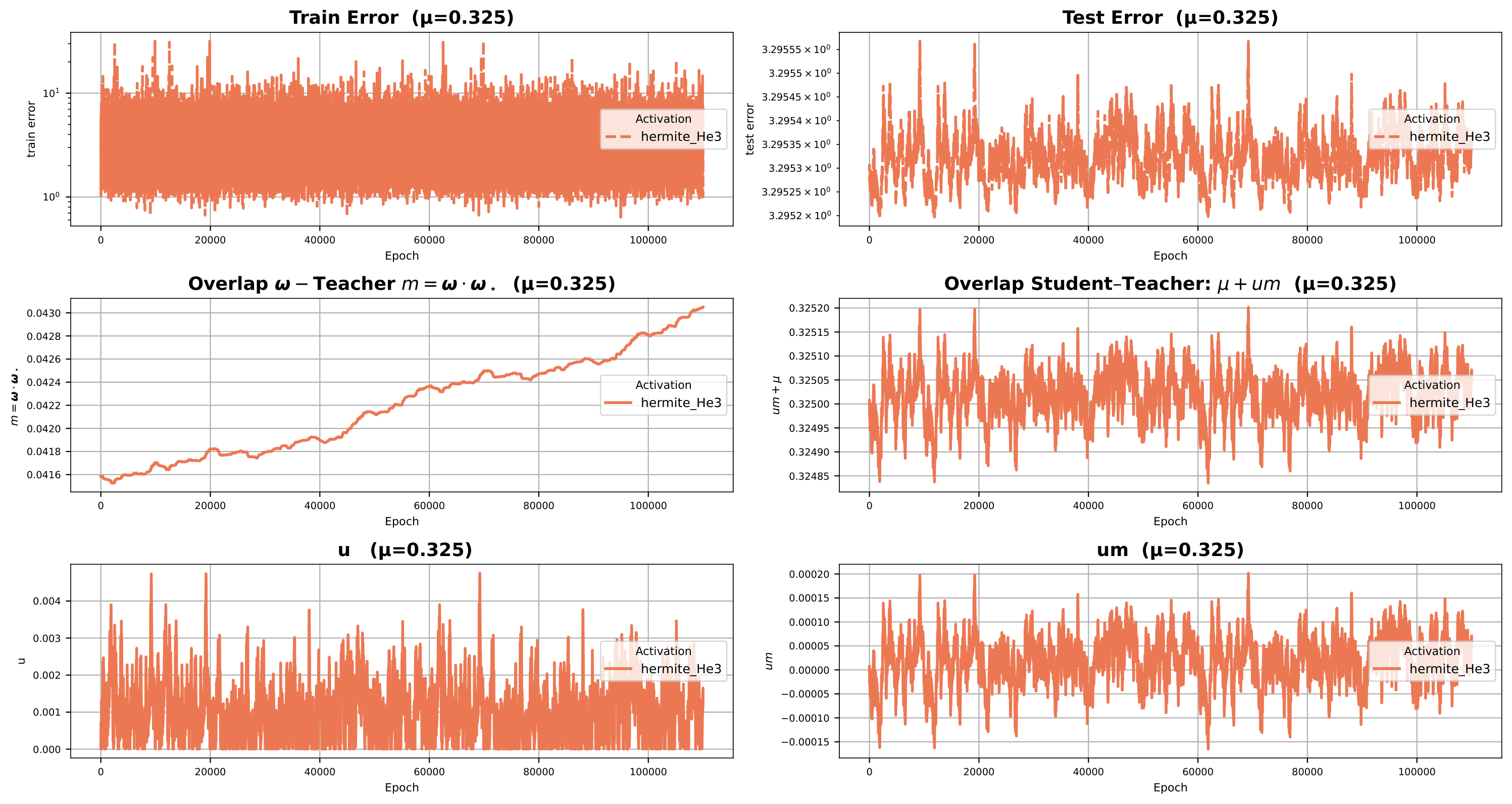}
    \end{subfigure}
   \caption{Learning dynamics of the student model~\eqref{equ:student} for different levels of pre-training alignment
$\mu \in \{0.1,0.2,0.325,0.4,0.5,0.7,0.9\}$, trained on data generated by the teacher model~\eqref{equ:teacher}.
We consider the matched teacher--student activation function to be a Hermite polynomial of degree $k=3$, as indicated in the legend.
The student is trained using one-pass SGD with batch size $B=5000$ and input dimension $d=1000$. \textbf{Top:} the panels report the test mean squared error (MSE) (left), the alignment between the student and teacher directions
$m=\boldsymbol{\omega}_\star\!\cdot\!\boldsymbol{\omega}$ (middle), and the effective teacher--student overlap $m_{\mathrm{eff}}=\mu+u m$ (right). Results are averaged over three independent runs. A spherical constraint is enforced by normalizing $\boldsymbol{\omega}$ after each gradient step, and the learning rate is set to
$lr=10^{-2}\,\delta_k$, with $\delta_k=1/(k!\,k)$.
As shown in the middle panel, the case $\mu=0.325$, which lies close to the singular region (cf.\ panel (b) of Fig.~\ref{fig:escaping_time_IE_1_matching_activation}),
exhibits a significantly longer escape time compared to the other values of $\mu$. \textbf{Bottom:} from the second to the last row, the panels display a more detailed evolution of the observables:
train/test error, the scalar parameter $u$, and the product $u m$ for the case $\mu=0.325$.
}
\label{fig:dynamics_Hermite_activation_appendix}
\end{figure}

\section{Additional Transformer LoRA Transfer Experiments}
\label{app:vit-lora-experiments}

We provide additional details on the controlled vision-transformer LoRA transfer experiments discussed in the main text. In all experiments, a small ViT is first pretrained on a source classification task, checkpoints are extracted along the pretraining trajectory, and the pretrained backbone is then frozen during downstream adaptation. LoRA modules are inserted in the attention projections and trained on the downstream task, together with the classification head when appropriate. Source checkpoints are compared under the same downstream protocol, and curves are averaged over seed-matched runs. The top row of each figure shows the source pretraining dynamics, while the bottom row shows downstream LoRA fine-tuning dynamics from different source checkpoints. Shaded regions indicate standard deviation across seeds.

\subsection{CIFAR-10$\to$SVHN}
\label{app:cifar10-svhn}

For the natural-image transfer experiment, we pretrain a small ViT classifier on the standard CIFAR-10 training set, consisting of $50{,}000$ images, and evaluate source performance on the $10{,}000$ CIFAR-10 test images. CIFAR-10 and SVHN are both $10$-class RGB classification tasks, but their input distributions differ substantially. We save source checkpoints along the CIFAR-10 pretraining trajectory and use each checkpoint to initialize a downstream SVHN model. During downstream adaptation, the pretrained backbone is frozen, and only the LoRA parameters and the readout are trained. The downstream budget is fixed to $100$ SVHN training examples per class, i.e. $1{,}000$ labeled examples in total, while performance is evaluated on the full SVHN test set. The same optimizer, learning rate, LoRA rank, and normalization protocol are used for all source checkpoints. Figure~\ref{fig:cifar10-svhn-2x2} shows that CIFAR-10 source accuracy continues to improve along pretraining, whereas SVHN adaptation is best from an earlier checkpoint.

\begin{figure}[!htbp]
    \centering
    \includegraphics[width=0.8\linewidth]{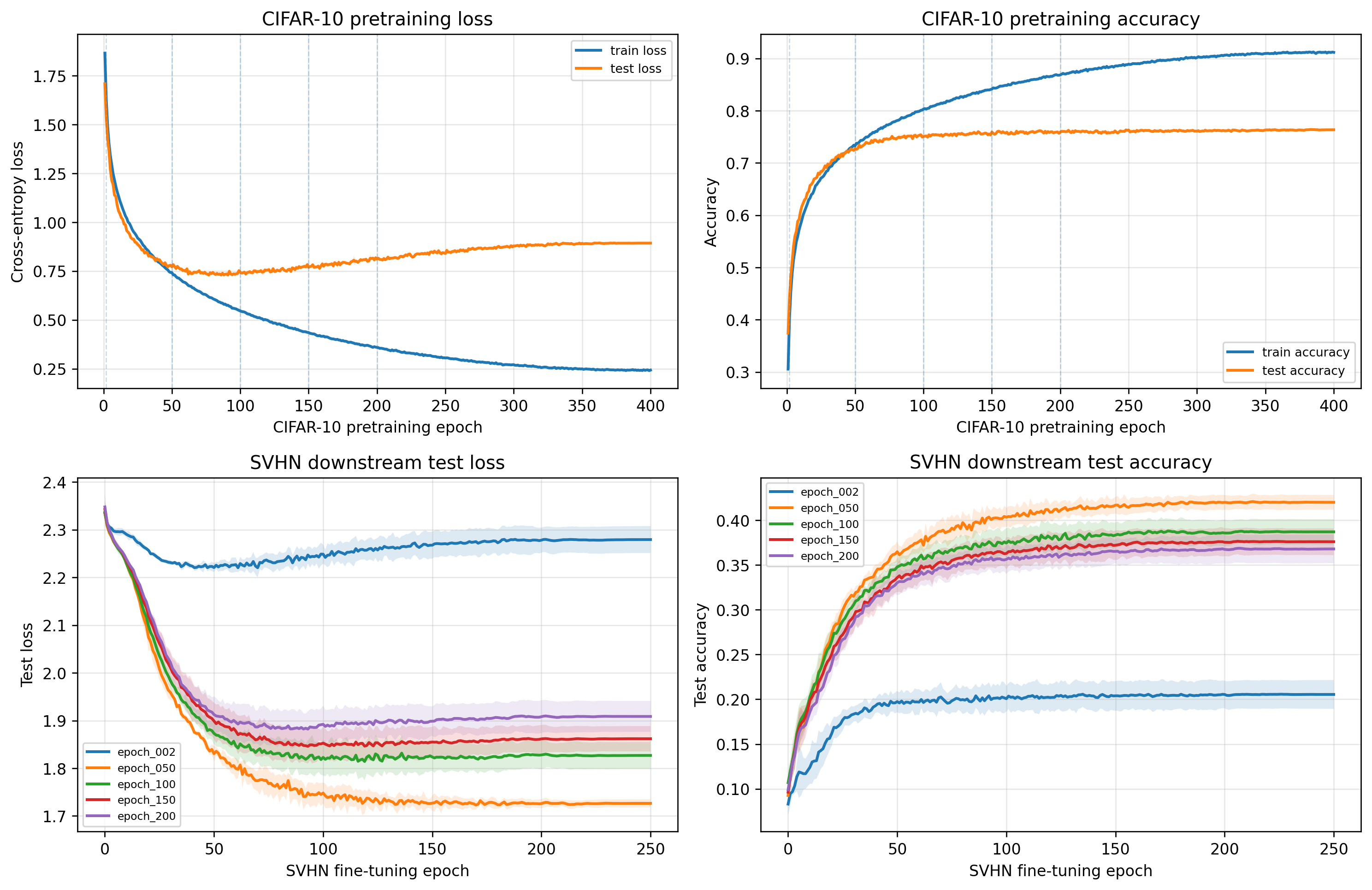}
    \caption{
    CIFAR-10 pretraining and CIFAR-10$\rightarrow$SVHN LoRA adaptation dynamics.
    Top row: CIFAR-10 pretraining loss and accuracy.
    Bottom row: downstream SVHN test loss and accuracy for LoRA fine-tuning initialized from different CIFAR-10 source checkpoints.
    Although CIFAR-10 source accuracy improves along pretraining, downstream SVHN adaptation is best from an earlier checkpoint, illustrating the decoupling between source-task accuracy and downstream LoRA adaptability.
    }
    \label{fig:cifar10-svhn-2x2}
\end{figure}

\subsection{EMNIST$\to$FashionMNIST}
\label{app:emnist-fashionmnist}

For the EMNIST$\to$FashionMNIST experiment, the source task is EMNIST Letters classification with $26$ classes, while the downstream task is FashionMNIST classification with $10$ classes. Thus, both the input distribution and the label space change. We pretrain a small ViT on the standard EMNIST Letters training split, containing $124{,}800$ images, with $10\%$ held out for validation, and evaluate source performance on the $20{,}800$ EMNIST Letters test images. The model uses patch size $4$, embedding dimension $192$, depth $3$, $4$ attention heads, pretraining batch size $256$, learning rate $10^{-3}$, weight decay $10^{-4}$, and a cosine learning-rate schedule. We save source checkpoints at epochs $2,5,10,15,20,50,$ and $100$. For downstream adaptation, the EMNIST backbone is frozen, LoRA modules are trained on the attention projections, and the classification head is reinitialized and trained jointly with LoRA because the label space changes from $26$ to $10$ classes. We use $500$ FashionMNIST examples per class, i.e. $5{,}000$ training examples in total, and evaluate on the full $10{,}000$-image FashionMNIST test set. The downstream fine-tuning uses LoRA rank $8$, LoRA scaling $\alpha=8$, LoRA targets $q,v$, batch size $128$, learning rate $5\times 10^{-4}$, and weight decay $10^{-4}$. Figure~\ref{fig:emnist-fashionmnist-2x2} shows that the latest EMNIST checkpoint has the best source accuracy, but an earlier checkpoint gives better FashionMNIST adaptation after fine-tuning.

\begin{figure}[!htbp]
    \centering
    \includegraphics[width=0.8\linewidth]{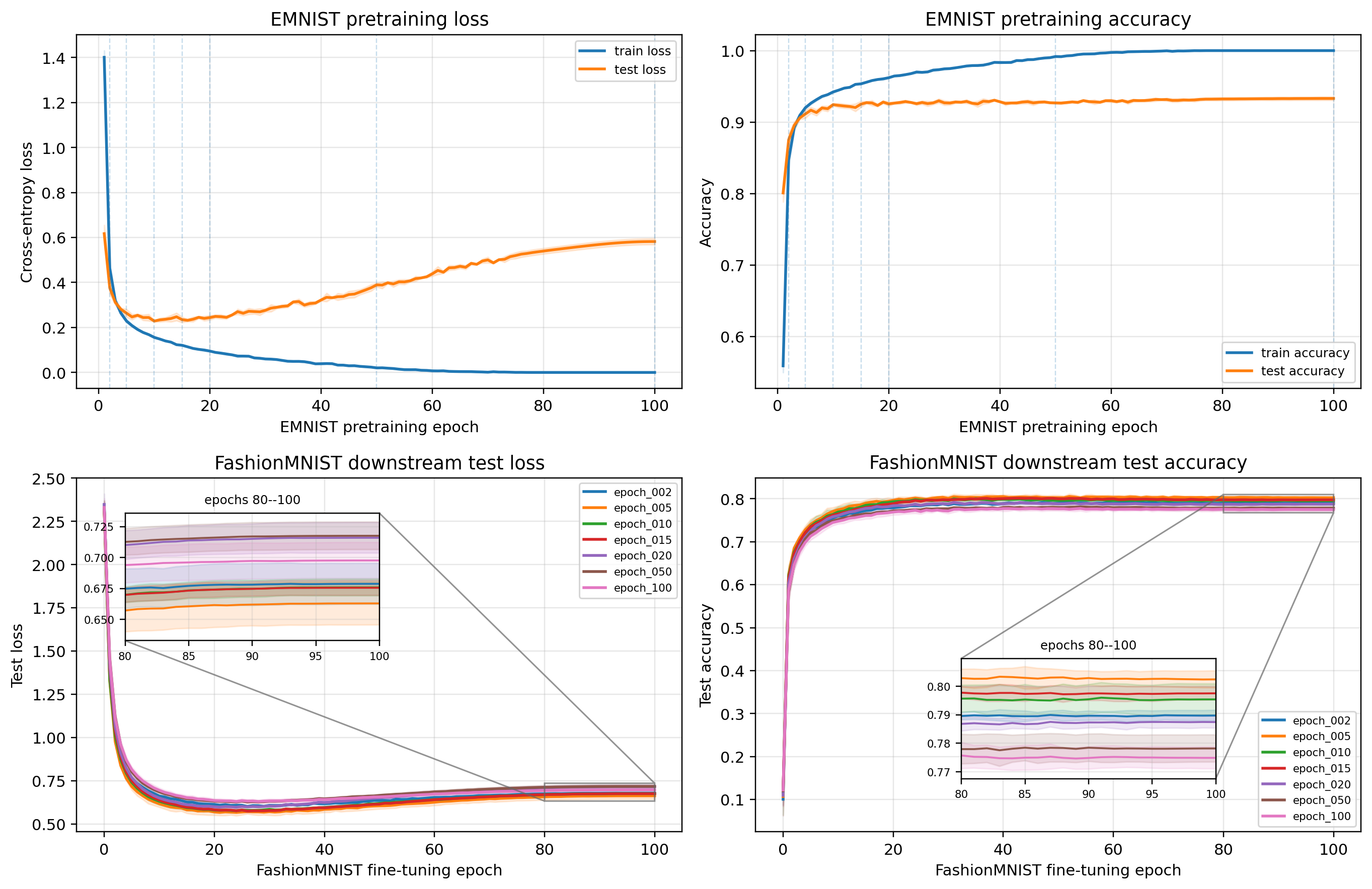}
    \caption{
    EMNIST pretraining and EMNIST$\rightarrow$FashionMNIST LoRA adaptation dynamics.
    Top row: EMNIST pretraining loss and accuracy.
    Bottom row: downstream FashionMNIST test loss and accuracy for LoRA fine-tuning initialized from different EMNIST source checkpoints.
    The downstream task changes both the input distribution and the label space, so the classification head is reinitialized and trained jointly with LoRA. Despite the later EMNIST checkpoints having higher source accuracy, the best FashionMNIST performance is obtained from an earlier source checkpoint. The insets zoom into epochs $80$--$100$, showing that this ordering persists near convergence.
    }
    \label{fig:emnist-fashionmnist-2x2}
\end{figure}

\subsection{EMNIST$\to$shifted-EMNIST}
\label{app:emnist-shifted}

For the shifted-EMNIST experiment, the source and downstream tasks share the same $26$-class EMNIST Letters label space, but the downstream input distribution is shifted. The shift is constructed by applying a rotation of $25^\circ$ and additive noise with standard deviation $0.20$ to the EMNIST images, isolating a covariate shift while keeping the semantic label space fixed. We pretrain a small ViT on clean EMNIST Letters using the standard $124{,}800$-image training split, with $10\%$ held out for validation, and evaluate source performance on the $20{,}800$ test images. The model uses patch size $4$, embedding dimension $128$, depth $2$, $4$ attention heads, pretraining batch size $256$, learning rate $10^{-3}$, and weight decay $10^{-4}$. We save source checkpoints at epochs $2,5,10,15,20,50,$ and $100$. For downstream adaptation, we initialize from each source checkpoint, freeze the backbone, and train LoRA modules on the attention projections with LoRA rank $4$ and scaling $\alpha=4$. Since the label space is unchanged, the pretrained head is retained and trained jointly with LoRA. We use a low-data downstream regime with $25$ shifted-EMNIST examples per class, i.e. $650$ training examples in total, batch size $128$, learning rate $7\times10^{-4}$, and weight decay $10^{-4}$. Downstream performance is evaluated on the full shifted version of the EMNIST Letters test set. Figure~\ref{fig:emnist-shifted-2x2} shows the strongest manifestation of the effect: late source checkpoints remain strong clean-EMNIST classifiers but adapt substantially more slowly and converge to worse shifted-EMNIST performance under LoRA.

\begin{figure}[!htbp]
    \centering
    \includegraphics[width=0.8\linewidth]{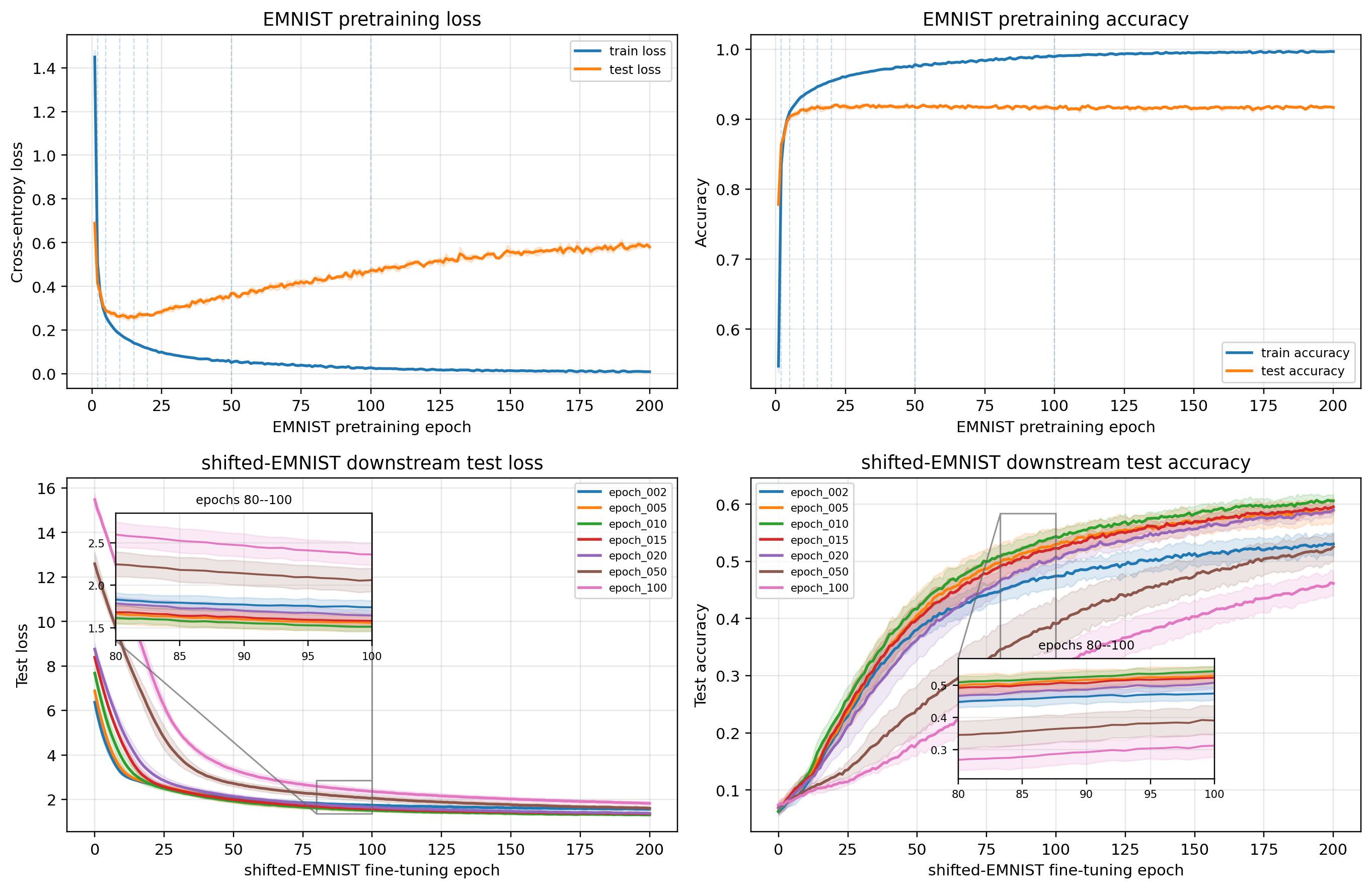}
    \caption{
    EMNIST pretraining and EMNIST$\rightarrow$shifted-EMNIST LoRA adaptation dynamics.
    Top row: EMNIST pretraining loss and accuracy.
    Bottom row: downstream shifted-EMNIST test loss and accuracy for LoRA fine-tuning initialized from different EMNIST source checkpoints.
    In this setting the downstream task preserves the label space but shifts the input distribution. The effect is stronger than in EMNIST$\rightarrow$FashionMNIST: later source checkpoints remain strong EMNIST classifiers but adapt substantially more slowly and converge to worse shifted-EMNIST performance under LoRA. The insets zoom into epochs $80$--$100$ and highlight the separation between early/intermediate and late source checkpoints.
    }
    \label{fig:emnist-shifted-2x2}
\end{figure}

\section{Beyond Single Index Model}
\label{role_of_the_rank_versus_signal}

As emphasized by Hu \textit{et al}.~\cite{hu2022lora}, the LoRA rank plays a non-trivial role. In particular, Sec.~7.3 suggests that the low-rank adaptation matrix can amplify task-relevant features that were already learned during pretraining but not sufficiently emphasized by the base model. Motivated by this perspective, we move beyond the single-index setting and consider a more realistic LoRA scenario, where the goal is to reconstruct a matrix of teacher directions $\omeg^\star\in\mathbb{R}^{K\times d}$ in a narrow two layers neural network model recently studied in~\cite{barbier2025generalization}. Our aim is to disentangle the roles of pretraining and rank: \textit{ we observed that the pretrained weights primarily control the escape from the search phase, while the LoRA rank mainly shapes the subsequent descent phase and the final attainable loss.} We consider both linear and $\mathrm{He}_2$ activations. For each case, we compute the characteristic escape time from the search phase; see Fig.~\ref{fig:escape_time_two_layer_neural_network}. For the $\mathrm{He}_2$ activation, we also integrate the full dynamics and show in Fig.~\ref{fig:dynamics_activation_two_layers_hermite2} that the pretrained weights guide the search phase, whereas the rank controls the final attainable loss.

\paragraph{Data generation.}
The data are generated according to a noiseless teacher model
\begin{equation}
\label{eq:teacher_two_layers}
y(\boldsymbol{x})=\frac{1}{K}\sum_{k=1}^K v_k \phi(\omeg_k^\star\cdot\boldsymbol{x}),
\end{equation}
where $\boldsymbol{x}_i\sim\mathcal N(0,I_d)$, the teacher weights satisfy
$\omeg_k^\star\in\mathbb S^{d-1}$ and $\omeg_k^\star\cdot\omeg_{k'}^\star=\delta_{kk'}$,
and $\phi:\mathbb R\to\mathbb R$ is a generic activation function.

\paragraph{LoRA student model.}
We consider a simplified LoRA-style parameterization of the student in which each teacher direction is initialized with a correlated component and adapted through a shared low-rank update.
Specifically, for rank $R$, let
$\boldU\in\mathbb{R}^{K\times R}$
and $\omeg\in\mathbb{R}^{d\times R}$,
with columns
$$
\omeg=\big[\omeg^{(1)},\dots,\omeg^{(R)}\big],
\qquad
\boldU=\big[\boldu^{(1)},\dots,\boldu^{(R)}\big].
$$
The effective weight vector associated with neuron $k$ is
\begin{equation}
\label{eq:lora_weights}
\tilde{\omeg}_k
=\omeg^{0}_k+
\sum_{r=1}^{R} u_{k,r}\,\omeg^{(r)} = \mu_k\,\omeg_k^\star
+
\sum_{r=1}^{R} u_{k,r}\,\omeg^{(r)},
\qquad \mu_k\in(0,1),
\end{equation}
 with $\omeg^{0}_k=\mu_k\,\omeg_k^\star$ where $\mu_k$ controls the strength of the pre-trained alignment with the $k$-th teacher direction.
The corresponding student output is
\begin{equation}
\label{eq:student_two_layers}
\hat y(\boldsymbol{x})=\frac{1}{K}\sum_{k=1}^K v_k \sigma(\tilde{\omeg}_k\cdot\boldsymbol{x}).
\end{equation}

\paragraph{Order parameters.}
As usual, define the overlaps (for $r,s\in\{1,\cdots,R\}$)
$$
m_{k,r}:=\omeg_k^\star\cdot \omeg^{(r)},
\qquad
q_{rs}:=\omeg^{(r)}\cdot \omeg^{(s)},
\qquad
\Delta_k:=1-\mu_k.
$$
We also introduce the rank-$r$ ``Gram'' of the $u$'s:
$$
C_{rs}:=\sum_{k=1}^K u_{k,r}u_{k,s}\qquad (r,s\in\{1,\cdots,R\}).
$$

\paragraph{Population loss.}
The population loss is
\begin{align}
\nonumber \mathcal L(\omeg,\boldU)
&=\frac{1}{2K}\,\mathbb E_x\Big[\big(y(\boldx)-\hat y(\boldx)\big)^2\Big].
\end{align}

\paragraph{Spherical gradient flow (columns constrained on the sphere).}
Assume each column is constrained to the sphere $\|\omeg^{(r)}\|^2=1$ (equivalently $q_{rr}=1$).
Then the spherical (Riemannian) gradient is
$$
\nabla^{\mathbb S}_{\omeg^{(r)}}\mathcal L
=
\Big(I-\omeg^{(r)}{\omeg^{(r)}}^\intercal\Big)\nabla_{\omeg^{(r)}}\mathcal L
=
\Big(I-\omeg^{(r)}{\omeg^{(r)}}^\intercal\Big)\nabla_{\omeg^{(r)}}\mathcal L,
$$
and the continuous-time gradient flow is
$$
\frac{d u_{k,r}}{dt}=-\frac{\partial \mathcal L}{\partial u_{k,r}},
\qquad
\frac{d \omeg^{(r)}}{dt}\,=-\nabla^{\mathbb S}_{\omeg^{(r)}}\mathcal L.
$$

\paragraph{Induced dynamics for the overlaps.}
Differentiate $m_{k,r}=\omeg_k^\star\cdot \omeg^{(r)}$:
$$
\frac{d m_{k,r}}{dt}
=-\omeg_k^\star\cdot \nabla^{\mathbb S}_{\omeg^{(r)}}\mathcal L=-\Big(\omeg_k^\star-m_{k,r}\omeg^{(r)}\Big)\nabla_{\omeg^{(r)}}\mathcal L.
$$

\emph{Effective alignment and sufficient statistics.}
For teacher direction $k$, the component of the student along $\omeg_k^\star$ is
\begin{equation}
\label{eq:m_eff_rankR}
m_{k,\mathrm{eff}}
:=\omeg_k^\star\cdot \tilde{\omeg}_k
=
\mu_k+\sum_{r=1}^R u_{k,r}m_{k,r},
\end{equation}
so tracking the pairs $(u_{k,r},m_{k,r})_{r=1}^R$ is sufficient to describe how much the $k$-th teacher direction
is represented in the current model.

\paragraph{Search-phase linearization for Linear activation}
Now we consider linear activation \begin{equation}
    \phi(x)=\sigma(x)=x
\end{equation}

In the search phase, the dynamics of the overlap leads to
\begin{equation*}
    \frac{d m_{ir}}{dt}
     =\frac{v_i\Delta_i}{K}\sum_{j=1}^K v_j u_{jr}
    ,\quad \frac{d u_{ir}}{dt}
    = \frac{v_i}{K}\big[ \sum_{j=1}^K v_j \Delta_j m_{jr} - \sum_{j=1}^K v_j u_{jr}
    \big].
\end{equation*}

Hence, introducing the collective scalar observables
\begin{equation}
    s_r := \sum_{j=1}^K v_j u_{jr} = \bv^\top \bm u_r,
    \qquad
    g_r := \sum_{j=1}^K v_j \Delta_j m_{jr}
    = (\bv\odot\bDelta)^\top \bm m_r,
\end{equation}
the dynamics becomes
\begin{equation}
    \frac{d m_{ir}}{dt}= \frac{v_i\Delta_i}{K}\,s_r,\quad \frac{d u_{ir}}{dt}
    = \frac{v_i}{K}\,(g_r-s_r) .
\end{equation}

Let
\begin{equation}
    \bm d := \bv\odot \bDelta.
\end{equation}
Then for each fixed rank direction $r$,
\begin{equation*}
    \dot{\bm m}_r
    = \frac{1}{K}\,\bm d\,(\bv^\top \bm u_r)
    ,\quad
    \dot{\bm u}_r
    = \frac{1}{K}\,\bv\,\big( \bm d^\top \bm m_r - \bv^\top \bm u_r \big).
\end{equation*}
Equivalently,
\begin{equation}
    \dot{\bm m}_r = -K_0^{\rm lin}\bm u_r + O(\varepsilon^2),
    \qquad
    \dot{\bm u}_r = -\widehat K_0^{\rm lin}\bm m_r - E_0^{\rm lin}\bm u_r + O(\varepsilon^2),
\end{equation}
with
\begin{equation}
    K_0^{\rm lin} =   -\frac{1}{K}(\bv\odot\bDelta)\bv^\top,
    \qquad \widehat K_0^{\rm lin}  =-\frac{1}{K}\bv(\bv\odot\bDelta)^\top,
    \qquad E_0^{\rm lin} = \frac{1}{K}\bv\bv^\top.
\end{equation}

Projecting onto the two collective coordinates $(g_r,s_r)$ gives a closed system:
\begin{equation}
    \dot g_r= (\bv\odot\bDelta)^\top \dot{\bm m}_r
    = \frac{\|\bv\odot\bDelta\|^2}{K}\,s_r
   ,\quad \dot s_r= \bv^\top \dot{\bm u}_r
    = \frac{\|\bv\|^2}{K}\,(g_r-s_r).
\end{equation}
Therefore the non-trivial linearized dynamics is governed by
\begin{equation}
    \frac{d}{dt}
    \begin{pmatrix}
        g_r\\
        s_r
    \end{pmatrix}
    =
    \begin{pmatrix}
        0 & \alpha\\
        \beta & -\beta
    \end{pmatrix}
    \begin{pmatrix}
        g_r\\
        s_r
    \end{pmatrix}
    +O(\varepsilon^2),
\end{equation}
where
\begin{equation}
    \alpha:=\frac{\|\bv\odot\bDelta\|^2}{K}=\frac{1}{K}\sum_{i=1}^{K} v^2_i(1-\mu_i)^2,
    \qquad
    \beta:=\frac{\|\bv\|^2}{K}=\frac{1}{K}\sum_{i=1}^{K} v^2_i.
\end{equation}
Its eigenvalues are
\begin{equation}
    \lambda_\pm
    =
    \frac{-\beta\pm\sqrt{\beta^2+4\alpha\beta}}{2},
\end{equation}
so that the escape rate is
\begin{equation}
\label{equ:escaping_time_linear_two_layer}
  \tau^{-1} =  \lambda_{\rm esc}^{\rm lin}= \frac{-\beta+\sqrt{\beta^2+4\alpha\beta}}{2}.
\end{equation}

In particular, for linear activations the leading-order escape depends only on the collective mismatch vector $\bv\odot\bDelta$, and not on the rank $R$: each rank direction $r$ obeys the same reduced two-dimensional dynamics. 

\begin{figure}
    \centering
    \includegraphics[width=0.5\linewidth]{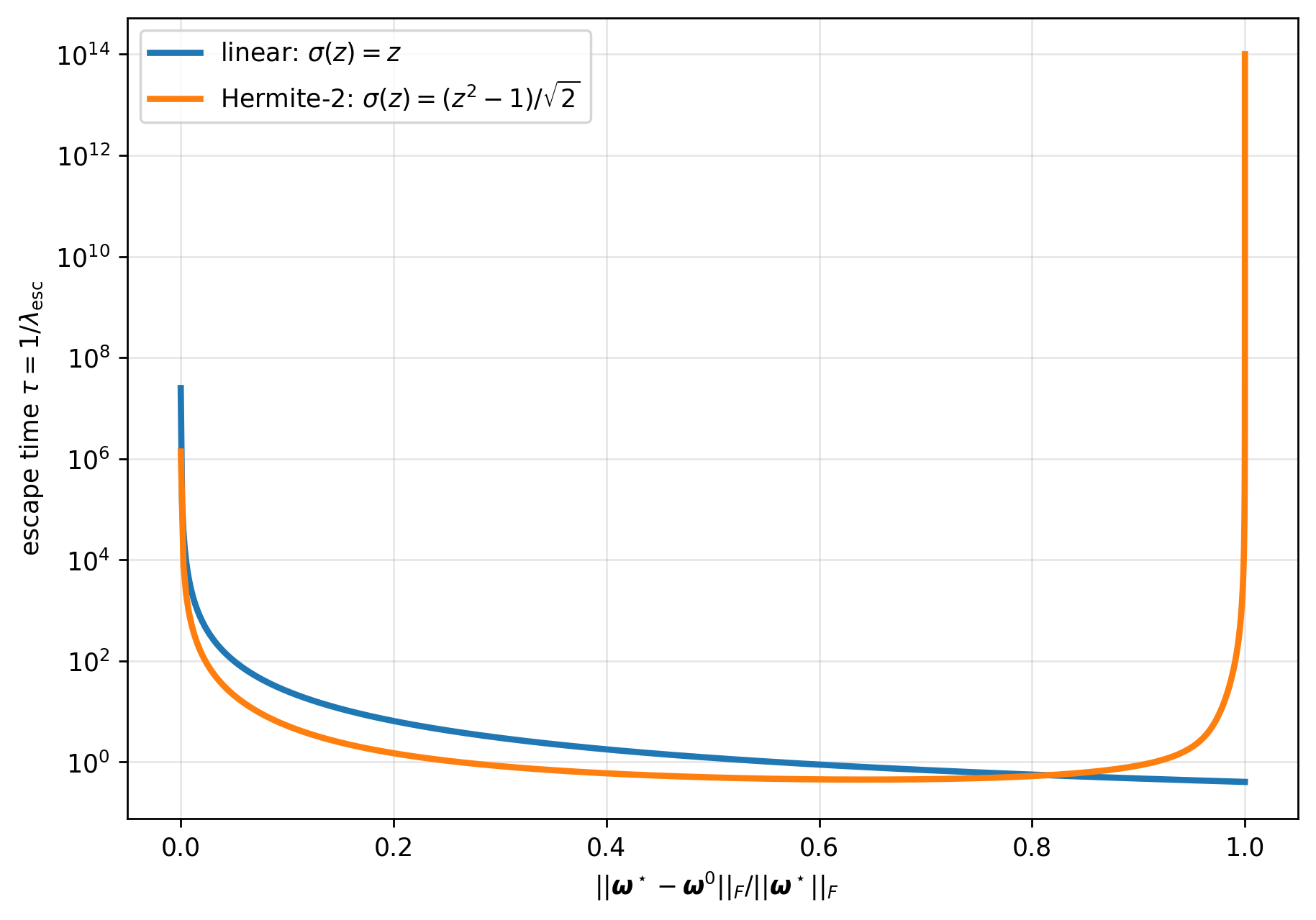}
    \caption{Search-phase escape time as a function of the normalized Frobenius distance between the pretrained weigth matrix $\omeg^0$ and ground truth matrix $\omeg^\star$ for linear \eqref{equ:escaping_time_linear_two_layer} and Hermite-2 \eqref{equ:escaping_time_quadratic_two_layer} activations. 
We vary a homogeneous pre-training alignment $\mu_i=\mu$ for all hidden units and keep the readout weights fixed, with $v_i=1$ throughout the comparison. 
For each value of $\mu$, the escape time is computed as $\tau=1/\lambda_{\rm esc}$ from the corresponding linearized search dynamics. 
The linear activation depends only on the collective mismatch vector $\bv\odot(1-\bmu)$, whereas the Hermite-2 activation uses the quadratic residual mismatch associated with the projected teacher subspace. 
The Frobenius norm is normalized by the teacher norm.
}
    \label{fig:escape_time_two_layer_neural_network}
\end{figure}

\begin{figure}
    \centering
        \includegraphics[width=1.0\textwidth]{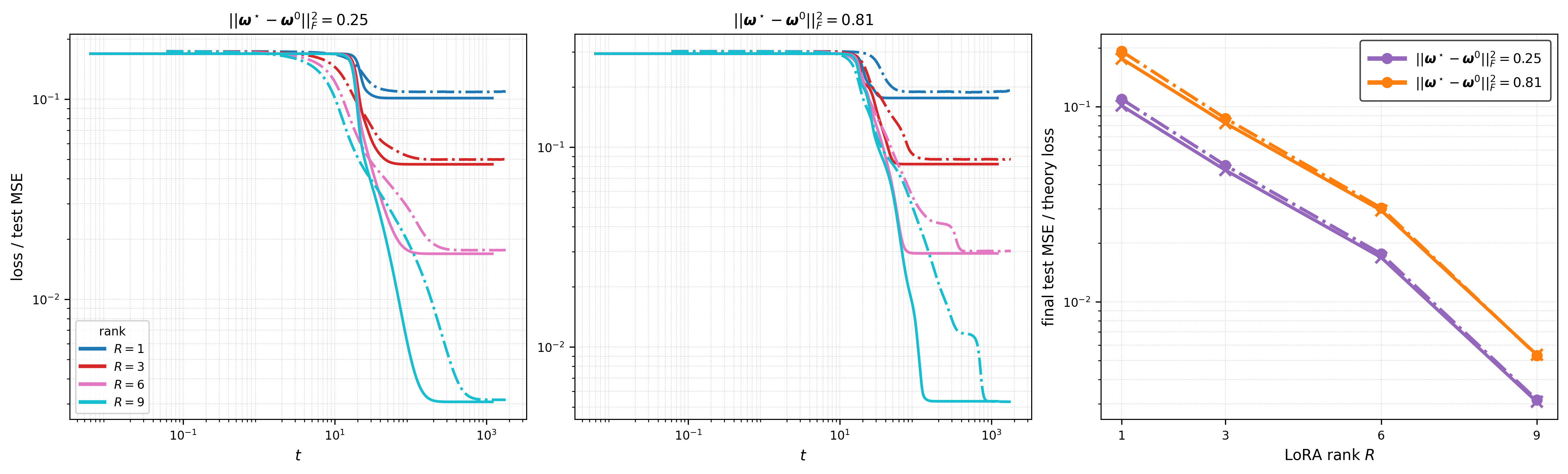}
        \label{fig:the_role_of_the_rank}
        \vspace{-19pt}
    \caption{
Learning dynamics of the two-layer student model with Hermite-2 activation and homogeneous readout, using output weights $v_i=1$. The hidden-layer width is set to $K=10$. We consider two levels of initial mismatch between the student pretrained weight and the teacher weights, measured by $\|\boldsymbol{\omega}^\star-\boldsymbol{\omega}^{0}\|_F^2\in\{0.25,0.81\}$, obtained by taking the same pretraining alignment $\mu_i\in\{0.5,0.1\}$ across all hidden units. The student is trained on data generated by the teacher model~\eqref{eq:teacher_two_layers}. We compare the overlap-theory prediction, obtained by integrating the ODE~\eqref{equ:two_layer_dynamic_hermite2}, with finite online-SGD experiments for different LoRA ranks $R\in\{1,3,6,9\}$. Solid lines correspond to the ODE prediction, while dash-dotted lines correspond to online SGD. The online-SGD experiments use input dimension $d=50$, learning rate $\eta=0.06$, and batch size $B=5000$, and results are averaged over three independent runs. \textbf{Left--middle:} evolution of the population loss \eqref{equ:population_loss_He2} as a function of time for the two pretraining mismatch levels, with colors indicating the LoRA rank. \textbf{Right:} final population loss as a function of the LoRA rank for different pretraining weights levels.
}
 \label{fig:dynamics_activation_two_layers_hermite2}
\end{figure}

\paragraph{Search-phase linearization for normalized Hermite-2.}

We consider the activation
\begin{equation}
    \sigma(z)=\phi(z)=\frac{\mathrm{He}_2(z)}{\sqrt 2}
    =\frac{z^2-1}{\sqrt 2}.
\end{equation}

In the search phase we assume
\begin{equation}
    u_{ir}=O(\varepsilon),
    \qquad
    m_{ir}=O(\varepsilon),
    \qquad
    q_{rs}=\delta_{rs}+O(\varepsilon).
\end{equation}

The linearized search dynamics is:
\begin{align}
    \dot m_{ir}
    = -\kappa_i u_{ir}   +O(\varepsilon^2),\quad
    \dot u_{ir}= -\kappa_i m_{ir}  -\eta_i u_{ir}+O(\varepsilon^2),\quad \dot q_{rs}= O(\varepsilon^2),
\end{align}
where
\begin{equation}
    \kappa_i=\frac{v_i\mu_i}{K}\big[ 3v_i(\rho_i-1)
        + \sum_{j\neq i}v_j(\rho_j-1) \big],
    \quad \eta_i :=\frac{v_i}{K} \big[ v_i(3\rho_i-1)+ \sum_{j\neq i}v_j(\rho_j-1)
    \big].
\end{equation}

We recall that
\begin{equation}
    r_{ii}=\rho_i+O(\varepsilon^2),
    \qquad
    r^\star_{ii}=\mu_i+O(\varepsilon^2),
    \qquad
    r_{ij}=O(\varepsilon^2),
    \qquad
    r^\star_{ij}=O(\varepsilon^2),\quad 
    \rho_i:=\mu_i^2.
\end{equation}

Equivalently, for each fixed pair $(i,r)$,
\begin{equation}
    \frac{d}{dt}
    \begin{pmatrix}
        m_{ir}\\
        u_{ir}
    \end{pmatrix}=
    - \begin{pmatrix}
        0 & \kappa_i\\
        \kappa_i & \eta_i
    \end{pmatrix}
    \begin{pmatrix}
        m_{ir}\\
        u_{ir}
    \end{pmatrix}
    +O(\varepsilon^2).
\end{equation}
The two eigenvalues of this block are
\begin{equation}
    \lambda_{i,\pm}
    =
    \frac{-\eta_i
    \pm \sqrt{\eta_i^2+4\kappa_i^2}}{2}.
\end{equation}
Therefore 
\begin{equation}
\label{equ:escaping_time_quadratic_two_layer}
    \lambda_{\rm esc}
    =  \max_i  \frac{-\eta_i +\sqrt{\eta_i^2+4\kappa_i^2}}{2}.
\end{equation}

\paragraph{Gradient Flow Equation in the Descend-Phase for Hermite-2} 

~\\

The dynamics of the order parameter are
\begin{equation}
\label{equ:two_layer_dynamic_hermite2}
 \dot u_{ir}
=-\gamma G^u_{ir},\quad \dot m_{ir}
=-\gamma H_{ir}-\frac{\chi}{2}m_{ir}L_{rr}+O(\gamma^3),
\quad
\dot q_{rs}=-\gamma(J_{sr}+J_{rs})+\frac{\chi}{2}\Big[2L_{rs}
-q_{rs}(L_{rr}+L_{ss})\Big]+O(\gamma^3).
\end{equation} 
with 
\begin{align*}
L_{rs}
&=\sum_{a\neq r}\sum_{b\neq s}\Big(q_{ab}-q_{sb}q_{as}-q_{ra}q_{rb}+q_{ra}q_{sb}q_{rs}\Big)
G^q_{ra}G^q_{sb}+\sum_{a\neq r}\sum_{j=1}^K\Big(m_{ja}-q_{as}m_{js}-q_{ra}m_{jr}+q_{ra}q_{rs}m_{js}\Big)G^q_{ra}G^m_{js}
\nonumber\\
&+\sum_{i=1}^K\sum_{b\neq s}
\Big(m_{ib}-q_{sb}m_{is}-m_{ir}q_{rb}+m_{ir}q_{sb}q_{rs}\Big)G^m_{ir}G^q_{sb}+\sum_{i=1}^K\sum_{j=1}^K\Big(\delta_{ij}
-m_{is}m_{js}-m_{ir}m_{jr}+q_{rs}m_{ir}m_{js}
Big)G^m_{ir}G^m_{js},\\
L_{rr}
&=
\sum_{a,b\neq r}
(q_{ab}-q_{ra}q_{rb})
G^q_{ra}G^q_{rb}+
2\sum_{a\neq r}\sum_{i=1}^K
(m_{ia}-q_{ra}m_{ir})
G^q_{ra}G^m_{ir}+
\sum_{i=1}^K\sum_{j=1}^K
(\delta_{ij}-m_{ir}m_{jr})
G^m_{ir}G^m_{jr},
\end{align*}
where $\chi:=\frac{\gamma^2}{d}$ and $\gamma$ is the learning rate.

We have defined

\begin{align*}
G^m_{ir}
&=\frac{v_i}{K}\Bigg[v_i u_{ir}\Big(\mu_i(3r_{ii}-1)-2r_{ii}^\star\Big)+\sum_{j\neq i}v_j\Big(\mu_i(r_{jj}-1)u_{ir}+2(\mu_i r_{ij}-r_{ij}^\star)u_{jr}\Big)\Bigg],
\\
G^u_{ir}
&=\frac{v_i}{K}\Bigg[v_i\Big((3r_{ii}-1)\Xi_{ir}-2r_{ii}^\star m_{ir}\Big)+\sum_{j\neq i}v_j\Big((r_{jj}-1)\Xi_{ir}+2r_{ij}\Xi_{jr}-2r_{ji}^\star m_{jr}\Big)\Bigg],
\\
G^q_{rs}&=\frac{1}{2K}\Bigg[\sum_i v_i^2u_{ir}u_{is}(3r_{ii}-1)+\sum_{i\neq j}v_iv_j\Big(\frac12 u_{ir}u_{is}(r_{jj}-1)+\frac12 u_{jr}u_{js}(r_{ii}-1)+2u_{ir}u_{js}r_{ij}
\Big)
\Bigg],\\
H_{ir}
&=\sum_{a\neq r}(m_{ia}-m_{ir}q_{ra})G^q_{ra}+\sum_{j=1}^K(\delta_{ij}-m_{ir}m_{jr})G^m_{jr},\quad  J_{sr}=\sum_{a\neq r}(q_{as}-q_{ra}q_{rs})G^q_{ra}+\sum_{i=1}^K(m_{is}-q_{rs}m_{ir})G^m_{ir}.
\end{align*} 
Finally thee remaining scalar are:
\begin{align*}
    r_{ii}   &= \mu_i^2
    +2\mu_i\sum_{a=1}^{R}u_{ia}m_{ia}
    +\sum_{a,b=1}^{R}u_{ia}u_{ib}q_{ab},\quad r^\star_{ii}=\mu_i+\sum_{a=1}^{R}u_{ia}m_{ia},
    \\
    r_{ij} &= \mu_i\sum_{a=1}^{R}u_{ja}m_{ia} +\mu_j\sum_{a=1}^{R}u_{ia}m_{ja} +\sum_{a,b=1}^{R}u_{ia}u_{jb}q_{ab},
    r^\star_{ij}
    =\sum_{a=1}^{R}u_{ja}m_{ia},
    \qquad i\neq j,\\
    \Xi_{ir}&:=\mu_i m_{ir}+\sum_{a=1}^{R}u_{ia}q_{ra}.
\end{align*}

The system of ODE \eqref{equ:two_layer_dynamic_hermite2}  helps to track the dynamic of the population loss The population loss is given by \begin{equation}
\label{equ:population_loss_He2}
\mathcal L_{H_2} = \frac{1}{2K}\big[   \sum_{i,j=1}^K v^2_i  +  \sum_{k,\ell=1}^K v_k v_\ell r_{k\ell}^2
        - 2\sum_{i,k=1}^K v_i v_k (r^\star_{ik})^2
    \big].
\end{equation}


\end{document}